\documentclass{article}

\usepackage{microtype}
\usepackage{graphicx}
\usepackage{subcaption}
\usepackage{booktabs} 
\usepackage{makecell}

\usepackage{hyperref}

\usepackage[accepted]{icml2026}

\usepackage{amsmath}
\usepackage{amssymb}
\usepackage{mathtools}
\usepackage{amsthm}

\usepackage[utf8]{inputenc}
\usepackage{multido}
\usepackage{mathdots}
\usepackage{mathtools}
\usepackage{graphicx, animate}
\usepackage{ragged2e}
\usepackage{paracol}
\usepackage{tcolorbox}
\usepackage{array}
\usepackage{adjustbox}
\usepackage{caption}
\usepackage{mathrsfs}
\usepackage{amssymb} 
\usepackage{booktabs}
\usepackage{multirow}
\graphicspath{{figures/}}
\usepackage{abbreviations}
\usepackage[T1]{fontenc}
\usepackage[capitalize,noabbrev]{cleveref}
\usepackage{natbib}
\usepackage{xcolor} 
\usepackage{algorithm}  
\usepackage{dsfont}
\usepackage{subcaption}
\usepackage{enumitem}

\allowdisplaybreaks
\usepackage{threeparttable}
\usepackage{adjustbox}
\usepackage{tikz}



\usepackage[indLines]{algpseudocodex}
\tikzset{algpxIndentLine/.style={draw=black, very thin}}
\usepackage{makecell}
\usepackage{bbm}
\usepackage{stfloats}

\usepackage[capitalize,noabbrev]{cleveref}

\theoremstyle{plain}
\theoremstyle{definition}
\theoremstyle{remark}

\usepackage[textsize=tiny]{todonotes}

\icmltitlerunning{Practical and Optimal Algorithm for Linear Contextual Bandits with Rare Parameter Updates}

\begin{document}

\twocolumn[
  \icmltitle{Practical and Optimal Algorithm for Linear Contextual Bandits\\
  with Rare Parameter Updates}



  \icmlsetsymbol{equal}{*}

  \begin{icmlauthorlist}
    \icmlauthor{Sanghoon Yu}{snu}
    \icmlauthor{Min-hwan Oh}{snu}
  \end{icmlauthorlist}

  \icmlaffiliation{snu}{Seoul National University, Seoul, Republic of Korea}

  \icmlcorrespondingauthor{Min-hwan Oh}{minoh@snu.ac.kr}

  \icmlkeywords{Bandit algorithms, Linear Contextual bandit, Batched bandit, limited adaptivity}

  \vskip 0.3in
]



\printAffiliationsAndNotice{}  

\begin{abstract}
We study linear contextual bandits under \emph{rare parameter updates}:
the learner may incorporate reward feedback into its parameter estimate only at a small number of update times,
while still observing contexts online and selecting actions sequentially.
This viewpoint clarifies a practical distinction that is often blurred in the literature:
many ``strictly batched'' methods additionally restrict \emph{within-interval context adaptivity},
meaning that the action rule inside an interval cannot depend on the \emph{sequence} of realized contexts/actions in that interval (beyond the current round's context).
For linear contextual bandits, we propose two practical algorithms with only $\Ocal(\log\log T)$ parameter updates.
Our first algorithm \texttt{BLCE-G} attains minimax-optimal regret (up to polylogarithmic factors in $T$) \emph{simultaneously} in both the small-$K$ and large-$K$ regimes under a static schedule.
Our second algorithm \texttt{BLCE} removes the near G-optimal design step---a dominant computational bottleneck in prior strictly batched static-grid methods---yet preserves minimax-optimal regret and achieves the lowest known runtime complexity among optimal algorithms.
We further extend these rare-update and computational principles to generalized linear contextual bandits. Overall, our results yield minimax-optimal algorithms for linear contextual
bandits and a near-optimal generalized-linear extension under $\Ocal(\log\log T)$ parameter updates, while remaining computationally
efficient in practice.
\end{abstract}

\section{Introduction} \label{introduction}

Stochastic \emph{linear contextual bandits} are a cornerstone of sequential decision-making, where an agent repeatedly selects actions from a time-varying, feature-based set and observes rewards generated by an unknown linear model~\citep{abe1999associative,auer2002using,abe2003reinforcement,dani2008stochastic,li2010contextual,chu2011contextual,abbasi2011improved,li2019nearly,lattimore2020bandit,kirschner2021asymptotically}. The resulting low-dimensional structure enables efficient generalization and has seen wide application, from recommender systems~\citep{li2010contextual} and inventory control~\citep{jin2021shrinking} to clinical trials and precision medicine~\citep{lu2021bandit}.

In deployments of bandit algorithms, the practical bottleneck is often not selecting an action given a context, but \emph{updating} the underlying model once new reward feedback arrives: model retraining, confidence set recomputation, or posterior updates may require expensive pipelines, logging joins, privacy checks, or human-in-the-loop validation. This motivates learning under \emph{rare parameter updates}, where the learner is allowed to incorporate rewards into its parameter estimate only at a small number of update times, while contexts continue to arrive online and decisions must still be made sequentially.

A common special case is the \emph{batched-feedback} protocol, where rewards from an interval become available only at the end of that interval. However, in the contextual bandit literature, the term ``batched'' is also frequently used in a stronger sense that restricts \emph{within-interval context adaptivity}. Concretely, many strictly batched methods commit at the start of an interval to an action rule that cannot depend on the \emph{sequence} of realized contexts/actions within that interval.
This stronger restriction is not inherent to batched feedback: contexts are observed online before choosing actions and are thus readily available.
For example, \citet{karbasi2021parallelizing} formalize a batch policy at time $t$ in the current interval as a mapping
$
\pi_t : H_{t_{\ell-1}} \times C_t \to \mathcal{A},
$
where $H_{t_{\ell-1}}$ denotes the history up to the previous interval boundary and $C_t$ denotes the set of contexts observed up to and including time $t$.
This formalization permits dependence on within-interval contexts.
Motivated by practice, we therefore separate these notions and focus on \emph{rare parameter updates} as the key constraint, while allowing (and exploiting) reward-free within-interval updates based on the arriving contexts.

Recently, batched contextual bandit algorithms have been shown to achieve strong near-optimal regret guarantees with as few as $\Ocal(\log\log T)$ parameter-update times under i.i.d.\ contexts~\citep{ruan2021linear,hanna2023efficient,hanna2023contexts,zhang2025almost}. 
Yet, 
they 
rely on computationally expensive primitives such as G-optimal design, or suffer from exponential dependence on dimension due to discretization; see \cref{sec:bottleneck_strictly_batched}.
As a result, despite attractive small number of update frequency and their regret guarantees, their computational complexity is impractically high (see Section~\ref{sec:batched_vs_rare} for runtime experiments),
making these approaches hard to deploy in real-world pipelines where batched learning is meant to reduce cost.
This mismatch motivates our method, which preserves rare parameter updates while avoiding these costly subroutines and exploiting reward-free within-interval context information to remain computationally practical.


Our main message in this work is simple:
 one can retain the \emph{same} $\Ocal(\log\log T)$ number of parameter updates while substantially improving practical efficiency:
by allowing \emph{reward-free} within-interval state updates based only on the arriving contexts and chosen actions, we can avoid heavy G-optimal design steps and obtain algorithms that are both minimax-optimal and fast. This is precisely the operational regime practitioners often want: infrequent reward-dependent retraining, but lightweight online processing of contexts that are already observed. 

The discussion above leads to the following questions:
\begin{itemize}
    \item For \emph{linear contextual bandits}, can we design algorithms that achieve minimax-optimal regret in both small-$K$ and large-$K$ regimes using only $\Ocal(\log\log T)$ parameter updates? Can we also remove the G-optimal design procedure while preserving minimax optimality and practical runtime?
    \item Can we extend the same rare-update and computational
    principles to generalized linear contextual bandits while retaining
    near-optimal regret and $\Ocal(\log\log T)$ parameter updates?
\end{itemize}

Positive answers to these questions would unify theory and practice in linear and generalized linear contextual bandits, leading to algorithms that remain statistically optimal and computationally efficient under limited adaptivity. Our main contributions are summarized as follows:

\begin{itemize}
\item \textbf{Tightest regret bounds for linear contextual bandits under $\Ocal(\log\log T)$ parameter updates.}
We introduce \texttt{BLCE-G}, which combines near G-optimal design and arm elimination.
It achieves the worst-case regret bound
$
\Ocal \big(\sqrt{dT}(\sqrt{\log(KT)} \wedge \sqrt{d+\log T})\sqrt{\log d\log\log T}\big),
$
where $K$ is the number of arms, $d$ is the feature dimension, and $T$ is the horizon.
\texttt{BLCE-G} simultaneously matches the minimax lower bounds $\Omega(d\sqrt{T})$ in the large-$K$ regime ($K \geq \Omega(e^d)$) and $\Omega(\sqrt{dT \log K})$ in the small-$K$ regime ($K \leq \Ocal(e^d)$), up to logarithmic factors.

\item \textbf{First minimax-optimal algorithm in rare updates without G-optimal design.}
We propose \texttt{BLCE}, which replaces the G-optimal design step with uncertainty-driven exploration combined with arm elimination.
It still achieves minimax-optimal regret and attains the lowest total time complexity $\Ocal(Kd^2T\log\log T)$ among optimal algorithms.
Its guarantees also extend beyond conventional i.i.d.\ contexts, as discussed in \cref{rmk:loosing_iid}.

\item \textbf{Redefining the limited adaptability landscape: fully sequential vs.\ strictly batched vs.\ rare parameter updates.}
We explicitly distinguish three operational settings that have often been conflated :
(i) fully sequential (fully adaptive) contextual bandits;
(ii) \emph{strictly batched} contextual bandits, which disallow within-interval context adaptivity;
and (iii) rare parameter updates, which restrict reward-dependent parameter updates but allow reward-free within-interval context adaptivity.
We compare prior work and our algorithms along these dimensions (see \cref{sec:batched_vs_rare} and \cref{table:comparison}).


\item \textbf{Extensions to generalized linear contextual bandits.}
We extend the same rare-update and computational principles to generalized
linear rewards and propose \texttt{BGLE}, which retains
$\Ocal(\log\log T)$ parameter updates and near-optimal regret guarantees.


\item \textbf{Practical superiority.}
Our experiments demonstrate that \texttt{BLCE-G}, \texttt{BLCE}, and \texttt{BGLE} consistently outperform prior linear and generalized linear contextual baselines across various instances, combining provable efficiency with strong empirical performance and substantially reduced runtime overhead.
\end{itemize}


\begin{table*}[t]
\centering
\begin{threeparttable}
\scriptsize
\caption{Worst-case regret, \emph{parameter-update} complexity, and time complexity comparison for linear contextual bandits.
\textbf{Both \texttt{BLCE-G} and \texttt{BLCE} achieve minimax-optimal regret}, matching the minimax lower bounds ${\Omega}(\sqrt{dT\log K} \wedge d\sqrt{T})$~\citep{dani2008stochastic,he2022reduction} across all regimes while requiring only $\Ocal(\log\log T)$ parameter updates.
Among optimal algorithms, \textbf{\texttt{BLCE-G} attains the tightest regret, whereas \texttt{BLCE} achieves the lowest time complexity}.
Here, $\mathcal{T}_{\mathrm{opt}}$ denotes the cost of a single call to the linear optimization oracle.
}
\label{table:comparison}
\begin{tabular}{lllll}
\toprule
\textbf{Paper} & \textbf{Worst-Case Regret} & \makecell[l]{\textbf{Parameter} \\ \textbf{Updates}} &  \makecell[l]{\textbf{Context} \\ \textbf{Adaptive}} &  \textbf{Time Complexity} \\
\midrule
{\citet{abbasi2011improved}} & $\Ocal(d\sqrt{T}\log T)$ & $\Ocal(d\log T)$ & $\texttt{Yes}$ & $\Ocal((Kd+d^2)T+Kd^3\log T)$ 
\\
{\citet{karbasi2021parallelizing}} & $\Ocal(d^{3/2}\sqrt{T}\log T)$ & $\Ocal(K\log T)$& $\texttt{Yes}$ & $\Ocal((Kd+d^2)T)$ 
\\
{\citet{ruan2021linear}} & $\Ocal(\sqrt{dT\log(dKT)\log d}\log\log T)$ & $\Ocal(\log \log T)$ & $\texttt{No}$ & $\Ocal(Kd^4T(\log T+\log d))$
\\
{\citet{hanna2023contexts}} & $\Ocal(d\sqrt{T\log T}\log\log T)$ & $\Ocal(\log \log T)$ & $\texttt{No}$ & $\Omega(T^d)$ \\
\addlinespace[1.5pt]
\multirow{2}{*}{\citet{hanna2023efficient}} & \multirow{2}{*}{$\Ocal(d^{3/2}\sqrt{T}\log T\log\log T)$} & \multirow{2}{*}{$\Ocal(\log \log T)$}& \multirow{2}{*}{$\texttt{No}$} & 
$\Ocal(d^3\mathcal{T}_{\mathrm{opt}}T\log d\log^3 T\log\log T)$ \\
&  & & &
$+ \, \Ocal(d^4\log d\log^3 T\log\log T)$ \\
\addlinespace[1.5pt]
\multirow{2}{*}{\citet{zhang2025almost}} & \multirow{2}{*}{$\Ocal(\sqrt{dT\log (dKT)\log T}\,\log(dT) \log\log T)$} & \multirow{2}{*}{$\Ocal(\log \log T)$} & \multirow{2}{*}{$\texttt{No}$} & $\Ocal(Kd^2T\log\log T)$
\\
& & & & $+\Ocal(Kd^{7/2}\sqrt{T\log(dKT)\log T})$ \\
\midrule
\cref{BLCE-G} (\texttt{BLCE-G}) & 
$\min\begin{cases*}
\Ocal\big(\sqrt{dT\log(KT)\log d\log\log T}\big)\\
\Ocal\big(\sqrt{d(d+\log T)T\log d\log\log T}\big)
\end{cases*}$
& $\Ocal(\log \log T)$ & $\texttt{Yes}$ & $\Ocal(Kd^2T(d+\log\log T))$ \\
\addlinespace[3pt]
\cref{BLCE} (\texttt{BLCE}) & $\min\begin{cases*}
\Ocal\big(\sqrt{dT\log(KT)\log T\log\log T}\big)\\
\Ocal\big(\sqrt{d(d+\log T)T\log T\log\log T}\big)
\end{cases*}$ & $\Ocal(\log \log T)$ & $\texttt{Yes}$ & $\Ocal(Kd^2T\log\log T)$ \\
\bottomrule
\end{tabular}
\begin{tablenotes}
\footnotesize
\item \textbf{Context Adaptive:}
In this column, \texttt{Yes} means that, between reward-based parameter updates,
either
(i) the action-selection rule may depend on past within-interval
contexts or chosen actions, beyond the current round's context set, or
(ii) the next interval endpoint or parameter-update time may be
determined adaptively from contexts or chosen actions realized so far
within the interval.
\texttt{No} means that neither component has such dependence.
This column records algorithm-level context/action adaptivity between
parameter updates and is not an indicator of strict batching.
Under our terminology, strict batching is determined solely by the
action-selection rule and is independent of whether the update schedule
is adaptive.
\end{tablenotes}
\end{threeparttable}
\end{table*}

\section{Batched Feedback, Strict Batching, and Rare Parameter Updates}
\label{sec:batched_vs_rare}

The term ``batched bandits'' is used in the literature in multiple, operationally distinct senses.
Since our goal is to develop \emph{practical} algorithms and enable fair comparisons, we make these distinctions explicit.
In contextual problems, two axes are orthogonal: (i) when rewards are revealed (the feedback protocol), and (ii) whether the learner may adapt \emph{within an interval} to the realized contexts and past actions.
We organize the landscape into three settings---\emph{fully sequential}, \emph{strictly batched}, and \emph{rare parameter updates}---and emphasize that \emph{batched feedback} is a reward-delay constraint that can underlie either of the latter two.

We say the problem has \emph{batched feedback} if the horizon $[T]$ is partitioned into $B$ disjoint intervals and the rewards collected within an interval are revealed only at its endpoint.
This is purely a constraint on \emph{reward availability}: regardless of how rewards are delayed, contexts are still observed online before actions are chosen. In the \emph{fully sequential} setting, rewards are revealed immediately each round and the learner may update its parameter estimate and decision rule every round.
In batched-feedback terminology, this corresponds to the special case $B=T$. In contrast, we use the term \emph{strictly batched} to denote a stronger restriction than batched feedback: within each interval $(\mathcal{T}_{\ell-1},\mathcal{T}_\ell]$, the learner must commit at time $\mathcal{T}_{\ell-1}$ to an action rule that is \emph{not} allowed to depend on the \emph{realized sequence} of contexts and past actions within the interval.
Many static-grid contextual algorithms that are referred to as ``batched'' adopt (often implicitly) this strict restriction to obtain the minimal $\Ocal(\log\log T)$ parameter-update complexity~\citep{ruan2021linear,hanna2023contexts,zhang2025almost}.

Finally, we say an algorithm has \emph{rare parameter updates} if it updates its reward-based parameter estimate (e.g., via ridge regression, an MLE, or a posterior update) at most $B$ times over $T$ rounds.
Crucially, this notion imposes \emph{no restriction} on how the learner processes contexts within an interval: the learner may maintain and update reward-free state variables online using the realized contexts and chosen actions, and may let its action choices depend on these evolving state variables.
This distinction is practically important in contextual bandits because contexts are \emph{readily available} online: delaying the use of already observed contexts until the next parameter update is neither required by the feedback protocol nor typically natural.
Even when reward-driven retraining is expensive, updating lightweight statistics (e.g., Gram matrices, uncertainty scores, or elimination candidate sets) using contexts and actions is usually cheap and leverages information the learner has \emph{already} observed.
Allowing such reward-free within-interval adaptivity therefore bridges a natural gap between the two extremes: fully sequential methods (maximal reward adaptivity) and strictly batched methods (no within-interval context adaptivity).

Our algorithms operate under a batched-feedback protocol with a static schedule and perform only $\Ocal(\log\log T)$ parameter updates.
They are \emph{not strictly batched}: they perform lightweight, reward-free within-interval updates that influence action selection.
This design choice preserves the desired operational constraint (rare reward-based parameter retraining) while avoiding the dominant computational bottlenecks that arise when one insists on strict batching (see \cref{sec:bottleneck_strictly_batched}).

\section{Bottleneck of Existing Strictly Batched Algorithms}
\label{sec:bottleneck_strictly_batched}

A recurring theme in the strictly batched contextual bandits literature is that the lack of within-interval context adaptivity forces the algorithm to ``pre-plan'' exploration within each interval.
To guarantee sufficient information gain \emph{without} reacting to the realized sequence of contexts/actions, existing strictly batched methods often invoke computationally heavy primitives:
repeated G-optimal design computations, discretization of continuous spaces via fine nets, or repeated oracle calls.
These steps can dominate the overall runtime, even when the number of parameter updates is only $\Ocal(\log\log T)$.

Concretely, \citet{ruan2021linear} repeatedly compute a G-optimal design distribution on the observed arm set as contexts arrive; such per-round design computation is costly, leading to a large polynomial dependence on $d$ in the overall runtime.
In contrast, \citet{hanna2023contexts} reduce contextual bandits to a non-contextual linear bandit instance by discretizing the space using a fine $(1/T)$-net, whose size scales as $\Omega(T^d)$; this yields an exponential dependence on $d$ which is impractical.
\citet{hanna2023efficient} (which is theoretically suboptimal) address large (possibly infinite) action spaces via optimization-oracle access and discuss (approximate) optimal design primitives; their framework can require many oracle calls, and computing a $1$-approximate G-optimal design is NP-hard, so even constant-factor approximations may still be computationally demanding in practice.
Finally, \citet{zhang2025almost} exhibit a particularly sharp bottleneck concentrated in the \emph{first} batch: their Algorithm~1 executes a (near) G-optimal design policy at every round of batch~1, incurring $\Ocal(Kd^3)$ time per round; since their first-batch length scales as $\widetilde{\Theta}(\sqrt{dT} \wedge T)$, this contributes an additional high-order cost of order $\widetilde{\Ocal}(Kd^{7/2}\sqrt{T} \wedge Kd^3T)$, also leading to a large polynomial dependence on $d$.
The numerical evaluations of these algorithms consistently show the inefficiency mentioned above~(see Table~\ref{table:runtime_results}).

As we focus on \emph{rare parameter updates} setting instead of strictly batched setting, it has not been immediately clear as to whether optimal results can be derived while avoiding the computational inefficiency mentioned above.


\section{Related Work} \label{related_work}

A substantial literature on \emph{limited adaptivity} and \emph{batched} bandits spans settings from multi-armed to linear (contextual) models.
Early work established near-optimal learning with few policy updates in the multi-armed setting~\citep{perchet2016batched,gao2019batched,jin2021almost,jin2021double}, later extended to linear bandits under Gaussian-type features~\citep{han2020sequential} and adversarial features~\citep{esfandiari2021regret}, culminating in algorithms that achieve near-optimal regret with the minimal batch complexity $\Ocal(\log\log T)$~\citep{ren2024optimal,yu2025optimal}.
Although \citet{yu2025optimal} attain minimax-optimal regret in both regimes, their analysis is restricted to non-contextual batched bandits and does not extend to the linear contextual setting.

In linear \emph{contextual} bandits, recent static-grid algorithms~\citep{ruan2021linear, hanna2023contexts, zhang2025almost} achieve optimal regret with only $\Ocal(\log\log T)$ update times under i.i.d.\ contexts.
However, these approaches rely on G-optimal design, do not leverage within-interval context information, and achieve minimax-optimality only in a single regime.
In contrast, \citet{abbasi2011improved} require $\Ocal(d\log T)$ updates, which exceeds the $\Ocal(\log\log T)$ barrier.
Similarly, \citet{karbasi2021parallelizing} requires $\Ocal(K\log T)$ updates and provides a suboptimal regret guarantee of $\widetilde{\Ocal}(d^{3/2}\sqrt{T})$.
Finally, \citet{hanna2023efficient} improve computational efficiency but also obtain a suboptimal regret dependence $\widetilde{\Ocal}(d^{3/2}\sqrt{T})$.

\section{Preliminaries} \label{preliminaries}

\subsection{Notations} \label{notations}
For a set, $|\cdot|$ denotes its cardinality.
For $x\in\RR^d$, $\|x\|_2$ denotes the Euclidean norm, and for a positive definite matrix $H$, we write
$\|x\|_{H}\coloneqq\sqrt{x^\top H x}$.
For a matrix, $\tr(\cdot)$ and $\det(\cdot)$ denote its trace and determinant.
For $n\in\NN$, we write $[n]\coloneqq\{1,\ldots,n\}$, and use $\wedge$ for the minimum operator.
For symmetric matrices $A,B$ of the same dimension, $A\preceq B$ (resp.\ $A\succeq B$) means that $B-A$ (resp.\ $A-B$) is positive semidefinite.
The indicator $\mathbbm{1}_{\{E\}}$ equals $1$ if the event $E$ occurs and $0$ otherwise.
Finally, the natural filtration is
$
\mathcal{F}_t\coloneqq\sigma(\mathcal{A}_1,x_{1,a_1},r_1,\ldots,\mathcal{A}_t,x_{t,a_t},r_t),
$
with $\mathcal{F}_0$ being the trivial $\sigma$-algebra. 

\subsection{Problem Setting: Linear Contextual Bandits with Batched Feedback} \label{problem_setting_CLB}
We study the stochastic linear contextual bandit problem.
At each round $t \in [T]$, the agent observes contextual features of $K$ arms
$\mathcal{A}_t := \{x_{t,1}, \dots, x_{t,K}\} \subseteq \RR^d$
and selects one arm $x_{t,a_t} \in \mathcal{A}_t$, receiving the reward
$r_t = \langle x_{t,a_t}, \theta^* \rangle + \eta_t$,
where $\theta^* \in \RR^d$ is unknown and $\eta_t$ is independent $\sigma$-subgaussian noise. Following the convention in the linear contextual bandit literature~\citep{ruan2021linear,hanna2023efficient, hanna2023contexts, zhang2025almost}, we assume that each feature set $\mathcal{A}_t$ is drawn i.i.d.\ from an unknown distribution $\mathcal{D}$ across rounds, while allowing arbitrary correlation among features presented within the same round.
(In~\cref{rmk:loosing_iid}, we discuss how \cref{BLCE} and its analysis further relax this i.i.d.\ assumption.) The performance of the agent is measured by the \emph{cumulative expected regret} $\mathcal{R}(T) := \mathbb{E}\!\left[\sum_{t=1}^T \bigl(\langle x_t^*, \theta^*\rangle - \langle x_{t,a_t}, \theta^*\rangle \bigr)\right],$
where $x_t^* \in \argmax_{x \in \mathcal{A}_t} \langle x, \theta^* \rangle$ denotes an optimal arm.
We make the following standard assumptions.

\begin{assumption} \label{assum:norm}
\(\|x\|_2 \leq 1\) for all $x \in \mathcal{A}_t$ and $\|\theta^*\|_2 \leq 1$.
\end{assumption}

\begin{assumption} \label{assum:subgauss}
The noise $\eta_t$ is a $1$-subgaussian random variable for all $t \in [T]$.
\end{assumption}

In the batched-feedback setting, the horizon $[T]$ is partitioned into $B$ disjoint \emph{parameter-update intervals},
\[
\{1, \dots , T\}
= \underbrace{[\mathcal{T}_0\!+\!1, \ldots, \mathcal{T}_1]}_{\text{interval 1}}
\,\cup \dots \cup\,
\underbrace{[\mathcal{T}_{B-1}\!+\!1, \ldots, \mathcal{T}_B]}_{\text{interval }B}.
\]
At time $t\in(\mathcal{T}_{\ell-1},\mathcal{T}_\ell]$, the agent may use reward information from all previous intervals (up to $\mathcal{T}_{\ell-1}$), while contexts up to and including time $t$ are always observable.
The rewards for all rounds within the current interval become available only at the end of that interval (time $\mathcal{T}_\ell$).
Importantly, we do \emph{not} require the action rule within an interval to be strictly batched:
the agent may update \emph{reward-free} state variables using observed contexts and chosen actions, while the parameter update that incorporates rewards occurs only at interval boundaries.

\section{Algorithms \& Regret Analysis} \label{algorithms}

\subsection{Note on the Update Schedule}\label{sec:grid_schedule}
Before presenting the algorithms, it is useful to clarify the notion of
\emph{update schedules}, which is distinct from whether an algorithm is
strictly batched. Three types of update schedules are commonly considered:
the \emph{static grid}, where update times are fixed in advance; the
\emph{adaptive grid}, where interval lengths are chosen adaptively; and the
\emph{rare-switching} schedule studied by \citet{abbasi2011improved}, which
allows fully adaptive interaction while limiting the number of parameter
recomputations. Following
\citet{ruan2021linear,hanna2023efficient,hanna2023contexts,zhang2025almost},
we adopt the most restrictive of these choices, a \emph{static grid}, which
yields $B=\Ocal(\log\log T)$ parameter updates.

\subsection{Proposed Algorithms}

\subsubsection{\texttt{BLCE-G} Algorithm}
We now present our proposed algorithms for linear contextual bandits under rare parameter updates. 
We begin with \texttt{BLCE-G}, which stands for \emph{Batched-Feedback Linear Contextual Bandit with Elimination and G-optimal Design}. 
Its pseudocode is provided in Algorithm~\ref{BLCE-G}.
In the first interval, rounds are divided into two phases with ratio $c : (1-c)$. During the first $c$-fraction, arms are sampled according to a near G-optimal design over $\mathcal{A}_t$ (Line 5), while in the remaining $(1-c)$-fraction the algorithm selects the most informative direction with respect to the current
Gram matrix (Line~7). 
To reduce computational and runtime cost, we adopt a relaxation of the G-optimal design, namely the near G-optimal design, which loosens the bound by at most a factor of two. Formally, for any arm set \(\mathcal{A}\subset\RR^d\), there exists a design distribution \(\mathcal{K}_{\mathcal{A}}\) supported on \(\mathcal{A}\) such that
\[
    \max_{x\in \mathcal{A}} x^\top(\mathbb{E}_{z \sim \mathcal{K}_{\mathcal{A}}} [zz^\top])^{-1}x \leq 2d .
\]
As shown in  \cref{cor:near_G_optimal}, such a design can be computed in time $\Ocal(Kd^3)$. For fair comparison, we also account for this cost when evaluating other algorithms that rely on G-optimal design~\citep{ruan2021linear,hanna2023contexts,zhang2025almost}. After each arm pull, the inverse Gram matrix is updated using the Sherman-Morrison formula (reducing the cost from $\Ocal(d^3)$ to $\Ocal(d^2)$) (Line 8). At the end of the first interval, we set $V_1 \coloneqq H_{\mathcal{T}_1}$, compute the regression estimate $\hat\theta_1$, and reinitialize $H_{\mathcal{T}_1}$ for the next interval (Line 9).

For any interval $\ell \geq 2$, the algorithm eliminates suboptimal arms $\ell-1$ times using the past estimates $\hat\theta_1, \dots, \hat\theta_{\ell-1}$, yielding a nested sequence of feasible sets \(\mathcal{A}_{t}^{(1)},\dots,\mathcal{A}_{t}^{(\ell-1)} \) (Line 13). 
In each round $t$, the elimination threshold $\varepsilon_{t,k}$ for $k \in [\ell-1]$ is defined as
\begin{align}
\underset{y \in \mathcal{A}_{t}^{(k-1)}}{\max}& \|y\|_{V_k^{-1}} \bigg(\sqrt{2\log\Big(|\mathcal{A}_{t}^{(k-1)}|(B-1)T^2\Big)} + \sqrt{\lambda} \notag\\
&\wedge \, 2\sqrt{\log\bigg(\frac{2^{6d-5}\pi d(B-1)^2}{15^{d-1}/T^2}\bigg)} +2\sqrt{\lambda} \bigg).
\label{eq:epsilon-tk}
\end{align}
Within interval $\ell \geq 2$, the rounds are partitioned in the ratio $c^2:c(1-c):(1-c)$. In the first $c^2$-fraction, arms are sampled according to a near G-optimal design over $\mathcal{A}_t^{(\ell-1)}$ (Line 15). In the next $c(1-c)$-fraction, the algorithm selects the most informative direction relative to the Gram matrix (Line 17). In the final $(1-c)$-fraction, arms are chosen greedily with respect to the latest estimate (Line 19). As before, after each pull the Gram matrix is updated (Line 20), and at the end of the interval we set $V_\ell \coloneqq H_{\mathcal{T}_\ell}$, compute $\hat\theta_\ell$, and reinitialize $H_{\mathcal{T}_\ell}$ for the next interval (Line 21).

\subsubsection{\texttt{BLCE} Algorithm}

We next introduce our second algorithm, \texttt{BLCE}, which stands for
\emph{Batched-Feedback Linear Contextual Bandit with Elimination}. Its
pseudocode is given in Algorithm~\ref{BLCE}. \texttt{BLCE} follows the
same overall elimination-based framework as \texttt{BLCE-G}, including
the use of the elimination threshold in \eqref{eq:epsilon-tk}, but
removes the G-optimal design phase. The corresponding rounds are
instead allocated to uncertainty-driven exploration. 
To our best knowledge,
\texttt{BLCE} is the first linear contextual bandit algorithm with
\emph{rare parameter updates} to achieve minimax-optimal regret without
relying on G-optimal design, which has traditionally played a central
role in strictly batched static-grid methods
\citep{ruan2021linear,hanna2023contexts,zhang2025almost}.
Notably, both \texttt{BLCE-G} and \texttt{BLCE} avoid enforcing any fixed choice of $c \in (0,1)$, thereby providing theoretical guarantees together with practical flexibility in balancing exploration and exploitation. 

\begin{algorithm}[tb]
	\caption{\texttt{BLCE-G}}
        \label{BLCE-G}
	\begin{algorithmic}[1]

        \Statex {\bfseries Input:} Horizon $T$; regularization parameter $\lambda$;
        within-interval allocation rate $c$

        \vspace{-0.2cm}
        \Statex \hrulefill
        
 \State \textbf{Set static schedule:}
  $\mathcal{T}_1 \leftarrow  \Big\lceil\frac{\sqrt{T}}{\log_2\log_2T}\Big\rceil + 1$;\\
 $\mathcal{T}_\ell \leftarrow  \Big(\mathcal{T}_{\ell-1} + \Big\lceil \frac{T^{1-2^{-\ell}}}{\log_2\log_2T}\Big\rceil +  2 \Big) \wedge T$ for $\ell \geq 2$;\\
 $B \leftarrow \min\{\ell\geq 1:\mathcal{T}_\ell=T\}$;
 \State \textbf{Initialize: }
 $H_0 \gets \lambda I$;
 
 \For{$t \gets 1,2,\dots,\mathcal{T}_1$}

 \If{$t \leq \big\lceil c\sqrt{T}/\log_2\log_2T \big\rceil$}
 
 
 \State Pull arm $x_{t,a_t} \sim \pi_{G'}(\mathcal{A}_t)$;

 \Else
 

 \State Pull arm $x_{t,a_t} \in \arg\max_{x \in \mathcal{A}_t}\|x\|_{H_{t-1}^{-1}}$;

 \EndIf

 \State $H_t^{-1} \gets H_{t-1}^{-1} - H_{t-1}^{-1}x_{t,a_t}x_{t,a_t}^\top H_{t-1}^{-1}/(1 + x_{t,a_t}^\top H_{t-1}^{-1}x_{t,a_t})$;
 
 \EndFor
 
 \State $V_1^{-1} \gets H_{\mathcal{T}_1}^{-1}$, $\hat{\theta}_1 \gets V_1^{-1}\sum_{t=1}^{\mathcal{T}_1} r_tx_{t,a_t}$, $H_{\mathcal{T}_1} \gets \lambda I$;
 
 \For{$\ell \gets 2,\dots,B$}
 
 \For{$t \gets \mathcal{T}_{\ell-1}+1,\dots,\mathcal{T}_{\ell}$}
 
 \For{$k \gets 1,\dots,\ell-1 $}

 \State $x_{t}^{(k)} \gets \arg\max_{x \in \mathcal{A}_t^{(k-1)}}\langle x, \hat{\theta}_{k}\rangle$,  $\mathcal{A}_t^{(k)} \gets \left\{x \in \mathcal{A}_t^{(k-1)} \,  \bigg| \, \langle \hat{\theta}_k, x_{t}^{(k)} - x \rangle \leq 2\varepsilon_{t,k} \right\}$,
where $\varepsilon_{t,k}$ is given by Eq.\eqref{eq:epsilon-tk};

 \EndFor
 
 \If{$t \leq \mathcal{T}_{\ell-1} + \big\lceil c^2T^{1-2^{-\ell}}/\log_2\log_2T \big\rceil$}
 
 \State Pull $x_{t,a_t} \sim \pi_{G'}(\mathcal{A}_t^{(\ell-1)})$;

 \ElsIf{$t \leq \mathcal{T}_{\ell-1} + \big\lceil c^2T^{1-2^{-\ell}}/\log_2\log_2T \big\rceil + \big\lceil c(1-c)T^{1-2^{-\ell}}/\log_2\log_2T \big\rceil$}

 \State Pull $x_{t,a_t} \in \arg\max_{x \in \mathcal{A}_t^{(\ell-1)}}\|x\|_{H_{t-1}^{-1}}$;
 
 \Else 

 \State Pull $x_{t,a_t} \in \arg\max_{x \in \mathcal{A}_t^{(\ell-1)}}\langle x, \hat{\theta}_{\ell-1}\rangle$;
 
 \EndIf
 
 \State $H_t^{-1} \gets H_{t-1}^{-1} - H_{t-1}^{-1}x_{t,a_t}x_{t,a_t}^\top H_{t-1}^{-1}/(1 + x_{t,a_t}^\top H_{t-1}^{-1}x_{t,a_t})$;
 
 \EndFor
 
 \State $V_\ell^{-1} \gets H_{\mathcal{T}_\ell}^{-1}$, $\hat{\theta}_{\ell} \gets V_\ell^{-1} \sum_{t=\mathcal{T}_{\ell-1}+1}^{\mathcal{T}_\ell} r_tx_{t,a_t}$, $H_{\mathcal{T}_\ell} \gets \lambda I$;
 
 \EndFor
 
\end{algorithmic}
\end{algorithm}

\begin{algorithm}[tb]
	\caption{\texttt{BLCE}}
        \label{BLCE}
	\begin{algorithmic}[1]
 \Statex {\bfseries Input:} Horizon $T$; 
regularization parameter $\lambda$;
within-interval allocation rate $c$;

\vspace{-0.2cm}
 \Statex \hrulefill
 

 \State \textbf{Set static schedule:} $\mathcal{T}_1 \leftarrow
\Big\lceil
\frac{\sqrt{T}}{\log_2\log_2T}
\Big\rceil$;\\
$\mathcal{T}_\ell \leftarrow
\Big(
\mathcal{T}_{\ell-1}
+
\Big\lceil
\frac{T^{1-2^{-\ell}}}{\log_2\log_2T}
\Big\rceil
+1
\Big)\wedge T$
for $\ell\geq2$;\\
$B\leftarrow
\min\{\ell\geq1:\mathcal{T}_\ell=T\}$;

 \State \textbf{Initialize:} $H_0 \gets \lambda I$, $b_0 \gets \mathbf{0}$;

 \For{$t \gets 1,2,\dots,\mathcal{T}_1$}

 \State Pull arm $x_{t,a_t} \in \arg\max_{x \in \mathcal{A}_t}\|x\|_{H_{t-1}^{-1}}$;

 \State $H_t^{-1} \gets H_{t-1}^{-1} - H_{t-1}^{-1}x_{t,a_t}x_{t,a_t}^\top H_{t-1}^{-1}/(1 + x_{t,a_t}^\top H_{t-1}^{-1}x_{t,a_t})$;

 \EndFor

 \State $V_1^{-1} \gets H_{\mathcal{T}_1}^{-1}$, $\hat{\theta}_1 \gets V_1^{-1}\sum_{t=1}^{\mathcal{T}_1} r_tx_{t,a_t}$, $H_{\mathcal{T}_1} \gets \lambda I$;
 
 \For{$\ell \gets 2,\dots,B$}
 
 \For{$t \gets \mathcal{T}_{\ell-1}+1,\dots,\mathcal{T}_{\ell}$}
 
 \For{$k \gets 1,\dots,\ell-1 $}

 \State $x_t^{(k)} \gets \arg\max_{x \in \mathcal{A}_t^{(k-1)}}\langle x, \hat{\theta}_{k}\rangle$, $\mathcal{A}_t^{(k)} \gets \left\{x \in \mathcal{A}_t^{(k-1)} \,  \bigg| \, \langle \hat{\theta}_k, x_t^{(k)} - x \rangle \leq 2\varepsilon_{t,k} \right\}$,
 where $\varepsilon_{t,k}$ is given by Eq.\eqref{eq:epsilon-tk};
 
 \EndFor

 \If{$t \leq \mathcal{T}_{\ell-1} + \big\lceil cT^{1-2^{-\ell}}/\log_2\log_2T \big\rceil$}

 \State Pull $x_{t,a_t} \in \arg\max_{x \in \mathcal{A}_t^{(\ell-1)}}\|x\|_{H_{t-1}^{-1}}$;

 \Else

 \State Pull $x_{t,a_t} \in \arg\max_{x \in \mathcal{A}_t^{(\ell-1)}}\langle x, \hat{\theta}_{\ell-1}\rangle$;

 \EndIf
 
 \State $H_t^{-1} \gets H_{t-1}^{-1} - H_{t-1}^{-1}x_{t,a_t}x_{t,a_t}^\top H_{t-1}^{-1}/(1 + x_{t,a_t}^\top H_{t-1}^{-1}x_{t,a_t})$;
 
 \EndFor
 
 \State $V_\ell^{-1} \gets H_{\mathcal{T}_\ell}^{-1}$, $\hat{\theta}_{\ell} \gets V_\ell^{-1} \sum_{t=\mathcal{T}_{\ell-1}+1}^{\mathcal{T}_\ell} r_tx_{t,a_t}$, $H_{\mathcal{T}_\ell} \gets \lambda I$;
 
 \EndFor
 
\end{algorithmic}
\end{algorithm}

\subsection{Regret Analysis for Linear Contextual Bandits} \label{regret_analysis}

We now present the regret guarantees for the proposed algorithms,
starting with \texttt{BLCE-G}.

\begin{theorem} [Regret of  \texttt{BLCE-G}] \label{thm:BLCE_g}
Suppose that Assumptions~\ref{assum:norm} and \ref{assum:subgauss} hold and that the context sets
$\{\mathcal{A}_t\}_{t=1}^T$ are drawn i.i.d.\ from an unknown distribution
$\mathcal{D}$. 
Run \texttt{BLCE-G} for $T$ rounds, using $\lambda = \log (dT)$, and any fixed constant $c\in(0,1)$. 
    The worst-case cumulative regret satisfies
    {\small
    \begin{align*}
        \mathcal{R}(T) \!&= \Ocal\Big(\!\sqrt{dT}(\sqrt{\log(KT)} \wedge \!\sqrt{d+\log T})\sqrt{\log d\log\log T}\Big) \\
        \!&= \tilde{\Ocal} \Big(\!\sqrt{dT\log K} \wedge d\sqrt{T}\Big) \; .
    \end{align*}
    }
\end{theorem}

$\textbf{Discussion of \cref{thm:BLCE_g}.}$ \label{dis:thm1}
Theorem~\ref{thm:BLCE_g} shows that \texttt{BLCE-G} achieves minimax-optimal regret for fully sequential linear contextual bandits while performing only $\Ocal(\log\log T)$ \emph{parameter updates} under a static schedule.
A notable feature of this bound is that it covers both regimes. In the small-$K$ regime $(K \leq \Ocal(e^d))$, the regret scales as $\widetilde\Ocal(\sqrt{dT\log K})$, and in the large-$K$ regime $(K \geq \Omega(e^d))$, as $\widetilde\Ocal(d\sqrt{T})$. 
Thus, \texttt{BLCE-G} provides the tightest known guarantees among linear contextual bandit algorithms with \emph{rare parameter updates} and, to our knowledge, is the first algorithm to match the minimax lower bounds in both regimes.
\vspace{0.2cm}

\begin{theorem} [Regret of  \texttt{BLCE}] \label{thm:BLCE}
    Suppose that Assumptions~\ref{assum:norm} and \ref{assum:subgauss} hold and that the context sets
$\{A_t\}_{t=1}^T$ are drawn i.i.d.\ from an unknown distribution
$\mathcal{D}$. 
Run \texttt{BLCE} for $T$ rounds, using $\lambda = 1$, and any fixed constant $c\in(0,1)$.
    The worst-case cumulative regret satisfies
    {\small
    \begin{align*}
        \mathcal{R}(T) \!&= \Ocal\Big(\!\sqrt{dT}(\sqrt{\log(KT)} \wedge \!\sqrt{d+\log T})\sqrt{\log T\log\log T}\Big) \\
        \!&= \tilde{\Ocal}\Big(\!\sqrt{dT\log K} \wedge d\sqrt{T}\Big) \; .
    \end{align*}
    }
\end{theorem}

$\textbf{Discussion of \cref{thm:BLCE}.}$ Theorem~\ref{thm:BLCE} establishes that \texttt{BLCE} also achieves minimax-optimal regret with only $\Ocal(\log\log T)$ parameter updates. 
Its key distinction lies in computation: by removing the G-optimal design step, \texttt{BLCE} significantly reduces both complexity and runtime while retaining minimax optimality. As shown in \cref{table:comparison}, this makes \texttt{BLCE} the first linear contextual bandit algorithm with \emph{rare parameter updates} to combine minimax optimality across both regimes with \emph{no reliance on G-optimal design}.

\begin{remark}
    \label{rmk:loosing_iid}
    While prior work on static-grid linear contextual bandits typically assumes i.i.d.\ contexts~\citep{ruan2021linear,hanna2023efficient,hanna2023contexts,zhang2025almost}, 
we show that this assumption can be relaxed to the following batch-wise conditions (for any $\ell \ge 1$)
    \\\\
    {\footnotesize
    (1) $\mathrm{Law}(\mathcal{A}_t | \mathcal{F}_{\mathcal{T}_{\ell-1}}) \sim \mathcal{D}_{\ell-1}$,  $\forall t \!\in\![\mathcal{T}_{\ell-1}+1, \mathcal{T}_{\ell+1}]$
    \\
    (2) $\mathcal{A}_t \! \indep \! \{\mathcal{A}_s, x_{s,a_s}, r_s\}_{s=\mathcal{T}_{\ell-1}+1}^{\mathcal{T}_\ell} | \mathcal{F}_{\mathcal{T}_{\ell-1}}$, $\forall t \!\in \! [\mathcal{T}_{\ell}+1, \mathcal{T}_{\ell+1}]$
    \\
    (3) $\mathcal{A}_s \! \indep \! \{\mathcal{A}_u, x_{u,a_u}, r_u\}_{u=\mathcal{T}_{\ell-1}+1}^{s-1} | \mathcal{F}_{\mathcal{T}_{\ell-1}}$, $\forall s \! \in \! (\mathcal{T}_{\ell-1}+1, \mathcal{T}_{\ell}]$
    }
    \\\\
    Given the history $\mathcal{F}_{\mathcal{T}_{\ell-1}}$, condition (1) requires intervals $\ell$ and $\ell\!+\!1$ to share the same conditional context law; condition (2) enforces that contexts are conditionally independent of the contexts/actions/rewards realized in previous interval; and condition (3) imposes within-interval conditional independence of each context from earlier within-interval observations. 
    These assumptions are strictly weaker than full i.i.d.\,, requiring only (i) equality of the conditional context law across consecutive intervals and (ii) conditional independence across and within intervals. This relaxation offers greater modeling flexibility while preserving the guarantees established in \cref{proof:thm2}.
\end{remark}

\subsection{Time-Complexity of Algorithms}
The computational cost of \texttt{BLCE-G} arises primarily from the near
G-optimal design step and arm elimination. By
\cref{cor:near_G_optimal}, each near G-optimal design call costs
$\Ocal(Kd^3)$. During arm elimination, computing
$\varepsilon_{t,k}$ costs $\Ocal(Kd^2)$ per elimination step; since there are
at most $\Ocal(\log\log T)$ elimination steps per round, the total cost is
$\Ocal(Kd^2T\log\log T)$. Therefore, the overall time complexity of
\texttt{BLCE-G} is
$\Ocal\!\bigl(Kd^2T(d+\log\log T)\bigr)$.
For \texttt{BLCE}, the near G-optimal design step is removed, so arm
elimination is the dominant cost, yielding a total time complexity of
$\Ocal(Kd^2T\log\log T)$.

\section{Extensions to Generalized Linear Contextual Bandits} \label{GLM_section}

In this section, we extend our results to \textit{generalized linear contextual bandits}. The formal definition of the generalized linear contextual bandit problem with batched feedback is provided in~\cref{problem_setting_GLB}.

\subsection{Proposed Algorithm: \texttt{BGLE}}
We propose \texttt{BGLE} (\emph{Batched-Feedback Generalized Linear Contextual Bandit with Elimination}), whose pseudocode is given in Algorithm~\ref{bgle}. The algorithm takes a regularization parameter $\lambda$, the known reward-support bound $R$, the known
parameter-norm bound $S$, and an allocation rate $c\in(0,1)$ as inputs. For the guarantee
in \cref{thm:bgle}, we set
$\lambda=R^2(d+\log T)$. Both $R$ and $S$ enter the
confidence radius $\beta(\lambda)$ and the weighting coefficients
$\alpha_{t,k}(\lambda)$, while the elimination
thresholds $\varepsilon'_{t,k}(\lambda)$ depend on them through
$\beta(\lambda)$. To extend our approach to the generalized linear setting, we build on the structure of Algorithm~\ref{BLCE}. In the first interval, the algorithm repeatedly pulls the most informative direction with respect to the current Gram matrix (Line 4) and updates its inverse via the Sherman–Morrison (Line 5). At the interval boundary, we set $V_1 \coloneqq H_{\mathcal{T}_1}$, compute the MLE $\hat\theta_1$ for the per–round log-loss $\ell_t(\theta)\!=\!m(\langle x_{t,a_t},\theta\rangle)-r_t\langle x_{t,a_t},\theta\rangle$, and reinitialize $H_{\mathcal{T}_1}$ for the next interval (Line 6). For each interval $\ell \geq 2$, the Gram matrix is \emph{weighted} by $\alpha_{t,\ell-1}(\lambda)\,\dot\mu(\langle x_{t,a_t},\hat\theta_{\ell-1}\rangle)$ (Line 17),
where
\begin{align*}
\alpha_{t,k}(\lambda) &= \exp(-2RS)\,\mathbbm{1}_{\{k=1\}} \\
&+
\exp(-R(2S \wedge \|x_{t,a_t}\|_{V_k^{-1}}\beta(\lambda)))\,\mathbbm{1}_{\{k\geq 2\}}\,.
\end{align*}

Beginning at interval $\ell \geq 3$, the algorithm performs $\ell-2$ elimination rounds using the estimates $\hat\theta_2,\dots,\hat\theta_{\ell-1}$, yielding nested feasible sets \(\mathcal{A}_{t}^{(2)},\dots,\mathcal{A}_{t}^{(\ell-1)} \) (Lines 11-12). Since no elimination is conducted with $\hat\theta_1$, we set $\mathcal{A}_t=\mathcal{A}_t^{(0)}=\mathcal{A}_t^{(1)}$. The elimination threshold $\varepsilon_{t,k}'(\lambda)$ for $k \in [\ell-1]\setminus\{1\}$ is defined as 
\begin{align*}
 &\underset{y \in \mathcal{A}_{t}^{(k-1)}}{\max}\|y\|_{V_k^{-1}} \!\Big(24RS(\!\sqrt{d+\log T} +\! R(d+\log T)/\sqrt{\lambda}) \\
 & + 2S\sqrt{\lambda}\Big) \; ,
\end{align*}
which, under the choice $\lambda = R^2(d + \log T)$, simplifies to $\max_{y \in \mathcal{A}_{t}^{(k-1)}}\|y\|_{V_k^{-1}}(50RS \sqrt{d+\log T})$.
Within each interval $\ell \geq 2$, the action selection strategy follows that of \texttt{BLCE}, splitting the interval in the ratio $c:(1-c)$ between exploration and exploitation. The key difference is that arm selection is based on the \emph{weighted} Gram matrix (Lines 14 and 16). At the end of interval $\ell$, we set $V_\ell \coloneqq H_{\mathcal{T}_\ell}$, compute the MLE $\hat\theta_\ell$ for $\ell_t(\theta)$, and reinitialize $H_{\mathcal{T}_\ell}$ for the next interval (Line 18).

\subsection{Regret Analysis of $\texttt{BGLE}$}\label{sec:regret-GLM}


To state our regret guarantee for generalized linear contextual bandits
and facilitate comparison with prior work, we introduce several quantities
that characterize the nonlinearity of a given problem instance. 
For any
arm set $\mathcal{A}$, let
$x^* \in \arg\max_{x\in\mathcal{A}}
\mu(\langle x,\theta^*\rangle)$
denote an optimal arm, and define
\begin{gather*}
    \kappa
    \coloneqq
    \!\!\! \sup_{\mathcal{A}\in\operatorname{supp}(\mathcal{D})}
    \sup_{x\in\mathcal{A}}
    \frac{1}{\dot\mu(\langle x,\theta^*\rangle)},
    \enskip
    \hat\kappa
    \coloneqq
    \frac{1}{
    \mathbb{E}_{\mathcal{A}\sim\mathcal{D}}
    [\dot\mu(\langle x^*,\theta^*\rangle)]
    },
    \\
    R_{\dot\mu}
    \coloneqq\!\!
    \sup_{\mathcal{A}\in\operatorname{supp}(\mathcal{D})}\!
    \dot\mu(\langle x^*,\theta^*\rangle).
\end{gather*}
Here, $\kappa$ is the instance-specific worst-case inverse curvature over
all feasible arms and can become large when the corresponding natural
parameters approach a saturation region. The quantity $\hat\kappa$ is the
reciprocal of the average curvature at an optimal arm, while
$R_{\dot\mu}$ denotes the largest curvature attained at an optimal arm.

These quantities do not impose additional assumptions, but simply summarize
instance-dependent properties under the standing conditions. Since
$\|x\|_2\leq 1$ and $\|\theta^*\|_2\leq S$, all feasible natural parameters
lie in $[-S,S]$. For any nondegenerate link, $\dot\mu(z)>0$ on this interval,
so $\kappa$, $\hat\kappa$, and $R_{\dot\mu}$ are all well-defined and finite.
We now state the regret guarantee for \texttt{BGLE}.
\vspace{0.2cm}

\begin{theorem} [Regret of  \texttt{BGLE}] \label{thm:bgle}
    Suppose that
Assumptions~\ref{assum:GLM_norm} and \ref{assum:GLM_diff} (stated in Appendix~\ref{sec:GLM_theorem_proof}) hold, that the context
sets $\{\mathcal{A}_t\}_{t=1}^T$ are drawn i.i.d.\ from an
unknown distribution $\mathcal{D}$, and that rewards are
almost surely supported on $[0,R]$.
Assume without loss of generality that the MLE computed at each interval boundary
satisfies $\|\hat{\theta}_\ell\|_2\leq S$ (otherwise, use the
projected estimator described in \cref{problem_setting_GLB}).
Run \texttt{BGLE} for $T$ rounds, using the known bounds $R$ and $S$,
$\lambda=R^2(d+\log T)$, and any fixed constant
$c\in(0,1)$, with the static schedule specified in
Algorithm~\ref{bgle}. The worst-case cumulative regret satisfies
    {\small
    \begin{align*}
        &\mathcal{R}(T) = \Ocal \Big(RS\sqrt{d(d+\log T)T\log T\log\log T / \hat\kappa}\Big)  
        \\
        & + \Ocal \bigg(\!\!\Big(\!R^2Se^{8RS}d(d+\log T)\log T\log\log T +  \frac{R}{\log\log T} \!\Big)T^{\frac{1}{3}} \!\bigg)  
        \\
        &=\widetilde{\Ocal}\left(RSd\sqrt{T}/\sqrt{\hat\kappa}\right)+\widetilde\Ocal \left((R^{2}Se^{8RS}d^2+R)T^{1/3} \right) \,.
    \end{align*}
    }
\end{theorem}

\vspace{-0.2cm}

\paragraph{Discussion of Theorem~3.} 

In terms of the leading-in-$T$ term in \cref{thm:bgle}, the regret bound of
\texttt{BGLE} has the same dependence on the key problem parameters $d$ and
$T$ as the regret bounds of its linear counterparts in
Theorems~\ref{thm:BLCE_g} and~\ref{thm:BLCE}.
For any nondegenerate link considered here, $\dot\mu(0)>0$.
Moreover, by Assumption~\ref{assum:GLM_norm} and
Lemma~\ref{lem:GLM_concordance},
$\kappa \leq e^{RS}/\dot\mu(0)$.
Thus, the exponential dependence on $RS$ used in the second term of
\cref{thm:bgle} corresponds to a worst-case uniform control of the inverse curvature over all feasible natural parameters. 
An instance-dependent version of the analysis can instead retain the
curvature quantities explicitly, resulting in a regret bound with
$\kappa$-dependence in the second term. 
Note that the leading term in \cref{thm:bgle} scales as
$1/\sqrt{\hat\kappa}$, where $\hat\kappa$ is the instance-dependent
average optimal-arm inverse curvature. Thus, the bound improves as
$\hat\kappa$ increases, while it does not depend on $\kappa$.

Moreover, \texttt{BGLE} uses only
$\Ocal(\log\log T)$ parameter updates and is, to our knowledge,
the first generalized linear contextual bandit algorithm with
rare parameter updates to combine near-optimal guarantees
with no reliance on G-optimal design. Finally, \texttt{BGLE}
inherits substantially lower computational complexity by
building on the \texttt{BLCE} framework.

\begin{remark} \label{rmk:time_complexity}
    The total time complexity of \texttt{BGLE} is $\Ocal(Kd^2T\log\log T + \mathcal{C}_{\mathrm{opt}}^{\mathrm{tot}})$, where the first term comes from the \texttt{BLCE}, and the second from computing the MLE at interval boundaries (see~\cref{table:GLM_comparison}).
\end{remark}

\section{Numerical Experiments} \label{sec:numerical_experiments}
We evaluate the performance of \texttt{BLCE-G} and \texttt{BLCE} over a horizon of $T=10{,}000$ across $10$ independent runs. At each round, $K$ arms are sampled i.i.d.\ from a \(d\)-dimensional uniform distribution, and the parameter \(\theta^*\) is drawn from a $d$-dimensional normal distribution. We consider four $(K,d)$ pairs: $(1000,5)$ and $(5000,10)$, representing the large-$K$ regime, and $(50,20)$ and $(100,30)$, representing the small-$K$ regime. For comparison, we benchmark against state-of-the-art algorithms: \texttt{RS-OFUL} \citep{abbasi2011improved}, 
\texttt{BatchLinUCB-DG} \citep{ruan2021linear}, \texttt{SoftBatch} \citep{hanna2023contexts}, and \texttt{BatchLearning} \citep{zhang2025almost}. Hyperparameters are set consistently with theory, ensuring all choices satisfy the required conditions: \texttt{BLCE-G} and \texttt{BLCE} use within-interval allocation rate $c = 0.5$; \texttt{RS-OFUL} uses switching parameter $C=1$; and \texttt{SoftBatch} employs discretization parameter $q = 1/(8\sqrt{d})$. Algorithms requiring G-optimal design are implemented using the same near G-optimal routine. Due to the substantial computational overhead reported in \cref{table:comparison}, we omit regret plots for the method of \citet{hanna2023contexts}. 

\begin{figure*}[!ht]
\centering
\begin{subfigure}{\textwidth}
\includegraphics[width=\textwidth]{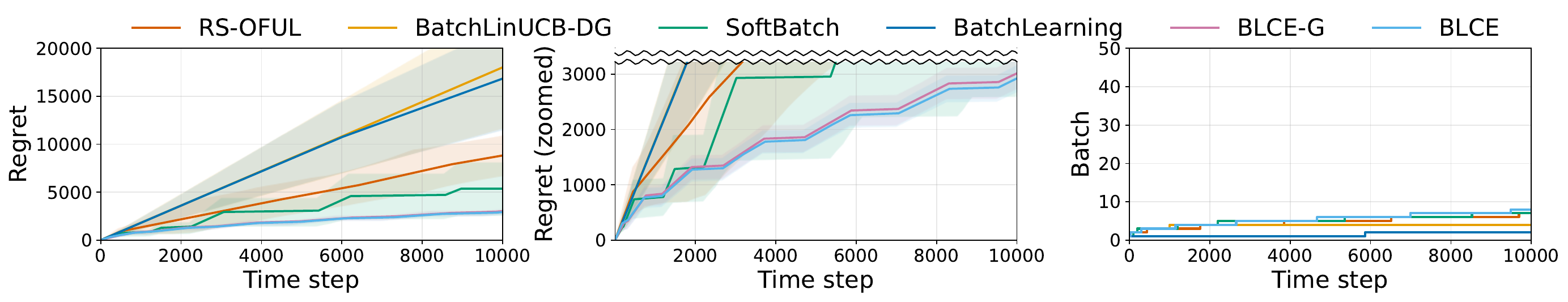}
\centering
{\small \raisebox{0.5cm}{(a) $K=1000, d=5$} \par}
\end{subfigure}\vspace{-1em}
\begin{subfigure}{\textwidth}
\includegraphics[width=\textwidth]{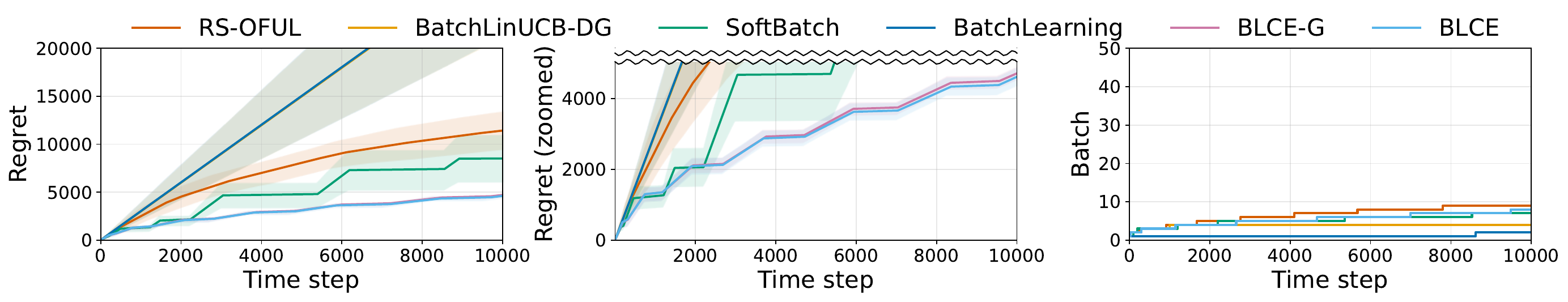}
\centering
{\small \raisebox{0.5cm}{(b) $K=5000, d=10$} \par}
\end{subfigure}\vspace{-1em}
\begin{subfigure}{\textwidth}
\includegraphics[width=\textwidth]{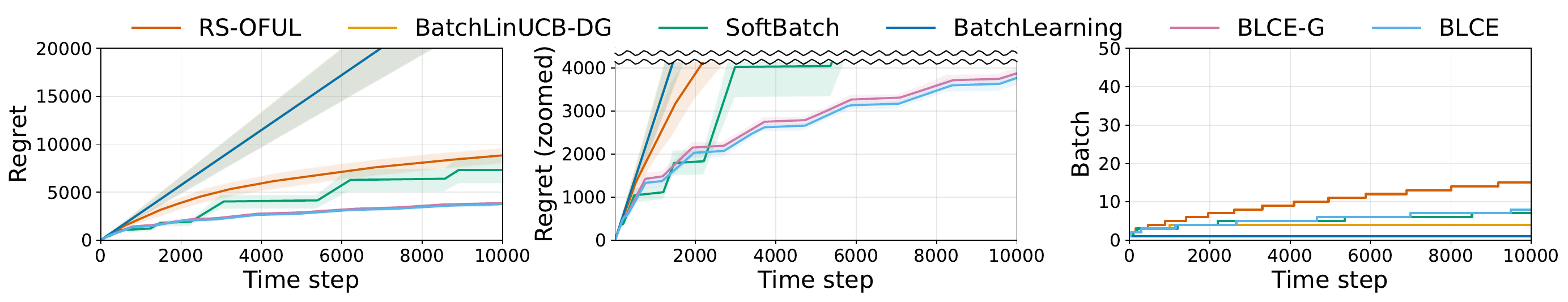}
\centering
{\small \raisebox{0.5cm}{(c) $K=50, d=20$} \par}
\end{subfigure}\vspace{-1em}
\begin{subfigure}{\textwidth}
\includegraphics[width=\textwidth]{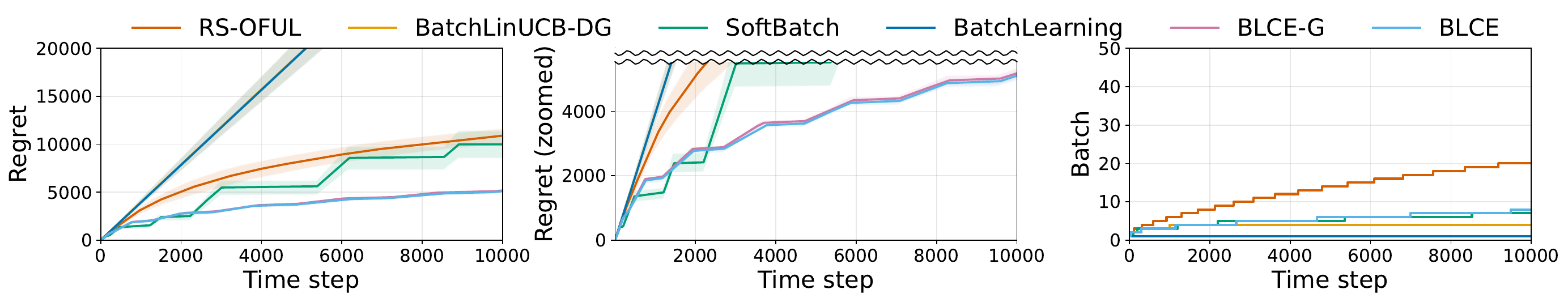}
\centering
{\small \raisebox{0.5cm}{(d) $K=100, d=30$} \par}
\end{subfigure}\vspace{-2em}
\setcounter{figure}{0}
\captionof{figure}{Regret, zoomed-in regret, and update (batch) complexity over time for different values of $K$ and $d$.}
\label{uniform figure}
\end{figure*}

\begin{table*}[!ht]  
\caption{Average runtime (seconds) over 10 runs.}
\label{table:runtime_results}
\centering
\begin{footnotesize}
\begin{tabular}{lrrrrrrr}
\toprule
& \multicolumn{2}{c}{\textbf{Suboptimal algorithms}} 
& \multicolumn{5}{c}{\textbf{Optimal algorithms}} \\
\cmidrule(lr){2-3}\cmidrule(lr){4-8}
\textbf{$(K,d)$} 
& \texttt{RS-OFUL} & \texttt{SoftBatch} 
& \texttt{BatchLinUCB-DG} & \citet{hanna2023contexts} & \texttt{BatchLearning} & \texttt{BLCE-G} & \texttt{BLCE} \\
\midrule
$(1000,5)$   & $0.85$ & $4.18$ & $290.87$ & Exponential & $166.17$ & $23.40$ & $5.91$ \\
$(5000,10)$  & $4.15$ & $15.17$ & $1300.01$ & Exponential & $621.09$ & $40.27$ & $12.83$ \\
$(50,20)$    & $0.42$ & $3.74$ & $1031.66$ & Exponential & $45.85$ & $2.26$ & $1.06$ \\
$(100,30)$   & $0.61$ & $5.50$ & $2987.07$ & Exponential & $77.01$ & $3.70$ & $1.62$ \\
\bottomrule
\end{tabular}
\end{footnotesize}
\end{table*}

We present three types of figures: (i) the average cumulative regret (solid line) with its standard deviation (shaded region) over 10 runs, (ii) zoomed-in views of regret curves to highlight differences between \texttt{BLCE-G} and \texttt{BLCE}, and (iii) the average parameter-update complexity across $10$ runs, illustrating the frequency of parameter updates (for consistency with prior work, we label the horizontal axis using the term ``batch''). As shown in Figure~\ref{uniform figure}, both \texttt{BLCE-G} and \texttt{BLCE} consistently outperform all baselines in both large-$K$ and small-$K$ regimes, achieving the lowest regret with greater stability. Runtime comparisons in \cref{table:runtime_results} further show that our methods incur substantially lower computational cost; in particular, \texttt{BLCE}, which eliminates G-optimal design entirely, achieves the fastest runtime among optimal algorithms, comparable even to suboptimal baselines. Overall, these results demonstrate that \texttt{BLCE-G} and \texttt{BLCE} combine minimax-optimal regret with practical efficiency. For generalized linear contextual bandits, \texttt{BGLE} similarly outperforms the baseline, achieving lowest regret, stable performance, and reduced computation. See \cref{exp:GLM} for details.

\section{Conclusion}
We studied contextual bandits under rare parameter updates, separating the cost of reward-dependent retraining from lightweight reward-free processing of contexts within an interval.
For linear contextual bandits, \texttt{BLCE-G} achieves the tightest known rare-update regret guarantees across both small-$K$ and large-$K$ regimes, while \texttt{BLCE} removes the G-optimal design step and preserves minimax-optimality with substantially lower computational cost. 
For generalized linear contextual bandits, \texttt{BGLE} extends the same design principle, achieving near-optimal regret with only
$\Ocal(\log\log T)$ parameter updates. 
Together, these results suggest that limited reward adaptivity need not force computationally expensive pre-planned exploration, provided that within-interval context information can be utilized.



\section*{Acknowledgments}
This work was supported by the National Research Foundation of Korea~(NRF) grant and the Institute of Information \& communications Technology Planning \& Evaluation~(IITP) grant both funded by the Korea government~(MSIT) (No. RS-2022-NR071853, RS-2023-00222663, RS-2025-25418056, RS-2025-25463302, RS-2026-25507282).

\section*{Impact Statement}

This paper presents work whose goal is to advance the field of Machine
Learning. There are many potential societal consequences of our work, none
which we feel must be specifically highlighted here.

\bibliography{references}
\bibliographystyle{icml2026}

\newpage
\appendix
\onecolumn
\textbf{\LARGE Appendix}
\section{Proof of Theorem 1}
\begin{lemma} \label{lem:concentration} \citep{yu2025optimal}
Let $V_\ell$ be the Gram matrix and $\hat{\theta}_\ell$ be the least squares estimator obtained from the contexts in the $\ell$-th interval $(\ell\geq 1)$. Then, for any $x \in \mathbb{R}^d$ and $0<\delta<1$, the following inequality holds with probability at least $1-\delta$
\begin{equation*}
    |\langle x, \hat{\theta}_{\ell} - \theta^{*} \rangle | \leq \left(\sqrt{2 \log \bigg(\frac{1}{\delta}\bigg) } +\sqrt{\lambda}\right)\|x\|_{V_{\ell}^{-1}}\,.
\end{equation*}
\end{lemma}

\begin{lemma} \label{lem:estimate_difference} \citep{yu2025optimal}
Let $V_\ell$ be the Gram matrix and $\hat{\theta}_\ell$ be the least squares estimator obtained from the contexts in the $\ell$-th interval $(\ell\geq 1)$. Then, for any $0 < \zeta, \delta < 1$ and $d \geq 2$, the following inequality holds with probability at least $1-\delta$
\[
\|\hat{\theta}_{{\ell}}-\theta^{*} \|_{V_{{\ell}}} \leq \frac{\sqrt{(d-1)\log(\frac{4}{4\zeta^2-\zeta^4})+2\log(\frac{\sqrt{2\pi d}}{\delta})} +\sqrt{\lambda}}{1-\zeta}\; .
\]
\end{lemma}
\noindent
\cref{lem:concentration} and \cref{lem:estimate_difference} were originally proved in the linear bandit setting. However, the proofs do not rely on the non-contextual assumption that the feature vectors of arms remain fixed. Therefore, the results can be directly applied to the linear contextual bandit setting as well.

\begin{corollary} \label{cor:estimate_difference}
Let $V_\ell$ be the Gram matrix and $\hat{\theta}_\ell$ be the least squares estimator obtained from the contexts in the $\ell$-th interval $(\ell\geq 1)$. Then, for any $0 < \delta < 1$ and $d \geq 2$, the following inequality holds with probability at least $1-\delta$
\[
\|\hat{\theta}_{{\ell}}-\theta^{*} \|_{V_{{\ell}}} \leq 2\sqrt{\log\bigg(\frac{2^{6d-5}\pi d}{15^{d-1}\delta^2}\bigg)} +2\sqrt{\lambda}\; .
\]
\end{corollary}

\begin{proof}
Substituting $\zeta = 0.5$ into \cref{lem:estimate_difference} yields the desired result.
\end{proof}

\begin{lemma}[Good event] \label{lem:good_event}
Define the following quantities:
\begin{align*}
\beta_{t, \ell}^{(1)}(\delta) &\coloneqq \sqrt{2\log\bigg(\frac{2|\mathcal{A}_{t}^{(\ell-1)}|(B-1)T}{\delta}\bigg)} + \sqrt{\lambda} \,,\\
\beta_{t,\ell}^{(2)}(\delta) &\coloneqq 2\sqrt{\log\bigg(\frac{2^{6d-3}\pi d(B-1)^2}{15^{d-1}\delta^2}\bigg)} +2\sqrt{\lambda}\,,\\
\varepsilon_{t,\ell}(\delta) &\coloneqq \underset{y \in \mathcal{A}_{t}^{(\ell-1)}}{\max}\|y\|_{V_\ell^{-1}} \cdot \Big(\beta_{t,\ell}^{(1)}(\delta) \wedge \beta_{t, \ell}^{(2)}(\delta) \Big) \; .
\end{align*}
Here, $\delta$ is a constant in the interval $(0,1)$. Then, the following event $E$ holds with probability at least $1-\delta$
$$ E \coloneqq \bigcap\limits_{\ell = 1}^{B-1}\bigcap\limits_{t = \mathcal{T}_{\ell}+1}^{T}\left\{|\langle x,\hat{\theta}_\ell - \theta^* \rangle | \leq \varepsilon_{t,\ell}(\delta), \; \forall x \in \mathcal{A}_{t}^{(\ell-1)}\right\} \; .$$
\end{lemma}

\begin{proof}
Fix arbitrary $\ell$ and $t$. By \cref{lem:concentration}, the following inequality holds for all $x \in \mathcal{A}_{t}^{(\ell-1)}$ with probability at least $1- \frac{\delta}{2(B-1)T}$
\begin{equation*}
    |\langle x, \hat{\theta}_{\ell} - \theta^{*} \rangle | \leq \underset{y \in \mathcal{A}_{t}^{(\ell-1)}}{\max}\|y\|_{V_\ell^{-1}} \cdot \left(\sqrt{2 \log \bigg(\frac{2|\mathcal{A}_{t}^{(\ell-1)}|(B-1)T}{\delta}\bigg) } +\sqrt{\lambda}\right) \; .
\end{equation*}
Applying a union bound over all $\ell$ and $t$, the following event holds with probability at least $1-\frac{\delta}{2}$
\begin{equation}  \label{eq:first_good_event}
\bigcap\limits_{\ell = 1}^{B-1}\bigcap\limits_{t = \mathcal{T}_{\ell}+1}^{T}\left\{|\langle x,\hat{\theta}_\ell - \theta^* \rangle | \leq \underset{y \in \mathcal{A}_{t}^{(\ell-1)}}{\max}\|y\|_{V_\ell^{-1}} \cdot \beta_{t,\ell}^{(1)}(\delta) , \; \forall x \in \mathcal{A}_{t}^{(\ell-1)}\right\} \; .
\end{equation}
Next, for fixed $\ell$, by the Cauchy-Schwarz inequality and \cref{cor:estimate_difference}, the following inequality holds for all $t$ and $x \in \mathcal{A}_{t}^{(\ell-1)}$ with probability at least $1-\frac{\delta}{2(B-1)}$
\[
| \langle x, \hat{\theta}_{\ell} - \theta^* \rangle | \leq \|x\|_{V_{{\ell}}^{-1}}\|\hat{\theta}_{\ell} - \theta^*\|_{V_{{\ell}}} \leq \underset{y \in \mathcal{A}_{t}^{(\ell-1)}}{\max}\|y\|_{V_\ell^{-1}} \cdot \left( 2\sqrt{\log\bigg(\frac{2^{6d-3}\pi d(B-1)^2}{15^{d-1}\delta^2}\bigg)} +2\sqrt{\lambda} \right) \; . 
\]
A union bound over $\ell $ yields the following event with probability at least $1-\frac{\delta}{2}$
\begin{equation} \label{eq:second_good_event}
\bigcap\limits_{\ell = 1}^{B-1}\bigcap\limits_{t = \mathcal{T}_{\ell}+1}^{T}\left\{|\langle x,\hat{\theta}_\ell - \theta^* \rangle | \leq \underset{y \in \mathcal{A}_{t}^{(\ell-1)}}{\max}\|y\|_{V_\ell^{-1}} \cdot \beta_{t,\ell}^{(2)}(\delta) , \; \forall x \in \mathcal{A}_{t}^{(\ell-1)}\right\} \; .
\end{equation}
Combining event~\eqref{eq:first_good_event} and event~\eqref{eq:second_good_event}, we conclude that the following event holds with probability at least $1-\delta$
\begin{equation*}
\bigcap\limits_{\ell = 1}^{B-1}\bigcap\limits_{t = \mathcal{T}_{\ell}+1}^{T}\left\{|\langle x,\hat{\theta}_\ell - \theta^* \rangle | \leq \underset{y \in \mathcal{A}_{t}^{(\ell-1)}}{\max}\|y\|_{V_\ell^{-1}} \cdot \Big(\beta_{t,\ell}^{(1)}(\delta) \wedge \beta_{t, \ell}^{(2)}(\delta) \Big)  , \; \forall x \in \mathcal{A}_{t}^{(\ell-1)} 
\right\} \; .
\end{equation*}
\end{proof}
\begin{lemma} \label{lem:optimal_arm}
Let $E$ be the good event defined in \cref{lem:good_event} with $\delta = \frac{2}{T}$. Conditioned on $E$, the optimal arm $x_{t}^* \in  \arg\max_{x \in \mathcal{A}_t} \langle x, \theta^* \rangle$ is never eliminated at any round $t$. In particular,
\begin{equation*}
x_t^* \in \mathcal{A}_t^{(\ell)}, \quad \text{for all } \, 1 \leq \ell \leq B-1 \, \text{ and } \, \mathcal{T}_\ell+1 \leq t \leq T \; .
\end{equation*}
\end{lemma}
\begin{proof}
Fix $t\in[\mathcal{T}_{s}+1,\mathcal{T}_{s+1}]$ for some $s\in[B-1]$. We show by induction on $\ell$ that $x_t^*\in\mathcal{A}_t^{(\ell)}$ for all $\ell\in[s]$.  
\\\\
\textbf{Base case ($\ell = 1$).} Since both $x_t^*$ and $x_{t}^{(1)}$ belong to $\mathcal{A}_t^{(0)}(=\mathcal{A}_t)$, we have
\begin{align*}
\langle \hat{\theta}_{1}, x_{t}^{(1)}-x_t^* \rangle
&= \langle \hat{\theta}_{1} - \theta^*, x_{t}^{(1)} - x_t^* \rangle + \langle \theta^*, x_{t}^{(1)} - x_t^* \rangle \\
&\leq \langle \hat{\theta}_{1} - \theta^*, x_{t}^{(1)} - x_t^* \rangle \\
&\leq |\langle \hat{\theta}_{1} - \theta^*, x_{t}^{(1)} \rangle| +  |\langle \hat{\theta}_{1} - \theta^*, x_t^* \rangle|  \\
&\leq 2\varepsilon_{t,1} \; ,
\end{align*}
where the first inequality follows from the optimality of $x_t^*$ and the last from the definition of the good event $E$. Hence $x_t^*\in \mathcal{A}_t^{(1)}$.
\\\\
\textbf{Inductive step.} Assume $x_t^* \in \mathcal{A}_t^{(\ell-1)}$ for some $\ell \in \{2, \dots, s\}$. Since $x_{t}^{(\ell)} \in \mathcal{A}_t^{(\ell-1)}$, we similarly obtain
\begin{align*}
\langle \hat{\theta}_\ell, x_{t}^{(\ell)} - x_t^* \rangle
&= \langle \hat{\theta}_\ell - \theta^*, x_{t}^{(\ell)} - x_t^* \rangle + \langle \theta^*, x_{t}^{(\ell)} - x_t^* \rangle \\
&\leq \langle \hat{\theta}_\ell - \theta^*, x_{t}^{(\ell)} - x_t^* \rangle \\
&\leq |\langle \hat{\theta}_\ell - \theta^*, x_{t}^{(\ell)} \rangle| + |\langle \hat{\theta}_\ell - \theta^*, x_t^* \rangle| \\
&\leq 2\varepsilon_{t,\ell} \; ,
\end{align*}
which shows that $x_t^* \in \mathcal{A}_t^{(\ell)}$. By induction, the claim holds for all $\ell \in [s]$, completing the proof.
\end{proof}

\begin{lemma} \label{lem:elliptical_potential} \citep{abbasi2011improved}
Let $\{x_1,\dots,x_n\} \subset \mathbb{R}^d$ be a sequence of vectors such that $\|x_i\|_2 \leq 1$ for all $i \in [n]$. Let $H_0 \in \mathbb{R}^{d \times d}$ be a positive definite matrix, and define $H_t \coloneqq H_0 + \sum_{i=1}^{t} x_ix_i^\top$. Then, the following inequality holds
\[
\sum_{t=1}^{n}\min\left\{1, \|x_t\|_{H_{t-1}^{-1}}^2\right\} \leq 2 \log\left(\frac{\det(H_n)}{\det(H_0)}\right) \; .
\]
\end{lemma}

\begin{corollary} \label{cor:elliptical_potential}
   Let $\{x_1,\dots,x_n\} \subset \mathbb{R}^d$ be a sequence of vectors such that $\|x_i\|_2 \leq 1$ for all $i \in [n]$. Suppose $\lambda \geq 1$, and define $H_0 \coloneqq \lambda I$ and $H_t \coloneqq \lambda I + \sum_{i=1}^{t} x_ix_i^\top$ for each $t \in \{1,\dots,n\}$. Then, for any $1 \leq m \leq n$,
\[
\sum_{t=m}^{n} \|x_t\|_{H_{t-1}^{-1}}^2 \leq 2 \log\left(\frac{\det(H_n)}{\det(H_{m-1})}\right) \; .
\] 
\end{corollary}

\begin{proof}
    Since $H_{t-1}^{-1} \preceq \lambda^{-1}I$ for all $t$, it follows that $H_{t-1}^{-1} \preceq \lambda^{-1}I \preceq I$. Consequently, $\|x_t\|_{H_{t-1}^{-1}}^2 \leq \|x_t\|_2^2 \leq 1$ for each $t \in [m,n]$. Applying~\cref{lem:elliptical_potential} over the interval $[m,n]$ yields the stated bound.
\end{proof}

\begin{lemma} \label{lem:variance_to_matrix}
    Let $H$ be a positive definite matrix. Suppose $x \in \mathbb{R}^d$ satisfies $\|x\|_{H^{-1}}^2 \leq c$. Then $cH \succeq xx^\top$.
\end{lemma}

\begin{proof}
    For any $z \in \mathbb{R}^d$, we have
    \[
    z^\top(cH)z \geq \|x\|_{H^{-1}}^2 \cdot \|z\|_{H}^2 = \|H^{-\frac{1}{2}}x\|^2 \cdot \|H^{\frac{1}{2}}z\|^2 \geq \langle x, z \rangle^2 = z^\top(xx^\top)z \; .
    \]
    The first inequality follows from the assumption $\|x\|_{H^{-1}}^2 \leq c$. The second inequality follows from the Cauchy–Schwarz inequality. Since this bound holds for all $z \in \mathbb{R}^d$, the matrix inequality $cH \succeq xx^\top$ follows.
\end{proof}

\begin{lemma}[Pseudoinverse variance-to-matrix]
\label{lem:pseudoinverse_variance_to_matrix}
Let $M\succeq 0$ and suppose $x\in\operatorname{range}(M)$.
If $x^\top M^+x\le c$, then
\[
xx^\top \preceq cM.
\]
\end{lemma}

\begin{proof}
For any $z\in\mathbb R^d$,
\[
(z^\top x)^2
=
\bigl((M^{1/2}z)^\top M^{+/2}x\bigr)^2
\le
(z^\top Mz)(x^\top M^+x)
\le
c\,z^\top Mz.
\]
Since this holds for every $z\in\mathbb R^d$, we obtain
$xx^\top\preceq cM$.
\end{proof}

\begin{lemma} \label{lem:empirical_Gram_matrix} \citep{ruan2021linear}
    Let $x_1,\dots,x_n$ be independent and identically distributed (i.i.d.) random vectors drawn from a distribution $\mathcal{D}$ such that $\|x_i\|_2 \leq 1$ almost surely. For any cutoff level $\lambda > 0$, the following inequality holds with probability at least $1-2d\exp(-\frac{\lambda n}{8})$
    \[
    3\lambda I + \frac{1}{n}\sum_{i=1}^{n}x_ix_i^\top \succeq \frac{1}{8} \mathbb{E}_{x \sim \mathcal{D}}\left[xx^\top\right]
    \]
\end{lemma}

\begin{corollary} \label{cor:empirical_Gram_matrix}
    Let $x_1,\dots,x_n$ be independent and identically distributed (i.i.d.) random vectors drawn from a distribution $\mathcal{D}$ such that $\|x_i\|_2 \leq 1$ almost surely. Then, the following inequality holds with probability at least $1-\frac{2}{T}$
    \[
    \frac{24\log(dT)}{n} I + \frac{1}{n}\sum_{i=1}^{n}x_ix_i^\top \succeq \frac{1}{8} \mathbb{E}_{x \sim \mathcal{D}}\left[xx^\top\right]
    \]
\end{corollary}

\begin{proof}
    The proof is a direct application of \cref{lem:empirical_Gram_matrix}. We achieve the desired inequality by setting the cutoff level $\lambda$ to $\lambda = \frac{8\log(dT)}{n}$.
\end{proof}

\begin{lemma}\label{lem:gopt_mvee} \citep{todd2007khachiyan}
Let $X=\{x_1,\dots,x_K\}\subset\mathbb{R}^d$ be a set of $K$ points that spans $\mathbb{R}^d$, and fix $\varepsilon\in(0,1]$. Then Khachiyan’s barycentric coordinate descent algorithm computes a $(1+\varepsilon)$-approximation to the minimum-volume enclosing ellipsoid of $X$ in
\[
O\!\Big( K d^2 \,\big( [(1+\varepsilon)^{2/(d+1)} - 1]^{-1} + \log d \big) \Big)
\]
arithmetic operations. In particular, since $[(1+\varepsilon)^{2/(d+1)} - 1]^{-1}=\Theta(d/\varepsilon)$ for $\varepsilon\in(0,1]$, the total time complexity simplifies to
\[
O\!\left( \tfrac{K d^3}{\varepsilon} \right)\; .
\]
\end{lemma}

\begin{corollary}[Near G-optimal design] \label{cor:near_G_optimal}
Let $X=\{x_1,\dots,x_K\}\subset\mathbb{R}^d$ and let $r=\mathrm{rank}(X)\le d$.
Let $U\in\mathbb{R}^{d\times r}$ have orthonormal columns spanning $\mathrm{span}\{x_i\}$, and define $x_i':=U^\top x_i\in\mathbb{R}^r$ together with
$\Sigma'(w):=\sum_{i=1}^K w_i\,x_i' x_i'^{\!\top}$ for $w\in\Delta_K$ ($K$-dimensional probability simplex).
Applying the algorithm of \cref{lem:gopt_mvee} to $\{x_i'\}_{i=1}^K$ with accuracy parameter fixed to $\varepsilon=1$ returns weights $w_\circ\in\Delta_K$ such that
\[
\max_{j\in[K]} x_j'^{\!\top}\Big(\Sigma'(w_\circ)\Big)^{-1}x_j' \;\le\; 2\,r \; ,
\]
in
\[
O\!\Big(Kr^2\big([(1+1)^{2/(r+1)}-1]^{-1}+\log r\big)\Big) \;=\; O(Kr^3)
\]
arithmetic operations.
Equivalently, in the original space we have
\[
\max_{j\in[K]} x_j^{\!\top}\big(\Pi\,\Sigma(w_\circ)^{+}\Pi\big)\,x_j \;\le\; 2\,r \; ,
\qquad
\text{where }\ \Sigma(w):=\sum_{i=1}^K w_i\,x_i x_i^\top,\ \Pi:=UU^\top \; ,
\]
with ${}^+$ denoting the Moore--Penrose pseudoinverse.
Hence $w_\circ$ is a constant-factor near-optimal design in the effective dimension $r$, computed in polynomial time $O(Kr^3)$ by the method in \cref{lem:gopt_mvee}.
\end{corollary}

\begin{proof}
Consider the subspace spanned by $\{x_i\}$ of dimension $r$, let $U$ be an orthonormal basis, and project $x_i' := U^\top x_i \in \mathbb{R}^r$. 
The MVEE and $D/G$-optimal design problems are equivalently posed in $\mathbb{R}^r$, and guarantees for $x_i'$ and $\Sigma'(w):=\sum_i w_i x_i' x_i'^{\!\top}$ transfer back to the original space through the projector $\Pi=UU^\top$ and the Moore--Penrose pseudoinverse. 
By \cref{lem:gopt_mvee}, Khachiyan’s barycentric coordinate descent computes an $\varepsilon$-approximate MVEE of $\{x_i'\}$ in 
\[
O\!\Big(Kr^2\big([(1+\varepsilon)^{2/(r+1)}-1]^{-1}+\log r\big)\Big)
\]
arithmetic operations; fixing $\varepsilon=1$ gives $O(Kr^3)$.

Moreover, the standard MVEE $\leftrightarrow D/G$-optimal duality and the algorithm’s stopping rule imply that the returned weights $w_\circ\in\Delta_K$ satisfy the \emph{constraint-violation guarantee}
\[
\max_{j\in[K]} x_j'^{\!\top}\Big(\Sigma'(w_\circ)\Big)^{-1}x_j' \;\le\; (1+\varepsilon)\,r \; .
\]
With $\varepsilon=1$ this yields
\[
\max_{j\in[K]} x_j'^{\!\top}\Big(\Sigma'(w_\circ)\Big)^{-1}x_j' \;\le\; 2\,r \; .
\]
Finally, lifting back to the original space uses the identity
\[
x_j'^{\!\top}\Big(\Sigma'(w_\circ)\Big)^{-1}x_j'
\;=\;
x_j^{\!\top}\big(\Pi\,\Sigma(w_\circ)^{+}\Pi\big)\,x_j\;,
\qquad 
\Sigma(w):=\sum_{i=1}^K w_i\,x_i x_i^\top,\ \Pi:=UU^\top,
\]
so that
\[
\max_{j\in[K]} x_j^{\!\top}\big(\Pi\,\Sigma(w_\circ)^{+}\Pi\big)\,x_j \;\le\; 2\,r\;.
\]
Hence $w_\circ$ is a constant-factor near-$G$-optimal design in effective dimension $r$, computed in polynomial time $O(Kr^3)$ by the method in \cref{lem:gopt_mvee}.
\end{proof}

\begin{lemma}[Design transfer for the G-design prefix of \texttt{BLCE-G}]
\label{lem:blceg_design_transfer}
Fix an interval $j\in[B]$.
Let
\[
n_{j,1}\coloneqq \mathcal T_{j-1}'-\mathcal T_{j-1},
\qquad
n_{j,2}\coloneqq \mathcal T_{j-1}''-\mathcal T_{j-1}'.
\]
For a fresh context draw in interval $j$, let $\mathcal A^{(j-1)}$
denote the post-elimination feasible set obtained by applying the first
$j-1$ elimination steps using the estimates available at the beginning of
interval $j$. For $j=1$, set $\mathcal A^{(0)}=\mathcal A$.

Define the conditional mean G-design covariance
\[
\bar\Sigma_j
\coloneqq
\mathbb E\!\left[
\mathbb E_{z\sim \pi_{G'}(\mathcal A^{(j-1)})}[zz^\top]
\,\middle|\,
\mathcal F_{\mathcal T_{j-1}}
\right].
\]
Let
\[
G_j
\coloneqq
\left\{
\bar\Sigma_j
\preceq
\frac{192}{n_{j,1}}H_{\mathcal T_{j-1}'}
\right\}.
\]
Then
\[
\mathbb P(G_j\mid \mathcal F_{\mathcal T_{j-1}})\ge 1-\frac{2}{T}.
\]
Moreover, on $G_j$,
\[
\mathbb E\!\left[
2\log\frac{\det(H_{\mathcal T_{j-1}''})}
{\det(H_{\mathcal T_{j-1}'})}
\,\middle|\,
\mathcal F_{\mathcal T_{j-1}'}
\right]
\le
2d\log\!\left(
1+\frac{384d\,n_{j,2}}{n_{j,1}}
\right).
\]
\end{lemma}

\begin{proof}
Condition on $\mathcal F_{\mathcal T_{j-1}}$.
During the G-design prefix of interval $j$, the selected arms
\[
g_s\coloneqq x_{s,a_s},
\qquad
s=\mathcal T_{j-1}+1,\ldots,\mathcal T_{j-1}',
\]
are i.i.d. draws from the mixture distribution induced by first drawing a
fresh context set and then sampling from
$\pi_{G'}(\mathcal A_s^{(j-1)})$.
The second moment of this mixture distribution is exactly $\bar\Sigma_j$.
Therefore, by \cref{cor:empirical_Gram_matrix}, with conditional
probability at least $1-2/T$,
\[
\frac{24\log(dT)}{n_{j,1}}I
+
\frac{1}{n_{j,1}}
\sum_{s=\mathcal T_{j-1}+1}^{\mathcal T_{j-1}'}
g_sg_s^\top
\succeq
\frac18\bar\Sigma_j.
\]
Equivalently,
\[
\bar\Sigma_j
\preceq
\frac{8}{n_{j,1}}
\left(
24\log(dT)I
+
\sum_{s=\mathcal T_{j-1}+1}^{\mathcal T_{j-1}'}
g_sg_s^\top
\right).
\]
Since $\lambda=\log(dT)$ and
\[
H_{\mathcal T_{j-1}'}
=
\lambda I+
\sum_{s=\mathcal T_{j-1}+1}^{\mathcal T_{j-1}'}
g_sg_s^\top,
\]
we have
\[
24\log(dT)I
+
\sum_{s=\mathcal T_{j-1}+1}^{\mathcal T_{j-1}'}
g_sg_s^\top
\preceq
24H_{\mathcal T_{j-1}'}.
\]
Thus
\[
\bar\Sigma_j
\preceq
\frac{192}{n_{j,1}}H_{\mathcal T_{j-1}'},
\]
which proves the first claim.

Now consider Phase 2 of interval $j$.
For
\[
s=\mathcal T_{j-1}'+1,\ldots,\mathcal T_{j-1}'',
\]
write
\[
\Sigma_s
\coloneqq
\mathbb E_{z\sim \pi_{G'}(\mathcal A_s^{(j-1)})}[zz^\top].
\]
Since $x_{s,a_s}\in\mathcal A_s^{(j-1)}$, the near G-optimal design
guarantee and \cref{lem:pseudoinverse_variance_to_matrix} imply
\[
x_{s,a_s}x_{s,a_s}^\top
\preceq
2d\,\Sigma_s.
\]
Let $H\coloneqq H_{\mathcal T_{j-1}'}$.
Then
\begin{align*}
2\log\frac{\det(H_{\mathcal T_{j-1}''})}{\det(H)}
&=
2\log\det\!\left(
I+
H^{-1/2}
\left(
\sum_{s=\mathcal T_{j-1}'+1}^{\mathcal T_{j-1}''}
x_{s,a_s}x_{s,a_s}^\top
\right)
H^{-1/2}
\right)\\
&\le
2\log\det\!\left(
I+
2d
\sum_{s=\mathcal T_{j-1}'+1}^{\mathcal T_{j-1}''}
H^{-1/2}\Sigma_sH^{-1/2}
\right).
\end{align*}
Conditional on $\mathcal F_{\mathcal T_{j-1}'}$, the Phase 2 context
sets are fresh i.i.d. draws, and hence
\[
\mathbb E[\Sigma_s\mid \mathcal F_{\mathcal T_{j-1}'}]
=
\bar\Sigma_j.
\]
Using the concavity of $\log\det(\cdot)$ on the positive definite cone,
we obtain
\begin{align*}
&\mathbb E\!\left[
2\log\frac{\det(H_{\mathcal T_{j-1}''})}{\det(H)}
\,\middle|\,
\mathcal F_{\mathcal T_{j-1}'}
\right]\\
&\quad\le
2\log\det\!\left(
I+
2d\,n_{j,2}\,
H^{-1/2}\bar\Sigma_jH^{-1/2}
\right).
\end{align*}
On the event $G_j$,
\[
H^{-1/2}\bar\Sigma_jH^{-1/2}
\preceq
\frac{192}{n_{j,1}}I.
\]
Therefore,
\[
\mathbb E\!\left[
2\log\frac{\det(H_{\mathcal T_{j-1}''})}{\det(H)}
\,\middle|\,
\mathcal F_{\mathcal T_{j-1}'}
\right]
\le
2d\log\!\left(
1+\frac{384d\,n_{j,2}}{n_{j,1}}
\right).
\]
This completes the proof.
\end{proof}

\begin{proof}[Proof of Theorem 1]
To establish the claimed interval complexity bound, we analyze the lower bound on the length of the $\ell$-th interval, where $\ell = \lceil \log_2\log_2 T \rceil$. By definition of the schedule, the length of this interval is at least
\[
\Bigg\lceil \frac{T^{1-2^{-\ell}}}{\log_2\log_2T}\Bigg\rceil +2 \geq \frac{T^{1-2^{-\lceil \log_2\log_2 T\rceil}}}{\log_2\log_2T} \geq \frac{T^{1-2^{- \log_2\log_2 T}}}{\log_2\log_2T} = \frac{T^{1-\frac{1}{\log_2T}}}{\log_2\log_2T} = \frac{T}{2\log_2\log_2T} \; .
\]
Since the interval length is non-decreasing in $\ell$, every interval with index $\ell \geq \lceil \log_2\log_2 T \rceil$ has length at least $\frac{T}{2\log_2\log_2T}$. Given the total time horizon $T$, the number of such intervals is therefore at most $\lceil 2 \log_2\log_2T \rceil$. Including the initial $\lceil  \log_2\log_2T \rceil - 1$ intervals, the total number of intervals $B$ is bounded by
\[ 
B \leq \lceil 2 \log_2\log_2T \rceil + \lceil  \log_2\log_2T \rceil - 1 \leq 3 \log_2\log_2T + 1 \; ,
\]
which implies that the interval complexity is $\Ocal(\log\log T)$.
\\\\
Next, we begin by decomposing the cumulative expected regret based on the good event $E$. The regret can be written as
\begin{align*}
\mathcal{R}(T) &= \sum_{k=1}^{B}\sum_{t=\mathcal{T}_{k-1}+1}^{\mathcal{T}_k} \mathbb{E}\bigg[\max_{x\in \mathcal{A}_t} \langle x,\theta^*\rangle - \langle x_{t,a_t},\theta^*\rangle \bigg] \\
&= \sum_{k=1}^{B}\sum_{t=\mathcal{T}_{k-1}+1}^{\mathcal{T}_k} \mathbb{E}\bigg[\langle x_t^*-x_{t,a_t}, \theta^* \rangle\bigg| E^c \bigg] \cdot \mathbb{P}(E^c) + \mathbb{E}\bigg[\langle x_t^*-x_{t,a_t}, \theta^* \rangle\bigg| E \bigg] \cdot \mathbb{P}(E) \; .
\end{align*}
Using the triangle inequality followed by the Cauchy-Schwarz inequality, we have
\begin{align*}
\langle x_t^*-x_{t,a_t}, \theta^* \rangle &\leq |\langle x_t^*-x_{t,a_t}, \theta^* \rangle| \\
&\leq |\langle x_t^*, \theta^* \rangle| + |\langle x_{t,a_t}, \theta^* \rangle| \\
&\leq \|x_t^*\|\cdot\|\theta^*\| + \|x_{t,a_t}\|\cdot\|\theta^*\| \\
&\leq 2 \; ,
\end{align*}
where the last inequality follows from Assumption~\ref{assum:norm}.
Hence, the cumulative expected regret can be bounded as
\begin{align*}
\mathcal{R}(T) &\leq 2T\cdot \mathbb{P}(E^c)+ \sum_{k=1}^{B}\sum_{t=\mathcal{T}_{k-1}+1}^{\mathcal{T}_k} \mathbb{E}\bigg[\langle x_t^*-x_{t,a_t}, \theta^* \rangle\bigg| E \bigg] \cdot \mathbb{P}(E) \\
&\leq \Ocal(1) + \sum_{k=1}^{B}\sum_{t=\mathcal{T}_{k-1}+1}^{\mathcal{T}_k} \mathbb{E}\bigg[\langle x_t^*-x_{t,a_t}, \theta^* \rangle\bigg| E \bigg] \; ,
\end{align*}
where the $\Ocal(1)$ term follows from the high-probability guarantee $\mathbb{P}(E^c) = \Ocal(1/T)$.
Throughout the remainder of the analysis, we therefore condition on the good event \(E\). Let $\text{Regret}_{\ell}$ denote the cumulative expected regret incurred during interval $\ell$. We analyze the regret separately for the case \(\ell = 1\) and for all subsequent intervals \(\ell \ge 2\).
\paragraph{Case 1: $\ell=1$.} In the first interval, the number of rounds is $\mathcal{T}_1 = \left\lceil \frac{\sqrt{T}}{\log_2\log_2 T}\right\rceil + 1$, and the instantaneous regret is bounded by 2. Therefore,
\[
\text{Regret}_{1} = \sum_{t=1}^{\mathcal{T}_1} \mathbb{E}[\langle x_t^*-x_{t,a_t}, \theta^* \rangle] \leq 2 \mathcal{T}_1 \leq 2 \bigg( \frac{\sqrt{T}}{\log_2\log_2 T} + 2 \bigg) = \Ocal\big(\sqrt{T}\big) \; .
\]
\paragraph{Case 2: $\ell \geq 2$.} For each interval $\ell \geq 2$, the rounds following the arm elimination steps are divided into three phases:
\\
\textbf{Phases 1 \& 2.} In the first two phases of length $\Big\lceil \frac{c^2T^{1-2^{-\ell}}}{\log_2\log_2T}\Big\rceil + \Big\lceil \frac{c(1-c)T^{1-2^{-\ell}}}{\log_2\log_2T}\Big\rceil$, the algorithm selects arm according to a near G-optimal design and the most informative direction with respect to the current Gram matrix. For any round $t$ in these phases, the instantaneous regret satisfies
\begin{align*}
\langle x_t^* - x_{t,a_t}, \theta^* \rangle &= \langle x_t^* - x_{t,a_t}, \theta^* - \hat{\theta}_{\ell-1} \rangle + \langle x_t^* - x_{t,a_t}, \hat{\theta}_{\ell-1} \rangle \\
&\leq \langle x_t^* - x_{t,a_t}, \theta^* - \hat{\theta}_{\ell-1} \rangle + \langle x_{t}^{(\ell-1)} - x_{t,a_t}, \hat{\theta}_{\ell-1} \rangle \\
&\leq |\langle x_t^* , \hat{\theta}_{\ell-1} - \theta^* \rangle| + |\langle x_{t,a_t}, \hat{\theta}_{\ell-1}- \theta^* \rangle| + 2\varepsilon_{t,\ell-1} \\
&\leq 4\varepsilon_{t,\ell-1} \; . \tag{3} \label{eq:regret_upper_bound_1,2}
\end{align*}
The first inequality uses the fact that $x_{t}^{(\ell-1)}$ is optimal with respect to $\hat{\theta}_{\ell-1}$ and that $x_t^*$ belongs to $\mathcal{A}_t^{(\ell-2)}$ by \cref{lem:optimal_arm}. The second inequality follows since $x_{t,a_t} \in \mathcal{A}_{t}^{(\ell-1)}$ by construction of the arm elimination step. The final inequality follows from the definition of the good event $E$ and again from \cref{lem:optimal_arm}, which guarantees $x_t^* \in \mathcal{A}_t^{(\ell-2)}$.  
\\\\
\textbf{Phase 3.} In the last phase, the algorithm selects arms greedily with respect to the estimated parameter, i.e., $x_{t,a_t} \in \arg\max_{x \in \mathcal{A}_t^{(\ell-1)}}\langle x, \hat{\theta}_{\ell-1}\rangle$. For any round $t$ in this phase, the instantaneous regret satisfies
\begin{align*}
\langle x_t^* - x_{t,a_t}, \theta^* \rangle &= \langle x_t^* - x_{t,a_t}, \theta^* - \hat{\theta}_{\ell-1} \rangle + \langle x_t^* - x_{t,a_t}, \hat{\theta}_{\ell-1} \rangle \\
&\leq \langle x_t^* - x_{t,a_t}, \theta^* - \hat{\theta}_{\ell-1} \rangle \\
&\leq |\langle x_t^* , \hat{\theta}_{\ell-1} - \theta^* \rangle| + |\langle x_{t,a_t}, \hat{\theta}_{\ell-1}- \theta^* \rangle| \\
&\leq 2\varepsilon_{t,\ell-1} \; . \tag{4} \label{eq:regret_upper_bound_3}
\end{align*}
The first inequality follows from the fact that $x_{t,a_t}$ is optimal with respect to $\hat{\theta}_{\ell-1}$ and that $x_t^*$ belongs to $\mathcal{A}_t^{(\ell-1)}$ by \cref{lem:optimal_arm}. The final inequality follows directly from the definition of the good event $E$ and from \cref{lem:optimal_arm}, as in Phases 1 \& 2.
\\\\
Therefore, it suffices to upper bound the quantity $\sum_{t=\mathcal{T}_\ell+1}^{\mathcal{T}_{\ell+1}}\varepsilon_{t,\ell}$ for each $\ell \in [B-1]$, which reduces to bounding the sum  $\sum_{t=\mathcal{T}_\ell+1}^{\mathcal{T}_{\ell+1}} \max_{{y \in \mathcal{A}_{t}^{(\ell-1)}}}\|y\|_{V_\ell^{-1}}$. Since both $V_\ell$ and the arm elimination rule---determined by $\hat{\theta}_1, \dots, \hat{\theta}_{\ell-1}$---are measurable with respect to $\mathcal{F}_{\mathcal{T}_{\ell}}$, they can be treated as fixed quantities conditional on this filtration. Given that the contexts are drawn independently and identically, it follows by the tower property that for any $t,v \in [\mathcal{T}_\ell+1, \mathcal{T}_{\ell+1}]$, we have
\begin{align*}
    \mathbb{E}\bigg[ \max_{{y \in \mathcal{A}_{t}^{(\ell-1)}}} y^\top V_\ell^{-1}y \bigg] 
    &= \mathbb{E}\bigg[\mathbb{E}\bigg[ \max_{{y \in \mathcal{A}_t^{(\ell-1)}}} y^\top V_\ell^{-1}y \, \bigg| \mathcal{F}_{\mathcal{T}_{\ell}} \bigg] \bigg] \\
    &= \mathbb{E}\bigg[\mathbb{E}\bigg[ \max_{{y \in \mathcal{A}_v^{(\ell-1)}}} y^\top V_\ell^{-1}y \, \bigg| \mathcal{F}_{\mathcal{T}_{\ell}} \bigg] \bigg] \\
    &= 
    \mathbb{E}\bigg[ \max_{{y \in \mathcal{A}_{v}^{(\ell-1)}}} y^\top V_\ell^{-1}y \bigg] \; . \tag{5} \label{eq:maximal_quadratic}
\end{align*}
Define $\mathcal{T}_{\ell-1}' \coloneqq \mathcal{T}_{\ell-1} + \Big\lceil \frac{c^2T^{1-2^{-\ell}}}{\log_2\log_2T} \Big\rceil$ and $\mathcal{T}_{\ell-1}'' \coloneqq \mathcal{T}_{\ell-1}' + \Big\lceil \frac{c(1-c)T^{1-2^{-\ell}}}{\log_2\log_2T} \Big\rceil$ for all $\ell \geq 2$,
whereas for the first interval ($\ell = 1$), we set $\mathcal{T}_{0}' \coloneqq \mathcal{T}_{0} + \Big\lceil \frac{c\sqrt{T}}{\log_2\log_2T} \Big\rceil$ and $\mathcal{T}_{0}'' \coloneqq \mathcal{T}_{1}$. Now, fix any $t \in [\mathcal{T}_\ell+1, \mathcal{T}_{\ell+1}]$ and consider the interval $s \in [\mathcal{T}_{\ell-1}'+1, \mathcal{T}_{\ell-1}'']$. Then, we have the following result
\begin{align*}
    \sum_{s=\mathcal{T}_{\ell-1}'+1}^{\mathcal{T}_{\ell-1}''} \mathbb{E}\bigg[ \max_{y \in \mathcal{A}_{t}^{(\ell-1)}} y^\top V_\ell^{-1}y \bigg] 
    &\leq  \sum_{s=\mathcal{T}_{\ell-1}'+1}^{\mathcal{T}_{\ell-1}''} \mathbb{E} \bigg[ \max_{y \in \mathcal{A}_{t}^{(\ell-1)}} y^\top H_{s-1}^{-1}y \bigg] \\
    &= \sum_{s=\mathcal{T}_{\ell-1}'+1}^{\mathcal{T}_{\ell-1}''} \mathbb{E} \bigg[\mathbb{E} \bigg[ \max_{y \in \mathcal{A}_{t}^{(\ell-1)}} y^\top H_{s-1}^{-1}y \, \bigg| \mathcal{F}_{s-1} \bigg] \bigg] \\
    &= \sum_{s=\mathcal{T}_{\ell-1}'+1}^{\mathcal{T}_{\ell-1}''} \mathbb{E} \bigg[\mathbb{E} \bigg[ \max_{y \in \mathcal{A}_{s}^{(\ell-1)}} y^\top H_{s-1}^{-1}y \, \bigg| \mathcal{F}_{s-1} \bigg] \bigg] \\
    &= \sum_{s=\mathcal{T}_{\ell-1}'+1}^{\mathcal{T}_{\ell-1}''} \mathbb{E} \bigg[ \max_{y \in \mathcal{A}_{s}^{(\ell-1)}} y^\top H_{s-1}^{-1}y  \bigg]
    \; . \tag{6} \label{eq:shift_uncertainty}
\end{align*}
The first inequality follows from the monotonicity of the matrices, as $H_{s-1} \preceq V_\ell$ for all $s$ in the interval. The first and last equalities follow from the tower property. The second equality uses the fact that, conditional on $\mathcal{F}_{s-1}$, both $H_{s-1}$ and the arm elimination rule (determined by $\hat{\theta}_1, \dots, \hat{\theta}_{\ell-1}$) are fixed, and the distribution of the contexts remains unchanged due to their i.i.d. nature.
\\\\
In Phase 2, the algorithm proceeds by selecting the most informative direction at each step with respect to the current Gram matrix. For all $\ell \geq 1$, we obtain
\begin{align*}
    \sum_{s=\mathcal{T}_{\ell-1}'+1}^{\mathcal{T}_{\ell-1}''} \mathbb{E} \bigg[ \max_{y \in \mathcal{A}_{s}^{(\ell-1)}} y^\top H_{s-1}^{-1}y  \bigg] 
    &= \mathbb{E} \left[ \sum_{s=\mathcal{T}_{\ell-1}'+1}^{\mathcal{T}_{\ell-1}''} \, \max_{y \in \mathcal{A}_{s}^{(\ell-1)}} y^\top H_{s-1}^{-1}y  \right] \\
    &= \mathbb{E} \left[ \sum_{s=\mathcal{T}_{\ell-1}'+1}^{\mathcal{T}_{\ell-1}''} x_{s,a_s}^\top H_{s-1}^{-1}x_{s,a_s}  \right] \\
    &\leq \mathbb{E} \left[ 2\log \left( \frac{\det(H_{\mathcal{T}_{\ell-1}''})}{\det(H_{\mathcal{T}_{\ell-1}'})} \right) \right] \; , \tag{7} \label{eq:elliptical_potential}
\end{align*}
where the second equality follows from the arm-selection strategy of Phase 2, and the first inequality follows from \cref{cor:elliptical_potential}.
\\\\
Next, we bound the expected determinant growth during Phase 2.
The previous argument reduces the problem to controlling
\[
\mathbb E\!\left[
2\log\left(
\frac{\det(H_{\mathcal T_{\ell-1}''})}
{\det(H_{\mathcal T_{\ell-1}'})}
\right)
\right].
\]
We apply \cref{lem:blceg_design_transfer} with $j=\ell$.
Let $G_\ell$ be the event defined in that lemma. Since
\[
\mathbb P(G_\ell^c\mid \mathcal F_{\mathcal T_{\ell-1}})\le \frac{2}{T},
\]
we have $\mathbb P(G_\ell^c)\le 2/T$.

Define
\[
Z_\ell
\coloneqq
2\log\left(
\frac{\det(H_{\mathcal T_{\ell-1}''})}
{\det(H_{\mathcal T_{\ell-1}'})}
\right).
\]
By the tower property and \cref{lem:blceg_design_transfer},
\[
\mathbb E[Z_\ell\mathbbm 1_{G_\ell}]
=
\mathbb E\!\left[
\mathbbm 1_{G_\ell}
\mathbb E[Z_\ell\mid \mathcal F_{\mathcal T_{\ell-1}'}]
\right]
\le
2d\log\!\left(
1+\frac{384d\,n_{\ell,2}}{n_{\ell,1}}
\right),
\]
where
\[
n_{\ell,1}
\coloneqq
\mathcal T_{\ell-1}'-\mathcal T_{\ell-1},
\qquad
n_{\ell,2}
\coloneqq
\mathcal T_{\ell-1}''-\mathcal T_{\ell-1}'.
\]
On the complement $G_\ell^c$, we use the crude determinant bound.
Since $H_{\mathcal T_{\ell-1}'}\succeq \lambda I$ and
$\operatorname{tr}(H_{\mathcal T_{\ell-1}''})\le d\lambda+T$, the AM--GM
inequality gives
\[
Z_\ell
\le
2\log\left(
\frac{(\operatorname{tr}(H_{\mathcal T_{\ell-1}''})/d)^d}
{\lambda^d}
\right)
\le
2d\log\left(1+\frac{T}{d\lambda}\right)
\le
2d\log(2T).
\]
Therefore,
\begin{align}
\mathbb E[Z_\ell]
&=
\mathbb E[Z_\ell\mathbbm 1_{G_\ell}]
+
\mathbb E[Z_\ell\mathbbm 1_{G_\ell^c}]
\nonumber\\
&\le
2d\log\!\left(
1+\frac{384d\,n_{\ell,2}}{n_{\ell,1}}
\right)
+
\frac{4d\log(2T)}{T}.
\label{eq:elliptical_dlogd_new}
\end{align}
For $\texttt{BLCE-G}$, $c\in(0,1)$ is fixed. By the definitions of
$\mathcal T_{\ell-1}'$ and $\mathcal T_{\ell-1}''$, the ratio
$n_{\ell,2}/n_{\ell,1}$ is bounded by a constant depending only on $c$.
For instance, for all sufficiently large $T$,
\[
\frac{n_{\ell,2}}{n_{\ell,1}}
\le
\frac{4}{c^2}.
\]
Absorbing constants into the $c$-dependent constant, we get
\begin{equation}
\mathbb E[Z_\ell]
\le
2d\log\left(\frac{1537d}{c^2}\right)
+
\frac{4d\log(2T)}{T}.
\tag{8}
\label{eq:elliptical_dlogd}
\end{equation}
\\\\
Now, we establish an upper bound on the cumulative expected regret. By combining the bounds derived in  (\ref{eq:regret_upper_bound_1,2}) and (\ref{eq:regret_upper_bound_3}), we can bound the cumulative expected regret for interval $\ell$, denoted by $\text{Regret}_\ell$, for any $\ell \geq 2$ as follows
\begin{align*}
    \text{Regret}_\ell &= \sum_{t=\mathcal{T}_{\ell-1}+1}^{\mathcal{T}_\ell} \mathbb{E}[\langle x_t^*-x_{t,a_t}, \theta^* \rangle] \\
    &\leq 4\sum_{t=\mathcal{T}_{\ell-1}+1}^{\mathcal{T}_\ell} \mathbb{E}[\varepsilon_{t,\ell-1}] \\
    &= 4\sum_{t=\mathcal{T}_{\ell-1}+1}^{\mathcal{T}_\ell} \mathbb{E}\left[\underset{y \in \mathcal{A}_{t}^{(\ell-2)}}{\max}\|y\|_{V_{\ell-1}^{-1}} \cdot \left(\beta_{t,\ell-1}^{(1)}\left(\frac{2}{T}\right) \wedge \beta_{t, \ell-1}^{(2)}\left(\frac{2}{T}\right) \right)\right] \\
    &\leq 4\!\!\sum_{t=\mathcal{T}_{\ell-1}+1}^{\mathcal{T}_\ell}\! \mathbb{E}\left[\underset{y \in \mathcal{A}_{t}^{(\ell-2)}}{\max}\|y\|_{V_{\ell-1}^{-1}} \right] \!\cdot \!\left(\!\sqrt{2\log\left({K(B-1)T^2}\right)} + \sqrt{\lambda}\bigwedge 2\sqrt{\log\bigg(\frac{2^{6d-5}\pi d(B-1)^2T^2}{15^{d-1}}\bigg)} +2\sqrt{\lambda}\!\right)\! .
\end{align*}
The last inequality follows from the fact that the size of the candidate arm set satisfies $|\mathcal{A}_t^{(\ell-2)}| \leq |\mathcal{A}_t| = K$ for all $t \in [\mathcal{T}_{\ell-1}+1, \mathcal{T}_\ell]$, allowing us to upper bound the confidence parameter $\beta_{t,\ell-1}^{(1)}$ uniformly over the action set.
\\\\
Using (\ref{eq:maximal_quadratic}), (\ref{eq:shift_uncertainty}), and (\ref{eq:elliptical_dlogd}), we can bound the summation $\sum_{t=\mathcal{T}_{\ell-1}+1}^{\mathcal{T}_\ell} \mathbb{E}\left[\max_{y \in \mathcal{A}_{t}^{(\ell-2)}}\|y\|_{V_{\ell-1}^{-1}} \right]$ for any $\ell \geq 2$ as
\begin{align*}
    \sum_{t=\mathcal{T}_{\ell-1}+1}^{\mathcal{T}_{\ell}} \mathbb{E}\left[\max_{y \in \mathcal{A}_{t}^{(\ell-2)}}\|y\|_{V_{\ell-1}^{-1}} \right]
    &\leq \sum_{t=\mathcal{T}_{\ell-1}+1}^{\mathcal{T}_\ell} \sqrt{ \mathbb{E}\left[\max_{y \in \mathcal{A}_{t}^{(\ell-2)}}y^\top V_{\ell-1}^{-1}y \right]} \\ 
    &= (\mathcal{T}_\ell - \mathcal{T}_{\ell-1}) \cdot \sqrt{\mathbb{E}\left[\max_{y \in \mathcal{A}_{t}^{(\ell-2)}}y^\top V_{\ell-1}^{-1}y \right]} \\
    &\leq \frac{\mathcal{T}_\ell - \mathcal{T}_{\ell-1}}{\sqrt{\mathcal{T}_{\ell-2}'' - \mathcal{T}_{\ell-2}'}} \cdot \sqrt{\sum_{s=\mathcal{T}_{\ell-2}'+1}^{\mathcal{T}_{\ell-2}''}  \mathbb{E} \bigg[ \max_{y \in \mathcal{A}_{s}^{(\ell-2)}} y^\top H_{s-1}^{-1}y  \bigg]} \\
    &\leq \frac{\mathcal{T}_\ell - \mathcal{T}_{\ell-1}}{\sqrt{\mathcal{T}''_{\ell-2} - \mathcal{T}'_{\ell-2}}} \cdot \sqrt{2d\log\left(\frac{1537d}{c^2}\right) + \frac{4d\log(2T)}{T}} \\
    &\leq \frac{\frac{T^{1-2^{-\ell}}}{\log_2\log_2T}+3}{\sqrt{\frac{c(1-c)T^{1-2^{1-\ell}}}{\log_2\log_2T}}} \cdot \sqrt{2d\log\left(\frac{1537d}{c^2}\right) + \frac{4d\log(2T)}{T}} \\
    &= \left(\!\sqrt{\frac{T}{c(1-c)\log_2\log_2T}} + 3\sqrt{\frac{\log_2\log_2T}{c(1-c)T^{1-2^{1-\ell}}}}\right)\! \sqrt{2d\log\left(\frac{1537d}{c^2}\right) + \frac{4d\log(2T)}{T}} \\
    &\leq \left(\sqrt{\frac{T}{c(1-c)\log_2\log_2T}} + \frac{3}{\sqrt{c(1-c)}}\right) \cdot \sqrt{2d\log d + 2\left(\log\left(\frac{1537}{c^2}\right)+2\right)d} \\ 
    &\leq c' \cdot \sqrt{\frac{dT\log d}{\log_2\log_2T}} \; ,
\end{align*}
where $t$ is arbitrary in the interval $[\mathcal{T}_{\ell-1}+1, \mathcal{T}_\ell]$, and we define  $c' \coloneqq 4\sqrt{\frac{2\left(\log\left(1537/c^2\right)+3\right)}{c(1-c)}}$. The first inequality applies Jensen’s inequality to move the square root outside the expectation. The first equality follows from (\ref{eq:maximal_quadratic}), while the second inequality uses the bound in (\ref{eq:shift_uncertainty}). The third inequality follows from (\ref{eq:elliptical_dlogd}). The fifth inequality follows from the facts that $\log_2\log_2T \leq T^{1-2^{1-\ell}}$ and $\log(2T) \leq T$ for all $T \geq 1$. The final inequality uses the bounds $1 \leq \frac{T}{\log_2\log_2T}$ and $d \leq d\log d$ for all $T \geq 1$ and $d \geq 3$.
\\\\
We now derive the cumulative expected regret after the first interval. Based on the previously derived results, the total regret from intervals $\ell = 2$ to $B$ can be bounded as follows
\begin{align*}
    \sum_{\ell=2}^{B} \text{Regret}_{\ell} &\leq  \sum_{\ell=2}^{B} 4c'\sqrt{\frac{dT\log d}{\log_2\log_2T}} \cdot \left(\sqrt{2\log\left({K(B-1)T^2}\right)} + \sqrt{\lambda} \bigwedge 2\sqrt{\log\bigg(\frac{2^{6d-5}\pi d(B-1)^2T^2}{15^{d-1}}\bigg)} +2\sqrt{\lambda}\right) \\
    &= \Ocal\left(\sqrt{dT\log d \log\log T} \cdot \left( \sqrt{\log(KT)} + \sqrt{\log{(dT)}}\wedge \sqrt{d\log\left(\frac{64}{15}\right) + \log(dT)} + \sqrt{\log(dT)} \right) \right) \\
    &= \Ocal \left(\sqrt{dT\log d \log\log T} \cdot \left( \left(\sqrt{\log(KT)} + \sqrt{\log{T}}\right) \wedge \left(\sqrt{d + \log T} + \sqrt{\log T} \right) \right)\right) \\ 
    &= \Ocal\left(\sqrt{dT\log d \log\log T} \cdot \left( \sqrt{\log(KT)} \wedge \sqrt{d + \log T}  \right) \right) \; .
\end{align*}
The first equality follows from substituting $B = \Ocal(\log\log T)$ and $\lambda = \log(dT)$. The second equality holds in the regime where $T \geq d$. This is a safe assumption, as for $T \leq d$, the total regret is trivially bounded by $\mathcal{R}(T) = \sum_{t=1}^{T}\mathbb{E}[\langle x_t^* - x_{t,a_t}, \theta^* \rangle] \leq 2T \leq 2\sqrt{dT} = \Ocal(\sqrt{dT})$, which is a much smaller bound.
\\\\
Thus, the total worst-case regret is bounded as
\begin{align*}
   \mathcal{R}(T) &\leq \sum_{k=1}^{B}\sum_{t=\mathcal{T}_{k-1}+1}^{\mathcal{T}_k} \mathbb{E}\bigg[\langle x_t^*-x_{t,a_t}, \theta^* \rangle\bigg| E \bigg] \\
   &= \text{Regret}_1 + \sum_{k=2}^{B} \text{Regret}_k \\
   &= \Ocal\big(\sqrt{T}\big) + \Ocal\left(\sqrt{dT\log d \log\log T} \cdot \left( \sqrt{\log(KT)} \wedge \sqrt{d + \log T}  \right) \right) \\
   &= \Ocal\left(\sqrt{dT} \left( \sqrt{\log(KT)} \wedge \sqrt{d + \log T}  \right) \sqrt{\log d \log\log T} \right) \; .
\end{align*}
Therefore, the worst-case regret for the algorithm is given by
\[
\mathcal{R}(T) = \Ocal\left(\sqrt{dT} \left( \sqrt{\log(KT)} \wedge \sqrt{d + \log T}  \right) \sqrt{\log d \log\log T} \right) = \tilde{\Ocal}\left(\sqrt{dT\log K} \wedge d\sqrt{T}\right) \; .
\]
\end{proof}
\clearpage
\section{Proof of Theorem 2}\label{proof:thm2}

\begin{proof}[Proof of Theorem 2]
As \cref{rmk:loosing_iid} encompasses the standard i.i.d.\ context assumption, we establish \cref{thm:BLCE} under the more general conditions specified therein.

To establish the claimed interval complexity bound, we analyze the lower bound on the length of the $\ell$-th interval, where $\ell = \lceil \log_2\log_2 T \rceil$. By definition of the schedule, the length of this interval is at least
\[
\Bigg\lceil \frac{T^{1-2^{-\ell}}}{\log_2\log_2T}\Bigg\rceil + 1  \geq \frac{T^{1-2^{-\lceil \log_2\log_2 T\rceil}}}{\log_2\log_2T} \geq \frac{T^{1-2^{- \log_2\log_2 T}}}{\log_2\log_2T} = \frac{T^{1-\frac{1}{\log_2T}}}{\log_2\log_2T} = \frac{T}{2\log_2\log_2T} \; .
\]
Since the interval length is non-decreasing in $\ell$, every interval with index $\ell \geq \lceil \log_2\log_2 T \rceil$ has length at least $\frac{T}{2\log_2\log_2T}$. Given the total time horizon $T$, the number of such intervals is therefore at most $\lceil 2 \log_2\log_2T \rceil$. Including the initial $\lceil  \log_2\log_2T \rceil - 1$ intervals, the total number of intervals $B$ is bounded by
\[ 
B \leq \lceil 2 \log_2\log_2T \rceil + \lceil  \log_2\log_2T \rceil - 1 \leq 3 \log_2\log_2T + 1 \; ,
\]
which implies that the interval complexity is $\Ocal(\log\log T)$.
\\\\
Next, we begin by decomposing the cumulative expected regret based on the good event $E$. The regret can be written as
\begin{align*}
\mathcal{R}(T) &= \sum_{k=1}^{B}\sum_{t=\mathcal{T}_{k-1}+1}^{\mathcal{T}_k} \mathbb{E}\bigg[\max_{x\in \mathcal{A}_t} \langle x,\theta^*\rangle - \langle x_{t,a_t},\theta^*\rangle \bigg] \\
&= \sum_{k=1}^{B}\sum_{t=\mathcal{T}_{k-1}+1}^{\mathcal{T}_k} \mathbb{E}\bigg[\langle x_t^*-x_{t,a_t}, \theta^* \rangle\bigg| E^c \bigg] \cdot \mathbb{P}(E^c) + \mathbb{E}\bigg[\langle x_t^*-x_{t,a_t}, \theta^* \rangle\bigg| E \bigg] \cdot \mathbb{P}(E) \; .
\end{align*}
Using the triangle inequality followed by the Cauchy-Schwarz inequality, we have
\begin{align*}
\langle x_t^*-x_{t,a_t}, \theta^* \rangle &\leq |\langle x_t^*-x_{t,a_t}, \theta^* \rangle| \\
&\leq |\langle x_t^*, \theta^* \rangle| + |\langle x_{t,a_t}, \theta^* \rangle| \\
&\leq \|x_t^*\|\cdot\|\theta^*\| + \|x_{t,a_t}\|\cdot\|\theta^*\| \\
&\leq 2 \; ,
\end{align*}
where the last inequality follows from Assumption~\ref{assum:norm}.
Hence, the cumulative expected regret can be bounded as
\begin{align*}
\mathcal{R}(T) &\leq 2T\cdot \mathbb{P}(E^c)+ \sum_{k=1}^{B}\sum_{t=\mathcal{T}_{k-1}+1}^{\mathcal{T}_k} \mathbb{E}\bigg[\langle x_t^*-x_{t,a_t}, \theta^* \rangle\bigg| E \bigg] \cdot \mathbb{P}(E) \\
&\leq \Ocal(1) + \sum_{k=1}^{B}\sum_{t=\mathcal{T}_{k-1}+1}^{\mathcal{T}_k} \mathbb{E}\bigg[\langle x_t^*-x_{t,a_t}, \theta^* \rangle\bigg| E \bigg] \; ,
\end{align*}
where the $\Ocal(1)$ term follows from the high-probability guarantee $\mathbb{P}(E^c) = \Ocal(1/T)$.
Throughout the remainder of the analysis, we therefore condition on the good event \(E\). Let $\text{Regret}_{\ell}$ denote the cumulative expected regret incurred during interval $\ell$. We analyze the regret separately for the case \(\ell = 1\) and for all subsequent intervals \(\ell \ge 2\).
\paragraph{Case 1: $\ell=1$.} In the first interval, the number of rounds is $\mathcal{T}_1 = \left\lceil \frac{\sqrt{T}}{\log_2\log_2 T}\right\rceil$, and the instantaneous regret is bounded by 2. Therefore,
\[
\text{Regret}_{1} = \sum_{t=1}^{\mathcal{T}_1} \mathbb{E}[\langle x_t^*-x_{t,a_t}, \theta^* \rangle] \leq 2 \mathcal{T}_1 \leq 2 \bigg( \frac{\sqrt{T}}{\log_2\log_2 T} + 1 \bigg) = \Ocal\big(\sqrt{T}\big) \; .
\]
\paragraph{Case 2: $\ell \geq 2$.} For each interval $\ell \geq 2$, the rounds following the arm elimination steps are divided into two phases:
\\
\textbf{Phase 1.} In the first phase of length $\Big\lceil \frac{cT^{1-2^{-\ell}}}{\log_2\log_2T}\Big\rceil$, the algorithm selects the most informative direction with respect to the current Gram matrix. For any round $t$ in these phases, the instantaneous regret satisfies
\begin{align*}
\langle x_t^* - x_{t,a_t}, \theta^* \rangle &= \langle x_t^* - x_{t,a_t}, \theta^* - \hat{\theta}_{\ell-1} \rangle + \langle x_t^* - x_{t,a_t}, \hat{\theta}_{\ell-1} \rangle \\
&\leq \langle x_t^* - x_{t,a_t}, \theta^* - \hat{\theta}_{\ell-1} \rangle + \langle x_{t}^{(\ell-1)} - x_{t,a_t}, \hat{\theta}_{\ell-1} \rangle \\
&\leq |\langle x_t^* , \hat{\theta}_{\ell-1} - \theta^* \rangle| + |\langle x_{t,a_t}, \hat{\theta}_{\ell-1}- \theta^* \rangle| + 2\varepsilon_{t,\ell-1} \\
&\leq 4\varepsilon_{t,\ell-1} \; . \tag{9} \label{eq:regret_upper_bound_4,5}
\end{align*}
The first inequality uses the fact that $x_{t}^{(\ell-1)}$ is optimal with respect to $\hat{\theta}_{\ell-1}$ and that $x_t^*$ belongs to $\mathcal{A}_t^{(\ell-2)}$ by \cref{lem:optimal_arm}. The second inequality follows since $x_{t,a_t} \in \mathcal{A}_{t}^{(\ell-1)}$ by construction of the arm elimination step. The final inequality follows from the definition of the good event $E$ and again from \cref{lem:optimal_arm}, which guarantees $x_t^* \in \mathcal{A}_t^{(\ell-2)}$.  
\\\\
\textbf{Phase 2.} In the last phase, the algorithm selects arms greedily with respect to the estimated parameter, i.e., $x_{t,a_t} \in \arg\max_{x \in \mathcal{A}_t^{(\ell-1)}}\langle x, \hat{\theta}_{\ell-1}\rangle$. For any round $t$ in this phase, the instantaneous regret satisfies
\begin{align*}
\langle x_t^* - x_{t,a_t}, \theta^* \rangle &= \langle x_t^* - x_{t,a_t}, \theta^* - \hat{\theta}_{\ell-1} \rangle + \langle x_t^* - x_{t,a_t}, \hat{\theta}_{\ell-1} \rangle \\
&\leq \langle x_t^* - x_{t,a_t}, \theta^* - \hat{\theta}_{\ell-1} \rangle \\
&\leq |\langle x_t^* , \hat{\theta}_{\ell-1} - \theta^* \rangle| + |\langle x_{t,a_t}, \hat{\theta}_{\ell-1}- \theta^* \rangle| \\
&\leq 2\varepsilon_{t,\ell-1} \; . \tag{10} \label{eq:regret_upper_bound_6}
\end{align*}
The first inequality follows from the fact that $x_{t,a_t}$ is optimal with respect to $\hat{\theta}_{\ell-1}$ and that $x_t^*$ belongs to $\mathcal{A}_t^{(\ell-1)}$ by \cref{lem:optimal_arm}. The final inequality follows directly from the definition of the good event $E$ and from \cref{lem:optimal_arm}, as in Phase 1.
\\\\
Therefore, it suffices to upper bound the quantity $\sum_{t=\mathcal{T}_\ell+1}^{\mathcal{T}_{\ell+1}}\varepsilon_{t,\ell}$ for each $\ell \in [B-1]$, which reduces to bounding the sum  $\sum_{t=\mathcal{T}_\ell+1}^{\mathcal{T}_{\ell+1}} \max_{{y \in \mathcal{A}_{t}^{(\ell-1)}}}\|y\|_{V_\ell^{-1}}$. Since both $V_\ell$ and the arm elimination rule---determined by $\hat{\theta}_1, \dots, \hat{\theta}_{\ell-1}$---are measurable with respect to $\mathcal{F}_{\mathcal{T}_{\ell}}$, they can be treated as fixed quantities conditional on this filtration. Moreover, conditional on 
$\mathcal{F}_{\mathcal{T}_\ell}$ the action sets $\{\mathcal{A}_s\}_{s = \mathcal{T}_\ell+1}^{\mathcal{T}_{\ell+2}}$ are \emph{identically distributed} with common law $\mathcal{D}_\ell$. Hence, by the tower property, for any $t,v \in [\mathcal{T}_\ell+1, \mathcal{T}_{\ell+1}]$ we have
\begin{align*}
    \mathbb{E}\bigg[ \max_{{y \in \mathcal{A}_{t}^{(\ell-1)}}} y^\top V_\ell^{-1}y \bigg] 
    &= \mathbb{E}\bigg[\mathbb{E}\bigg[ \max_{{y \in \mathcal{A}_t^{(\ell-1)}}} y^\top V_\ell^{-1}y \, \bigg| \mathcal{F}_{\mathcal{T}_{\ell}} \bigg] \bigg] \\
    &= \mathbb{E}\bigg[\mathbb{E}\bigg[ \max_{{y \in \mathcal{A}_v^{(\ell-1)}}} y^\top V_\ell^{-1}y \, \bigg| \mathcal{F}_{\mathcal{T}_{\ell}} \bigg] \bigg] \\
    &= 
    \mathbb{E}\bigg[ \max_{{y \in \mathcal{A}_{v}^{(\ell-1)}}} y^\top V_\ell^{-1}y \bigg] \; . \tag{11} \label{eq:maximal_quadratic_id}
\end{align*}
Define $\mathcal{T}_{\ell-1}' \coloneqq \mathcal{T}_{\ell-1} + \Big\lceil \frac{cT^{1-2^{-\ell}}}{\log_2\log_2T} \Big\rceil$ for all $\ell \geq 2$, while for the first interval we set $\mathcal{T}_{0}' \coloneqq \mathcal{T}_{1}$. Now fix any $t \in [\mathcal{T}_\ell+1, \mathcal{T}_{\ell+1}]$ and consider the interval $s \in [\mathcal{T}_{\ell-1}+1, \mathcal{T}_{\ell-1}']$. Then we obtain
\begin{align*}
    \sum_{s=\mathcal{T}_{\ell-1}+1}^{\mathcal{T}_{\ell-1}'} \mathbb{E}\bigg[ \max_{y \in \mathcal{A}_{t}^{(\ell-1)}} y^\top V_\ell^{-1}y \bigg] 
    &\leq  \sum_{s=\mathcal{T}_{\ell-1}+1}^{\mathcal{T}_{\ell-1}'} \mathbb{E} \bigg[ \max_{y \in \mathcal{A}_{t}^{(\ell-1)}} y^\top H_{s-1}^{-1}y \bigg] \\
    &= \sum_{s=\mathcal{T}_{\ell-1}+1}^{\mathcal{T}_{\ell-1}'} \mathbb{E} \bigg[\mathbb{E} \bigg[ \max_{y \in \mathcal{A}_{t}^{(\ell-1)}} y^\top H_{s-1}^{-1}y \, \bigg| \mathcal{F}_{\mathcal{T}_{\ell-1}} \vee \sigma(H_{s-1}) \bigg] \bigg] \\
    &= \sum_{s=\mathcal{T}_{\ell-1}+1}^{\mathcal{T}_{\ell-1}'} \mathbb{E} \bigg[\mathbb{E} \bigg[ \max_{y \in \mathcal{A}_{s}^{(\ell-1)}} y^\top H_{s-1}^{-1}y \, \bigg| \mathcal{F}_{\mathcal{T}_{\ell-1}} \vee \sigma(H_{s-1}) \bigg] \bigg] \\
    &= \sum_{s=\mathcal{T}_{\ell-1}+1}^{\mathcal{T}_{\ell-1}'} \mathbb{E} \bigg[ \max_{y \in \mathcal{A}_{s}^{(\ell-1)}} y^\top H_{s-1}^{-1}y  \bigg]
    \; . \tag{12} \label{eq:shift_uncertainty_id}
\end{align*}
The first inequality follows from monotonicity of the matrices, since $H_{s-1} \preceq V_\ell$ for all $s$ in the interval. The first and last equalities follow from the tower property. The second equality relies on the fact that, conditional on $\mathcal{F}_{\mathcal{T}_{\ell-1}} \vee \sigma(H_{s-1})$, both $H_{s-1}$ and the arm elimination rule (determined by $\hat{\theta}_1, \dots, \hat{\theta}_{\ell-1}$) are fixed. Moreover, conditional on $\mathcal{F}_{\mathcal{T}_{\ell-1}}$, the action set $\mathcal{A}_t$ is independent of $\{\mathcal{A}_v, x_{v,a_v}, r_v\}_{v=\mathcal{T}_{\ell-1}+1}^{\mathcal{T}_\ell}$ for $t \in [\mathcal{T}_{\ell}+1, \mathcal{T}_{\ell+1}]$, and  $\mathcal{A}_s$ is independent of $\{\mathcal{A}_u, x_{u,a_u}, r_u\}_{u=\mathcal{T}_{\ell-1}+1}^{s-1}$ for $s \in [\mathcal{T}_{\ell-1}+1, \mathcal{T}_{\ell}]$. Hence, conditional on $\mathcal{F}_{\mathcal{T}_{\ell-1}} \vee \sigma(H_{s-1})$, both $\mathcal{A}_t$ and $\mathcal{A}_s$ share the same conditional law $\mathcal{D}_{\ell-1}$, which justifies replacing $\mathcal{A}_t$ with $\mathcal{A}_s$ in the inner expectation.
\\\\
In Phase 1, the algorithm proceeds by selecting the most informative direction at each step with respect to the current Gram matrix. For all $\ell \geq 1$, we obtain
\begin{align*}
    \sum_{s=\mathcal{T}_{\ell-1}+1}^{\mathcal{T}_{\ell-1}'} \mathbb{E} \bigg[ \max_{y \in \mathcal{A}_{s}^{(\ell-1)}} y^\top H_{s-1}^{-1}y  \bigg] 
    &= \mathbb{E} \left[ \sum_{s=\mathcal{T}_{\ell-1}+1}^{\mathcal{T}_{\ell-1}'} \, \max_{y \in \mathcal{A}_{s}^{(\ell-1)}} y^\top H_{s-1}^{-1}y  \right] \\
    &= \mathbb{E} \left[ \sum_{s=\mathcal{T}_{\ell-1}+1}^{\mathcal{T}_{\ell-1}'} x_{s,a_s}^\top H_{s-1}^{-1}x_{s,a_s}  \right] \\
    &\leq \mathbb{E} \left[ 2\log \left( \frac{\det(H_{\mathcal{T}_{\ell-1}'})}{\det(\lambda I)} \right) \right] \\
    &\leq \mathbb{E} \left[ 2\log \left( \frac{(\tr(H_{\mathcal{T}_{\ell-1}'})/d)^d}{\lambda^d} \right) \right] \\
    &\leq 2d\log\left(1+ \frac{T}{d\lambda} \right) \\
    &\leq 2d\log(2T)
    \; , \tag{13} \label{eq:elliptical_potential_id}
\end{align*}
where the second equality follows from the arm-selection strategy of Phase 1, and the first inequality follows from \cref{cor:elliptical_potential}. The second inequality replaces the determinant by the bound $\det(H) \leq (\operatorname{tr}(H)/d)^d$, which is a consequence of the AM–GM inequality. The third inequality applies the fact that $\|x\|_2 \leq 1$ for all contexts, which implies $\operatorname{tr}(H_{\mathcal{T}_{\ell-1}'}) \leq d\lambda + T$.
\\\\
Now, we establish an upper bound on the cumulative expected regret. By combining the bounds derived in  (\ref{eq:regret_upper_bound_4,5}) and (\ref{eq:regret_upper_bound_6}), we can bound the cumulative expected regret for interval $\ell$, denoted by $\text{Regret}_\ell$, for any $\ell \geq 2$ as follows
\begin{align*}
    \text{Regret}_\ell &= \sum_{t=\mathcal{T}_{\ell-1}+1}^{\mathcal{T}_\ell} \mathbb{E}[\langle x_t^*-x_{t,a_t}, \theta^* \rangle] \\
    &\leq 4\sum_{t=\mathcal{T}_{\ell-1}+1}^{\mathcal{T}_\ell} \mathbb{E}[\varepsilon_{t,\ell-1}] \\
    &= 4\sum_{t=\mathcal{T}_{\ell-1}+1}^{\mathcal{T}_\ell} \mathbb{E}\left[\underset{y \in \mathcal{A}_{t}^{(\ell-2)}}{\max}\|y\|_{V_{\ell-1}^{-1}} \cdot \left(\beta_{t,\ell-1}^{(1)}\left(\frac{2}{T}\right) \wedge \beta_{t, \ell-1}^{(2)}\left(\frac{2}{T}\right) \right)\right] \\
    &\leq 4\!\!\sum_{t=\mathcal{T}_{\ell-1}+1}^{\mathcal{T}_\ell} \!\! \mathbb{E}\left[\underset{y \in \mathcal{A}_{t}^{(\ell-2)}}{\max}\|y\|_{V_{\ell-1}^{-1}} \right] \!\cdot \!\left(\!\sqrt{2\log\left({K(B-1)T^2}\right)} + \sqrt{\lambda} \bigwedge 2\sqrt{\log\bigg(\frac{2^{6d-5}\pi d(B-1)^2T^2}{15^{d-1}}\bigg)} +2\sqrt{\lambda}\right)\! .
\end{align*}
The last inequality follows from the fact that the size of the candidate arm set satisfies $|\mathcal{A}_t^{(\ell-2)}| \leq |\mathcal{A}_t| = K$ for all $t \in [\mathcal{T}_{\ell-1}+1, \mathcal{T}_\ell]$, allowing us to upper bound the confidence parameter $\beta_{t,\ell-1}^{(1)}$ uniformly over the action set.
\\\\
Using (\ref{eq:maximal_quadratic_id}), (\ref{eq:shift_uncertainty_id}), and (\ref{eq:elliptical_potential_id}), we can bound the summation $\sum_{t=\mathcal{T}_{\ell-1}+1}^{\mathcal{T}_\ell} \mathbb{E}\left[\max_{y \in \mathcal{A}_{t}^{(\ell-2)}}\|y\|_{V_{\ell-1}^{-1}} \right]$ for any $\ell \geq 2$ as
\begin{align*}
    \sum_{t=\mathcal{T}_{\ell-1}+1}^{\mathcal{T}_{\ell}} \mathbb{E}\left[\max_{y \in \mathcal{A}_{t}^{(\ell-2)}}\|y\|_{V_{\ell-1}^{-1}} \right]
    &\leq \sum_{t=\mathcal{T}_{\ell-1}+1}^{\mathcal{T}_\ell} \sqrt{ \mathbb{E}\left[\max_{y \in \mathcal{A}_{t}^{(\ell-2)}}y^\top V_{\ell-1}^{-1}y \right]} \\ 
    &= (\mathcal{T}_\ell - \mathcal{T}_{\ell-1}) \cdot \sqrt{\mathbb{E}\left[\max_{y \in \mathcal{A}_{t}^{(\ell-2)}}y^\top V_{\ell-1}^{-1}y \right]} \\
    &\leq \frac{\mathcal{T}_\ell - \mathcal{T}_{\ell-1}}{\sqrt{\mathcal{T}_{\ell-2}' - \mathcal{T}_{\ell-2}}} \cdot \sqrt{\sum_{s=\mathcal{T}_{\ell-2}+1}^{\mathcal{T}_{\ell-2}'}  \mathbb{E} \bigg[ \max_{y \in \mathcal{A}_{s}^{(\ell-2)}} y^\top H_{s-1}^{-1}y  \bigg]} \\
    &\leq \frac{\mathcal{T}_\ell - \mathcal{T}_{\ell-1}}{\sqrt{\mathcal{T}'_{\ell-2} - \mathcal{T}_{\ell-2}}} \cdot \sqrt{2d\log\left(2T\right)} \\
    &\leq \frac{\frac{T^{1-2^{-\ell}}}{\log_2\log_2T}+2}{\sqrt{\frac{cT^{1-2^{1-\ell}}}{\log_2\log_2T}}} \cdot \sqrt{2d\log\left(2T\right)} \\
    &= \left(\sqrt{\frac{T}{c\log_2\log_2T}} + 2\sqrt{\frac{\log_2\log_2T}{cT^{1-2^{1-\ell}}}}\right) \cdot \sqrt{2d\log\left(2T\right)}\\
    &\leq \left(\sqrt{\frac{T}{c\log_2\log_2T}} + \frac{2}{\sqrt{c}}\right) \cdot 2\sqrt{d\log T} \\ 
    &\leq c' \cdot \sqrt{\frac{dT\log T}{\log_2\log_2T}} \; ,
\end{align*}
where $t$ is arbitrary in the interval $[\mathcal{T}_{\ell-1}+1, \mathcal{T}_\ell]$, and we define  $c' \coloneqq \frac{6}{\sqrt{c}}$. The first inequality applies Jensen’s inequality to move the square root outside the expectation. The first equality follows from (\ref{eq:maximal_quadratic_id}), while the second inequality uses the bound in (\ref{eq:shift_uncertainty_id}). The third inequality follows from (\ref{eq:elliptical_potential_id}). The fifth inequality follows from the facts that $\log_2\log_2T \leq T^{1-2^{1-\ell}}$ and $\log(2T) \leq 2\log T$ for all $T \geq 2$. The final inequality uses the bounds $1 \leq \frac{T}{\log_2\log_2T}$ for all $T \geq 1$.
\\\\
We now derive the cumulative expected regret after the first interval. Building on the previously established results, the total regret incurred from intervals $\ell = 2$ to $B$ can be bounded as
\begin{align*}
    \sum_{\ell=2}^{B} \text{Regret}_{\ell} &\leq  \sum_{\ell=2}^{B} 4c'\sqrt{\frac{dT\log T}{\log_2\log_2T}} \cdot \left(\sqrt{2\log\left({K(B-1)T^2}\right)} + \sqrt{\lambda} \bigwedge 2\sqrt{\log\bigg(\frac{2^{6d-5}\pi d(B-1)^2T^2}{15^{d-1}}\bigg)} +2\sqrt{\lambda}\right) \\
    &= \Ocal\left(\sqrt{dT\log T \log\log T} \cdot \left( \sqrt{\log(KT)}  \wedge \sqrt{d\log\left(\frac{64}{15}\right) + \log(dT)} \, \right) \right) \\
    &= \Ocal \left(\sqrt{dT\log T \log\log T} \cdot \left( \sqrt{\log(KT)} \wedge \sqrt{d + \log T} \right) \right) \; .
\end{align*}
Here, the first equality follows from substituting $B = \Ocal(\log\log T)$ and $\lambda = \Ocal(1)$, while the second equality holds because the $\log(d)$ term is dominated by $d$.
\\\\
Thus, the total worst-case regret is bounded as
\begin{align*}
   \mathcal{R}(T) &\leq \sum_{k=1}^{B}\sum_{t=\mathcal{T}_{k-1}+1}^{\mathcal{T}_k} \mathbb{E}\bigg[\langle x_t^*-x_{t,a_t}, \theta^* \rangle\bigg| E \bigg] \\
   &= \text{Regret}_1 + \sum_{k=2}^{B} \text{Regret}_k \\
   &= \Ocal\big(\sqrt{T}\big) + \Ocal\left(\sqrt{dT\log T \log\log T} \cdot \left( \sqrt{\log(KT)} \wedge \sqrt{d + \log T}  \right) \right) \\
   &= \Ocal\left(\sqrt{dT} \left( \sqrt{\log(KT)} \wedge \sqrt{d + \log T}  \right) \sqrt{\log T \log\log T} \right) \; .
\end{align*}
Therefore, the worst-case regret for the algorithm is given by
\[
\mathcal{R}(T) = \Ocal\left(\sqrt{dT} \left( \sqrt{\log(KT)} \wedge \sqrt{d + \log T}  \right) \sqrt{\log T \log\log T} \right) = \tilde{\Ocal}\left(\sqrt{dT\log K} \wedge d\sqrt{T}\right) \; .
\]
\end{proof}

\clearpage
\section{Proof of Theorem 3}\label{sec:GLM_theorem_proof}

\begin{algorithm}[tb]
	\caption{\texttt{BGLE}}
        \label{bgle}
	\begin{algorithmic}[1]
    \Statex {\bfseries Input:} Horizon $T$;
regularization parameter $\lambda>0$;
known reward-support bound $R>0$;
known parameter-norm bound $S>0$;
within-interval allocation rate $c\in(0,1)$;
 
 \Statex \hrulefill

 \State \textbf{Set static schedule:} $\mathcal{T}_\ell
\leftarrow
\ell
\Big\lceil
\frac{T^{1/3}}{\log_2\log_2T}
\Big\rceil$
for $\ell\leq3$;\\
$\mathcal{T}_\ell
\leftarrow
\Bigg(
\mathcal{T}_{\ell-1}
+
\Bigg\lceil
\frac{
T^{1-\frac{1}{3\cdot2^{\ell-4}}}
}{
\log_2\log_2T
}
\Bigg\rceil
\Bigg)\wedge T$
for $\ell\geq4$;\\
$B\leftarrow
\min\{\ell\geq1:\mathcal{T}_\ell=T\}$;
 
 \State \textbf{Initialize:} $\lambda \gets R^2(d+\log T)$, $H_0 \gets  \lambda I$;

 \Statex \textbf{Auxiliary quantities:}

\Statex
$\displaystyle
\alpha_{t,k}(\lambda)
\coloneqq
\begin{cases}
\exp(-2RS), & k=1,\\[2pt]
\exp\!\left(
-R\left(
2S\wedge
\|x_{t,a_t}\|_{V_k^{-1}}\cdot\left(24RS\left(
\sqrt{d+\log T}
+
\frac{R(d+\log T)}{\sqrt{\lambda}}
\right)
+
2S\sqrt{\lambda}\right)
\right)
\right), & k\geq2;
\end{cases}
$

\Statex
$\displaystyle
\varepsilon'_{t,k}(\lambda)
\coloneqq
\max_{y\in\mathcal{A}_t^{(k-1)}}
\|y\|_{V_k^{-1}}\,\cdot
\left(24RS\left(
\sqrt{d+\log T}
+
\frac{R(d+\log T)}{\sqrt{\lambda}}
\right)
+
2S\sqrt{\lambda}\right),
\qquad k\geq2.
$

 \For{$t \gets 1,2,\dots,\mathcal{T}_1$}

 \State Pull arm $x_{t,a_t} \in \arg\max_{x \in \mathcal{A}_t}\|x\|_{H_{t-1}^{-1}}$;

 \State $H_t \gets H_{t-1} + x_{t,a_t}x_{t,a_t}^\top$;

 \EndFor

 \State $V_1 \gets H_{\mathcal{T}_1}$, $\hat{\theta}_1 \gets \arg\min_{\theta} \sum_{t=1}^{\mathcal{T}_1}\ell_{t}(\theta)$, $H_{\mathcal{T}_1} \gets \lambda I$;
 
 \For{$\ell \gets 2,\dots,B$}
 
 \For{$t \gets \mathcal{T}_{\ell-1}+1,\dots,\mathcal{T}_{\ell}$}

 \If{$\ell \geq 3$}
 
 \For{$k \gets 2,\dots,\ell-1 $}

 \State $x_t^{(k)} \gets \arg\max_{x \in \mathcal{A}_t^{(k-1)}}\langle x, \hat{\theta}_{k}\rangle$;

 \State $\mathcal{A}_t^{(k)} \gets \left\{x \in \mathcal{A}_t^{(k-1)} \,  \bigg| \, \langle \hat{\theta}_k, x_t^{(k)} - x \rangle \leq 2\varepsilon_{t,k}'(\lambda) \right\}$;
 
 \EndFor\EndIf
 


 \If{$t \leq \mathcal{T}_{\ell-1} + \big\lceil cT^{1-2^{{((4-\ell) \wedge 1)}}/3}/\log_2\log_2T \big\rceil$}

 \State Pull arm $x_{t,a_t} \in \arg\max_{x \in \mathcal{A}_t^{(\ell-1)}}\|x\|_{H_{t-1}^{-1}}$;

 \Else

 \State Pull arm $x_{t,a_t} \in \arg\max_{x \in \mathcal{A}_t^{(\ell-1)}}\langle x, \hat{\theta}_{\ell-1}\rangle$;

 \EndIf




 \State $H_t \gets H_{t-1} + \alpha_{t,\ell-1}(\lambda) \dot\mu(\langle x_{t,a_t}, \hat{\theta}_{\ell-1} \rangle)x_{t,a_t}x_{t,a_t}^\top$;
 
 \EndFor
 
 \State $V_\ell \gets H_{\mathcal{T}_\ell}$, $\hat{\theta}_{\ell} \gets \arg\min_{\theta} \sum_{t=\mathcal{T}_{\ell-1}+1}^{\mathcal{T}_\ell}\ell_{t}(\theta)$, $H_{\mathcal{T}_\ell} \gets \lambda I$;
 
 \EndFor
\end{algorithmic}
\end{algorithm}

\begin{table*}[t]
\centering
\begin{threeparttable}  
\scriptsize
\caption{
Worst-case regret, \emph{parameter-update} complexity, and time complexity comparison for generalized linear contextual bandits.
\textbf{\texttt{BGLE} uses only $\Ocal(\log\log T)$ parameter updates and has the lowest displayed time complexity among the listed methods.}
}
\label{table:GLM_comparison}
\begin{tabular}{lllll}
\toprule
\textbf{Paper} & \textbf{Worst-Case Regret} & \makecell[l]{\textbf{Parameter} \\ \textbf{Updates}} & \makecell[l]{\textbf{Context} \\ \textbf{Adaptive}} &\textbf{Time Complexity} \\
\midrule
\multirow{2}{*}{\citet{sawarni2024generalized}} 
  & $\Ocal\Big((RSd(\sqrt{d/\hat\kappa} \wedge \sqrt{R_{\dot\mu}\log d})\sqrt{\log T}\log\log T +R)\sqrt{T}\Big)$ & \multirow{2}{*}{$\Ocal(\log \log T)$} & \multirow{2}{*}{\texttt{No}} & $\Ocal(Kd^4T\log T + \mathcal{C}_{\mathrm{opt}}^{\mathrm{tot}})$ \\
  & $+ \, \Ocal\Big({(\kappa R_{\dot\mu}R^5S^2)}^{1/3}e^{2RS}d^2(\log T)^{2/3}\log\log T \cdot T^{\frac{1}{3}}\Big)$ & & & \makecell[l]{$+\, \Ocal\big(Kd^5(\kappa R_{\dot\mu})^{1/3}e^{2RS}$ \\ $(R^2S^2T\log^2 T)^{1/3}\big) $}\\
\midrule
\multirow{2}{*}{\cref{bgle}} 
  & $\Ocal \Big(RS\sqrt{d(d+\log T)\log T\log\log T / \hat\kappa} \cdot\sqrt{T}\Big)$ & \multirow{2}{*}{$\Ocal(\log \log T)$}& \multirow{2}{*}{\texttt{Yes}} & $\Ocal(Kd^2T\log\log T)$ \\
  & $+\, \Ocal \Big((R^2Se^{8RS}d(d+\log T)\log T\log\log T + \frac{R}{\log\log T})T^{\frac{1}{3}}\Big)$ & & & $+\,\Ocal(\mathcal{C}_{\mathrm{opt}}^{\mathrm{tot}})$\protect\footnotemark
  \\
\bottomrule
\end{tabular}
\end{threeparttable}
\end{table*}

\paragraph{Problem Setting: Generalized Linear Contextual Bandits with Batched Feedback} \label{problem_setting_GLB}
We next consider generalized linear contextual bandits, where rewards follow a one-parameter exponential family distribution.
Given an arm \(x\in\RR^d\) and an unknown parameter \(\theta^*\in\RR^d\), the reward \(r\) has density
$
p(r\,|\,x;\theta^*)
= \exp\!\bigl(r\langle x,\theta^*\rangle - m(\langle x,\theta^*\rangle) + h(r)\bigr)\,\nu(\mathrm{d}r),
$
with log-partition function \(m\), base measure \(\nu\), and link function \(\mu(z)\coloneqq m'(z)\).
We adopt the following standard assumptions.

\begin{assumption} \label{assum:GLM_norm}
\(\|x\|_2 \leq 1\) for all $x \in \mathcal{A}_t$, and $\|\theta^*\|_2 \leq S$ for a known constant $S>0$.
\end{assumption}

\begin{assumption} \label{assum:GLM_diff}
The log-partition function $m$ is convex and three times differentiable.
Equivalently, $\dot\mu \coloneqq m'' \geq 0$ and $m'''$ exists.
\end{assumption}

At each round $t \in [T]$, the learner observes an arm set
$\mathcal{A}_t = \{x_{t,1}, \ldots, x_{t,K}\} \subseteq \RR^d$
and selects $x_{t,a_t} \in \mathcal{A}_t$.
The reward $r_t$ is drawn from $p(r\,|\,x_{t,a_t};\theta^*)$ with natural parameter $\langle x_{t,a_t}, \theta^* \rangle$, and satisfies
$\mathbb{E}[r \,|\,x;\theta^*] = \mu(\langle x,\theta^* \rangle)$.
Performance is measured by the \emph{cumulative expected regret}
\[
\mathcal R(T)
= \mathbb{E}\!\left[\sum_{t=1}^T \bigl(\mu(\langle x_t^*,\theta^*\rangle)-\mu(\langle x_{t,a_t},\theta^*\rangle)\bigr)\right],
\]
where \(x_t^*\in\argmax_{x\in\mathcal A_t}\mu(\langle x,\theta^*\rangle)\) is an optimal arm.
Following \citet{sawarni2024generalized}, we assume i.i.d.\
contexts drawn from an unknown distribution $\mathcal D$
and rewards almost surely supported on $[0,R]$ for a known
constant $R>0$. This implies the self-concordance condition
$|\ddot\mu(z)|\leq R\dot\mu(z)$ for all $z\in\mathbb R$.
This property is crucial for our analysis.

We begin by assuming that the MLE estimator $\hat\theta$, obtained by minimizing the log-loss objective, always satisfies the boundedness condition $\|\hat\theta\|_2 \leq S$. 
If this condition does not hold, one may instead apply the non-convex projection technique of \citet{sawarni2024generalized}. 
The projected estimator preserves the same guarantees established in \citet{sawarni2024generalized}, up to a multiplicative factor of $2$. 
Therefore, the assumption $\|\hat\theta\|_2 \leq S$ can be made without loss of generality.

\begin{lemma} \label{lem:GLM_exponential_convex}
    For any $x \in [0,C]$, the following inequality holds
    \[
    e^x \leq \frac{x(e^C-1)}{C} + 1 \; .
    \]
\end{lemma}
\begin{proof}
    Apply the definition of convexity, $f\big((1-\alpha)a+\alpha b\big)\le (1-\alpha)f(a)+\alpha f(b)$, to $f(t)=e^t$ with $a=0$, $b=C$, and $\alpha=x/C\in[0,1]$. This gives
\[
e^{(1-\alpha)0+\alpha C} \le (1-\alpha)e^0+\alpha e^C
\quad\Rightarrow\quad
e^x \le 1 + \tfrac{x}{C}(e^C-1) \; ,
\]
which is the claim.
\end{proof}

\begin{lemma} \label{lem:GLM_concordance}
    For an exponential family distribution with log-partition function $m(\cdot)$, let $\mu(z) \coloneqq m'(z)$. Then, for all $x_1,x_2 \in \mathbb{R}$, we have
    \[
    e^{-R|x_2-x_1|}\dot\mu(x_2) \leq \dot\mu(x_1) \leq e^{R|x_2-x_1|}\dot\mu(x_2) \; .
    \]
\end{lemma}

\begin{proof}
    Without loss of generality, assume that $x_2 \geq x_1$. Define $h_1(x) \coloneqq \dot\mu(x)e^{Rx}$ and $h_2(x) \coloneqq \dot\mu(x)e^{-Rx}$. Differentiating these functions yields
    $h_1'(x) = (\ddot\mu(x)+R\dot\mu(x))e^{Rx}$ and $h_2'(x) = (\ddot\mu(x)-R\dot\mu(x))e^{-Rx}$. 
    By the self-concordance property, we have $h_1'(x) \geq 0$ and $h_2'(x) \leq 0$, which implies that $h_1(x)$ is non-decreasing and $h_2(x)$ is non-increasing. Consequently, $h_1(x_2) \geq h_1(x_1)$ and $h_2(x_2) \leq h_2(x_1)$, which together establish the desired inequality.
\end{proof}

\begin{lemma} \label{lem:GLM_ellipsoid} \citep{sawarni2024generalized}
    For each interval $\ell\geq 1$, let $r_{\mathcal{T}_{\ell-1}+1},\dots,r_{\mathcal{T}_\ell}$ denote independent random variables drawn from the canonical exponential family such that $\mathbb{E}[r_s] = \mu(\langle x_{s,a_s}, \theta^* \rangle)$ for some $\theta^* \in \mathbb{R}^d$. Define the maximum likelihood estimator by $\hat{\theta}_\ell = \arg\min_{\theta} \sum_{t=\mathcal{T}_{\ell-1}+1}^{\mathcal{T}_\ell}\ell_{t}(\theta)$, and let
    $
    V_{\ell}^* \coloneqq \lambda I + \sum_{s=\mathcal{T}_{\ell-1}+1}^{\mathcal{T}_\ell} \dot\mu(\langle x_{s,a_s}, \theta^* \rangle) x_{s,a_s}x_{s,a_s}^\top.
    $
    Then, with probability at least $1 - \frac{1}{T^2}$, the following inequality holds
    \[
    \|\hat{\theta}_\ell - \theta^*\|_{V_{\ell}^*} \leq 24RS \left(\sqrt{d+\log T} + \frac{R(d+\log T)}{\sqrt{\lambda}} \right) + 2S\sqrt{\lambda} \triangleq \beta(\lambda) \; .
    \]
\end{lemma}

\footnotetext{$\mathcal{C}_{\mathrm{opt}}^{\mathrm{tot}}$ denotes the total oracle cost of solving the log-loss minimization at interval boundaries. 
With $B$ intervals, $\mathcal{C}_{\mathrm{opt}}^{\mathrm{tot}}=\sum_{\ell=1}^{B}\mathcal{C}_{\mathrm{opt}}(\mathcal{T}_{\ell} - \mathcal{T}_{\ell-1},d)$, where
\(\mathcal{C}_{\mathrm{opt}}(n,d)\) is the cost of computing the unconstrained MLE from $n$ samples in $d$ dimensions.}

\begin{lemma} \label{lem:GLM_empirical_matrix}
For any interval $\ell \geq 2$, define $V_{\ell}^* \coloneqq \lambda I + \sum_{s=\mathcal{T}_{\ell-1}+1}^{\mathcal{T}_\ell} \dot\mu(\langle x_{s,a_s}, \theta^* \rangle) x_{s,a_s}x_{s,a_s}^\top$ and $V_\ell \coloneqq \lambda I + \sum_{s=\mathcal{T}_{\ell-1}+1}^{\mathcal{T}_\ell} \alpha_{s,\ell-1}(\lambda) \dot\mu(\langle x_{s,a_s}, \hat{\theta}_{\ell-1} \rangle) x_{s,a_s}x_{s,a_s}^\top$. 
Then, for every $\ell\geq 2$, with probability at least $1-\frac{1}{T^2}$, the following matrix inequality holds
\[
    V_\ell \preceq V_\ell^* \; .
\]
\end{lemma}

\begin{proof}
    We first consider the case $\ell = 2$, applying \cref{lem:GLM_concordance} yields
    \begin{align*}
        e^{-R|\langle x_{s,a_s}, \hat{\theta}_{1} - \theta^* \rangle|}\dot\mu(\langle x_{s,a_s}, \hat{\theta}_{1} \rangle) &\leq \dot\mu(\langle x_{s,a_s}, \theta^* \rangle) \; .
    \end{align*}
    By the assumptions $\|x_{s,a_s}\|_2 \leq 1$, $\|\theta^*\|_2 \leq S$, and $\|\hat{\theta}_{1}\|_2 \leq S$, we further obtain
    \begin{align*}
        |\langle x_{s,a_s}, \hat{\theta}_{1} - \theta^* \rangle| \leq \|x_{s,a_s}\|_2 \cdot \|\hat{\theta}_{1} - \theta^*\|_2 \leq \|\hat{\theta}_{1}\|_2 + \| \theta^*\|_2 \leq 2S\; .
    \end{align*}
    Consequently,
    \begin{align*}
        \alpha_{s,1}(\lambda) \dot\mu(\langle x_{s,a_s}, \hat{\theta}_{1} \rangle) &= e^{-2RS} \dot\mu(\langle x_{s,a_s}, \hat{\theta}_{1} \rangle) \\
        &\leq e^{-R|\langle x_{s,a_s}, \hat{\theta}_{1} - \theta^* \rangle|}\dot\mu(\langle x_{s,a_s}, \hat{\theta}_{1} \rangle) \\
        &\leq \dot\mu(\langle x_{s,a_s}, \theta^* \rangle) \; ,
    \end{align*}
    which establishes $V_2 \preceq V_2^*$.
    \\\\
    For the general case $\ell \geq 3$,  \cref{lem:GLM_concordance} gives
    \begin{align*}
        e^{-R|\langle x_{s,a_s}, \hat{\theta}_{\ell-1} - \theta^* \rangle|}\dot\mu(\langle x_{s,a_s}, \hat{\theta}_{\ell-1} \rangle) &\leq \dot\mu(\langle x_{s,a_s}, \theta^* \rangle) \; .
    \end{align*}
    Using \cref{lem:GLM_ellipsoid} together with the assumptions $\|x_{s,a_s}\|_2 \leq 1$, $\|\theta^*\|_2 \leq S$ and $\|\hat{\theta}_{\ell-1}\|_2 \leq S$, we obtain
    \begin{align*}
        |\langle x_{s,a_s}, \hat{\theta}_{\ell-1} - \theta^* \rangle| \leq (2S \wedge \|x_{s,a_s}\|_{V_{\ell-1}^{*-1}} \beta(\lambda)) \leq (2S \wedge \|x_{s,a_s}\|_{V_{\ell-1}^{-1}} \beta(\lambda)) \; ,
    \end{align*}
    where the second inequality follows inductively from $V_{\ell-1} \preceq V_{\ell-1}^*$. Therefore,
    \begin{align*}
        \alpha_{s,\ell-1}(\lambda) \dot\mu(\langle x_{s,a_s}, \hat{\theta}_{\ell-1} \rangle) &= e^{-R(2S \wedge  \|x_{s,a_s}\|_{V_{\ell-1}^{-1}} \beta(\lambda))} \dot\mu(\langle x_{s,a_s}, \hat{\theta}_{\ell-1} \rangle) \\
        &\leq e^{-R|\langle x_{s,a_s}, \hat{\theta}_{\ell-1} - \theta^* \rangle|}\dot\mu(\langle x_{s,a_s}, \hat{\theta}_{\ell-1} \rangle) \\
        &\leq \dot\mu(\langle x_{s,a_s}, \theta^* \rangle) \; ,
    \end{align*}
    which completes the proof that $V_\ell \preceq V_\ell^*$ for all $\ell \geq 2$.
\end{proof}

\begin{lemma}[Good event] \label{lem:GLM_good_event}
Define the following quantities:
\begin{align*}
\beta(\lambda) &\coloneqq 24RS \left(\sqrt{d+\log T} + \frac{R(d+\log T)}{\sqrt{\lambda}} \right) + 2S\sqrt{\lambda}\,,\\
\varepsilon_{t,\ell}'(\lambda) &\coloneqq \underset{y \in \mathcal{A}_{t}^{(\ell-1)}}{\max}\|y\|_{V_\ell^{-1}} \cdot \beta(\lambda)\; \text{ for } \; \ell \geq 2 \; .
\end{align*}
Then, with probability at least $1-\frac{2(B-2)}{T^2}$, the following event holds 
\begin{align*}
E' \coloneqq& \bigcap\limits_{\ell = 2}^{B-1}\bigcap\limits_{t = \mathcal{T}_{\ell}+1}^{T}\left\{|\langle x,\hat{\theta}_\ell - \theta^* \rangle |
\leq \varepsilon_{t,\ell}'(\lambda), \; \forall x \in \mathcal{A}_{t}^{(\ell-1)}\right\}
\; .
\end{align*}
\end{lemma}

\begin{proof}
 For any interval $\ell\geq 2$, by the Cauchy-Schwarz inequality together with \cref{lem:GLM_ellipsoid} and \cref{lem:GLM_empirical_matrix}, it follows that for every round $t$ and for all $x \in \mathcal{A}_{t}^{(\ell-1)}$, with probability at least $1-\tfrac{2}{T^2}$ we have
\[
| \langle x, \hat{\theta}_{\ell} - \theta^* \rangle | \leq \|x\|_{{V_{{\ell}}^*}^{-1}}\|\hat{\theta}_{\ell} - \theta^*\|_{V_{{\ell}}^*} \leq \|x\|_{V_\ell^{-1}} \cdot \beta(\lambda) \leq \underset{y \in \mathcal{A}_{t}^{(\ell-1)}}{\max}\|y\|_{V_\ell^{-1}} \cdot \beta(\lambda) \; . 
\]
Applying a union bound over all intervals $\ell$ then guarantees that the event $E'$ holds with probability at least $1-\frac{2(B-2)}{T^2}$, which completes the proof.
\end{proof}

\begin{lemma} \label{lem:GLM_optimal_arm}
Let $E'$ be the good event defined in \cref{lem:GLM_good_event}. Conditioned on $E'$, the optimal arm $x_{t}^* \in  \arg\max_{x \in \mathcal{A}_t} \langle x, \theta^* \rangle$ is never eliminated at any round $t$. In particular,
\begin{equation*}
x_t^* \in \mathcal{A}_t^{(\ell)}, \quad \text{for all } \, 1 \leq \ell \leq B-1 \, \text{ and } \, \mathcal{T}_\ell+1 \leq t \leq T \; .
\end{equation*}
\end{lemma}
\begin{proof}
Fix $t\in[\mathcal{T}_{s}+1,\mathcal{T}_{s+1}]$ for some $s\in[B-1]$. We show by induction on $\ell$ that $x_t^*\in\mathcal{A}_t^{(\ell)}$ for all $\ell\in[s]$.  
\\\\
\textbf{Base case ($\ell = 1$).} 
By definition we have $\mathcal{A}_t = \mathcal{A}_t^{(0)} = \mathcal{A}_t^{(1)}$, which immediately implies that $x_t^*\in \mathcal{A}_t^{(1)}$ holds trivially.
\\\\
\textbf{Inductive step.} Assume $x_t^* \in \mathcal{A}_t^{(\ell-1)}$ for some $\ell \in \{2, \dots, s\}$. Since $x_{t}^{(\ell)} \in \mathcal{A}_t^{(\ell-1)}$, we similarly obtain
\begin{align*}
\langle \hat{\theta}_\ell, x_{t}^{(\ell)} - x_t^* \rangle
&= \langle \hat{\theta}_\ell - \theta^*, x_{t}^{(\ell)} - x_t^* \rangle + \langle \theta^*, x_{t}^{(\ell)} - x_t^* \rangle \\
&\leq \langle \hat{\theta}_\ell - \theta^*, x_{t}^{(\ell)} - x_t^* \rangle \\
&\leq |\langle \hat{\theta}_\ell - \theta^*, x_{t}^{(\ell)} \rangle| + |\langle \hat{\theta}_\ell - \theta^*, x_t^* \rangle| \\
&\leq 2\varepsilon_{t,\ell}'(\lambda) \; ,
\end{align*}
which shows that $x_t^* \in \mathcal{A}_t^{(\ell)}$. By induction, the claim holds for all $\ell \in [s]$, completing the proof.
\end{proof}

\begin{proof}[Proof of Theorem 3] To establish the claimed interval complexity bound, we analyze the lower bound on the length of the $\ell$-th interval, where $\ell = \lceil \log_2\log_2 T \rceil + 4$. By definition of the schedule, the length of this interval is at least
\[
\Bigg\lceil \frac{T^{1-\frac{1}{3\cdot2^{{\ell-4}}}}}{\log_2\log_2T}\Bigg\rceil \geq \frac{T^{1-2^{-\lceil \log_2\log_2 T\rceil}}}{\log_2\log_2T} \geq \frac{T^{1-2^{- \log_2\log_2 T}}}{\log_2\log_2T} = \frac{T^{1-\frac{1}{\log_2T}}}{\log_2\log_2T} = \frac{T}{2\log_2\log_2T} \; .
\]
Since the interval length is non-decreasing in $\ell$, every interval with index $\ell \geq \lceil \log_2\log_2 T \rceil + 4$ has length at least $\frac{T}{2\log_2\log_2T}$. Given the total time horizon $T$, the number of such intervals is therefore at most $\lceil 2 \log_2\log_2T \rceil$. Including the initial $\lceil  \log_2\log_2T \rceil + 3$ intervals, the total number of intervals $B$ is bounded by
\[ 
B \leq \lceil 2 \log_2\log_2T \rceil + \lceil  \log_2\log_2T \rceil + 3\leq 3 \log_2\log_2T + 5 \; ,
\]
which implies that the interval complexity is $\Ocal(\log\log T)$.
\\\\
Next, we decompose the cumulative expected regret with respect to the good event $E'$. The regret can be expressed as
\begin{align*}
\mathcal{R}(T) &= \sum_{k=1}^{B}\sum_{t=\mathcal{T}_{k-1}+1}^{\mathcal{T}_k} \mathbb{E}\bigg[\max_{x\in \mathcal{A}_t} \mu(\langle x,\theta^*\rangle) - \mu(\langle x_{t,a_t},\theta^*\rangle) \bigg] \\
&= \sum_{k=1}^{B}\sum_{t=\mathcal{T}_{k-1}+1}^{\mathcal{T}_k} \mathbb{E}\bigg[\mu(\langle x_t^*, \theta^*\rangle)-\mu(\langle x_{t,a_t}, \theta^* \rangle)\bigg| {E'}^c \bigg] \cdot \mathbb{P}({E'}^c) + \mathbb{E}\bigg[\mu(\langle x_t^*, \theta^*\rangle)-\mu(\langle x_{t,a_t}, \theta^* \rangle)\bigg| E' \bigg] \cdot \mathbb{P}(E') \; .
\end{align*}
\\
Since the rewards are supported on $[0,R]$ and $\mathbb{E}[r \,|\,x;\theta^*] = \mu(\langle x,\theta^* \rangle)$ holds, the instantaneous regret is bounded by the support width, i.e.,
\begin{align*}
    \mu(\langle x_t^*, \theta^*\rangle)-\mu(\langle x_{t,a_t}, \theta^* \rangle) \leq R \; . \tag{14} \label{eq:GLM_regret_bound}
\end{align*}
Hence, the cumulative expected regret can be bounded as
\begin{align*}
\mathcal{R}(T) &\leq RT \cdot \frac{2(B-2)}{T^2} + \sum_{k=1}^{B}\sum_{t=\mathcal{T}_{k-1}+1}^{\mathcal{T}_k} \mathbb{E}\bigg[\mu(\langle x_t^*, \theta^*\rangle)-\mu(\langle x_{t,a_t}, \theta^* \rangle)\bigg| E' \bigg] \cdot \mathbb{P}(E') \\
&\leq \frac{2R(B-2)}{T} + \sum_{k=1}^{B}\sum_{t=\mathcal{T}_{k-1}+1}^{\mathcal{T}_k} \mathbb{E}\bigg[\mu(\langle x_t^*, \theta^*\rangle)-\mu(\langle x_{t,a_t}, \theta^* \rangle)\bigg| E' \bigg] \\
&= \Ocal \left(\frac{R\log\log T}{T} \right) + \sum_{k=1}^{B}\sum_{t=\mathcal{T}_{k-1}+1}^{\mathcal{T}_k} \mathbb{E}\bigg[\mu(\langle x_t^*, \theta^*\rangle)-\mu(\langle x_{t,a_t}, \theta^* \rangle)\bigg| E' \bigg] \; .
\end{align*}
Throughout the remainder of the analysis, we condition on the good event \(E'\). Let $\text{Regret}_{\ell}$ denote the cumulative expected regret incurred during interval $\ell$. We analyze the regret by separating the discussion into four cases, namely $\ell=1$, $\ell=2$, $\ell=3$, and the subsequent intervals with $\ell \ge 4$.

\paragraph{Case 1: $\ell \in \{1,2,3\}$.} In the first three intervals, the number of rounds is given by $\Big\lceil \frac{\sqrt[3]{T}}{\log_2\log_2 T}\Big\rceil$. By (\ref{eq:GLM_regret_bound}), each round incurs an instantaneous regret of at most $R$. Consequently, for the first interval we obtain
\begin{align*}
\text{Regret}_{1} = \sum_{t=1}^{\mathcal{T}_1} \mathbb{E}[\mu(\langle x_t^*,\theta^*\rangle) - \mu(\langle x_{t,a_t},\theta^*\rangle)] 
&\leq R\mathcal{T}_1 \leq R \bigg( \frac{\sqrt[3]{T}}{\log_2\log_2 T} + 1 \bigg) = \Ocal\left(\frac{R\sqrt[3]{T}}{\log\log T}\right) \; .
\end{align*}
The same reasoning applies to the second and third intervals, which yields the same order of regret,
\[
\text{Regret}_{2} = \text{Regret}_{3} = \Ocal\!\left(\frac{R\sqrt[3]{T}}{\log\log T}\right) \; .
\]
\paragraph{Case 2: $\ell \geq 4$.} For each interval $\ell \geq 4$, the instantaneous regret can be controlled using the Mean Value Theorem. For some $z$ lying between $\langle x_{t,a_t}, \theta^* \rangle$ and $\langle x_{t}^*, \theta^* \rangle$, we have
\begin{align*}
    \mu(\langle x_t^*,\theta^*\rangle) - \mu(\langle x_{t,a_t},\theta^*\rangle) &= \dot\mu(z)\langle x_t^* - x_{t,a_t}, \theta^* \rangle \\
    &= \dot\mu(z)(\langle x_t^* - x_{t,a_t}, \theta^* - \hat{\theta}_{\ell-1} \rangle + \langle x_t^* - x_{t,a_t}, \hat{\theta}_{\ell-1} \rangle) \\
    &\leq \dot\mu(z)(|\langle x_t^* , \hat{\theta}_{\ell-1} - \theta^* \rangle| + |\langle x_{t,a_t}, \hat{\theta}_{\ell-1}- \theta^* \rangle| + \langle x_t^{(\ell-1)} - x_{t,a_t}, \hat{\theta}_{\ell-1} \rangle) \\
    &\leq \underbrace{4\dot\mu(z)\varepsilon_{t,\ell-1}'(\lambda)}_{\triangleq A_t}
    \; . \tag{15} \label{eq:GLM_instant_bound}
\end{align*}
The first inequality uses the fact that $x_{t}^{(\ell-1)}$ is optimal with respect to $\hat{\theta}_{\ell-1}$, together with the guarantee from \cref{lem:GLM_optimal_arm} that $x_t^*$ belongs to $\mathcal{A}_t^{(\ell-2)}$. The final bound follows because $x_{t,a_t} \in \mathcal{A}_{t}^{(\ell-1)}$ by the arm elimination rule, combined with the definition of the good event $E'$ and again \cref{lem:GLM_optimal_arm}, which ensures that $x_t^* \in \mathcal{A}_t^{(\ell-2)}$.
As in the analysis of \cref{BLCE}, during the greedy selection step the regret can be bounded more tightly as $\mu(\langle x_t^*,\theta^*\rangle) - \mu(\langle x_{t,a_t},\theta^*\rangle)$ by $2\dot\mu(z)\varepsilon_{t,\ell-1}'(\lambda)$, analogous to \cref{eq:regret_upper_bound_6}. For simplicity in the subsequent analysis, we substitute the greedy selection step with uncertainty-driven exploration, so that in all intervals the algorithm may be analyzed under uncertainty-driven exploration alone. For rigor, interval end time $\mathcal{T}_\ell$ must be replaced by $\mathcal{T}_{\ell-1}' \coloneqq \mathcal{T}_{\ell-1} + \Big\lceil \tfrac{cT^{1-2^{(4-\ell)}/3}}{\log_2\log_2T} \Big\rceil$, as in the proof of \cref{thm:BLCE}, but this modification affects the regret bound only by a constant factor.
\paragraph{Bounding $\sum_{t=\mathcal{T}_{\ell-1}+1}^{\mathcal{T}_\ell}\mathbb{E}[A_t]$.}
\begin{align*}
    \sum_{t=\mathcal{T}_{\ell-1}+1}^{\mathcal{T}_\ell}\mathbb{E}[A_t] 
    &= 4\sum_{t=\mathcal{T}_{\ell-1}+1}^{\mathcal{T}_\ell} \mathbb{E}[\dot\mu(z)\varepsilon_{t,\ell-1}'(\lambda)] \\
    &\leq 4\sum_{t=\mathcal{T}_{\ell-1}+1}^{\mathcal{T}_\ell} \mathbb{E}\left[e^{R(\langle x_t^*, \theta^* \rangle - z)}\dot\mu(\langle x_t^*,\theta^*\rangle)\, \varepsilon_{t,\ell-1}'(\lambda)\right] \\
    &\leq 4\sum_{t=\mathcal{T}_{\ell-1}+1}^{\mathcal{T}_\ell} \mathbb{E}\left[e^{R\langle x_t^* - x_{t,a_t}, \theta^* \rangle }\dot\mu(\langle x_t^*,\theta^*\rangle)\, \varepsilon_{t,\ell-1}'(\lambda)\right] \\
    &\leq 4\sum_{t=\mathcal{T}_{\ell-1}+1}^{\mathcal{T}_\ell} \mathbb{E}\left[e^{4R(S \wedge \varepsilon_{t,\ell-1}'(\lambda))}\dot\mu(\langle x_t^*,\theta^*\rangle)\, \varepsilon_{t,\ell-1}'(\lambda)\right] \\
    &\leq 4\sum_{t=\mathcal{T}_{\ell-1}+1}^{\mathcal{T}_\ell} \mathbb{E}\left[\smash[b]{\underbrace{\dot\mu(\langle x_t^*,\theta^*\rangle)\, \varepsilon_{t,\ell-1}'(\lambda)}_{\triangleq B_t}\vphantom{\bigg|}}\right] + \mathbb{E}\left[\smash[b]{\underbrace{\frac{e^{4RS}(S \wedge \varepsilon_{t,\ell-1}'(\lambda))}{S}\dot\mu(\langle x_t^*,\theta^*\rangle)\, \varepsilon_{t,\ell-1}'(\lambda)}_{\triangleq C_t}\vphantom{\Bigg|}}\right] \; .
    \\
\end{align*}
The first inequality is a direct application of \cref{lem:GLM_concordance}. The second inequality holds because $z$ lies between $\langle x_{t,a_t}, \theta^* \rangle$ and $\langle x_{t}^*, \theta^* \rangle$. The third inequality follows from Assumption~\ref{assum:norm} together with the bound on $\langle x_t^*-x_{t,a_t}, \theta^* \rangle$ derived in (\ref{eq:GLM_instant_bound}). Finally, the last inequality is obtained by invoking \cref{lem:GLM_exponential_convex}.

\paragraph{Bounding $\sum_{t=\mathcal{T}_{\ell-1}+1}^{\mathcal{T}_\ell}\mathbb{E}[B_t]$.}
\begin{align*}
    \sum_{t=\mathcal{T}_{\ell-1}+1}^{\mathcal{T}_\ell}\mathbb{E}[B_t]
    &= \beta(\lambda)\sum_{t=\mathcal{T}_{\ell-1}+1}^{\mathcal{T}_\ell}\mathbb{E}\left[\max_{y \in \mathcal{A}_t^{(\ell-2)}} \|\dot\mu(\langle x_t^*,\theta^*\rangle)y\|_{V_{\ell-1}^{-1}}\right]\\
    &= \beta(\lambda)(\mathcal{T}_\ell - \mathcal{T}_{\ell-1})\mathbb{E}\left[\max_{y \in \mathcal{A}_t^{(\ell-2)}} \|\dot\mu(\langle x_t^*,\theta^*\rangle)y\|_{V_{\ell-1}^{-1}}\right] \tag{16} \label{eq:GLM_const_max_weight_norm}\\
    &= \frac{\beta(\lambda)(\mathcal{T}_\ell - \mathcal{T}_{\ell-1})}{\mathcal{T}_{\ell-1} - \mathcal{T}_{\ell-2}}\sum_{s=\mathcal{T}_{\ell-2}+1}^{\mathcal{T}_{\ell-1}}\mathbb{E}\left[\max_{y \in \mathcal{A}_t^{(\ell-2)}} \|\dot\mu(\langle x_t^*,\theta^*\rangle)y\|_{V_{\ell-1}^{-1}}\right]\\
    &\leq \frac{\beta(\lambda)(\mathcal{T}_\ell - \mathcal{T}_{\ell-1})}{\mathcal{T}_{\ell-1} - \mathcal{T}_{\ell-2}}\sum_{s=\mathcal{T}_{\ell-2}+1}^{\mathcal{T}_{\ell-1}}\mathbb{E}\left[\max_{y \in \mathcal{A}_t^{(\ell-2)}} \|\dot\mu(\langle x_t^*,\theta^*\rangle)y\|_{H_{s-1}^{-1}}\right]\\
    &= \frac{\beta(\lambda)(\mathcal{T}_\ell - \mathcal{T}_{\ell-1})}{\mathcal{T}_{\ell-1} - \mathcal{T}_{\ell-2}}\sum_{s=\mathcal{T}_{\ell-2}+1}^{\mathcal{T}_{\ell-1}}\mathbb{E}\left[\mathbb{E}\left[\max_{y \in \mathcal{A}_t^{(\ell-2)}} \|\dot\mu(\langle x_t^*,\theta^*\rangle)y\|_{H_{s-1}^{-1}} \Bigg| \mathcal{F}_{s-1}\right]\right]\\
    &= \frac{\beta(\lambda)(\mathcal{T}_\ell - \mathcal{T}_{\ell-1})}{\mathcal{T}_{\ell-1} - \mathcal{T}_{\ell-2}}\sum_{s=\mathcal{T}_{\ell-2}+1}^{\mathcal{T}_{\ell-1}}\mathbb{E}\left[\mathbb{E}\left[\max_{y \in \mathcal{A}_s^{(\ell-2)}} \|\dot\mu(\langle x_s^*,\theta^*\rangle)y\|_{H_{s-1}^{-1}} \Bigg| \mathcal{F}_{s-1}\right]\right]\\
    &= \frac{\beta(\lambda)(\mathcal{T}_\ell - \mathcal{T}_{\ell-1})}{\mathcal{T}_{\ell-1} - \mathcal{T}_{\ell-2}}\sum_{s=\mathcal{T}_{\ell-2}+1}^{\mathcal{T}_{\ell-1}}\mathbb{E}\left[\max_{y \in \mathcal{A}_s^{(\ell-2)}} \|\dot\mu(\langle x_s^*,\theta^*\rangle)y\|_{H_{s-1}^{-1}}\right]\\
    &= \frac{\beta(\lambda)(\mathcal{T}_\ell - \mathcal{T}_{\ell-1})}{\mathcal{T}_{\ell-1} - \mathcal{T}_{\ell-2}}\sum_{s=\mathcal{T}_{\ell-2}+1}^{\mathcal{T}_{\ell-1}}\mathbb{E}\left[\smash[b]{\underbrace{ \|\dot\mu(\langle x_s^*,\theta^*\rangle)x_{s,a_s}\|_{H_{s-1}^{-1}}}_{\triangleq D_s}\vphantom{\bigg|}}\right] \; .
    \\
\end{align*}
The second equality mirrors the reasoning in (\ref{eq:maximal_quadratic}), since both $V_{\ell-1}$ and the arm elimination rule---determined by $\hat{\theta}_1, \dots, \hat{\theta}_{\ell-2}$---are measurable with respect to $\mathcal{F}_{\mathcal{T}_{\ell-1}}$, and can therefore be regarded as fixed conditional on this filtration. Given that the contexts are drawn independently and identically, their values are equal. The first inequality follows from the monotonicity, as $H_{s-1} \preceq V_{\ell-1}$ for all $s$ in the interval. The fourth and sixth equalities use the tower property. The fifth equality relies on the fact that, conditional on $\mathcal{F}_{s-1}$, both $H_{s-1}$ and the arm elimination rule (determined by $\hat{\theta}_1, \dots, \hat{\theta}_{\ell-2}$) are fixed, while the distribution of the contexts remains unchanged. Finally, the last equality follows from the arm-selection strategy, since the factor $\dot\mu(\langle x_s^*,\theta^* \rangle)$ does not affect the maximization and can thus be pulled outside without altering the $\arg\max$.
\paragraph{Bounding $\sum_{s=\mathcal{T}_{\ell-2}+1}^{\mathcal{T}_{\ell-1}}\mathbb{E}[D_s]$.}
\begin{align*}
    \sum_{s=\mathcal{T}_{\ell-2}+1}^{\mathcal{T}_{\ell-1}}\mathbb{E}[D_s]
    &\leq \sum_{s=\mathcal{T}_{\ell-2}+1}^{\mathcal{T}_{\ell-1}}\mathbb{E}\left[\left\|e^{\frac{R}{2}|\langle x_s^*,\theta^*\rangle - \langle x_{s,a_s},\hat{\theta}_{\ell-2}\rangle|}\sqrt{\dot\mu(\langle x_{s}^*,\theta^*\rangle)\dot\mu(\langle x_{s,a_s},\hat{\theta}_{\ell-2}\rangle)}x_{s,a_s}\right\|_{H_{s-1}^{-1}}\right]\\
    &\leq \sum_{s=\mathcal{T}_{\ell-2}+1}^{\mathcal{T}_{\ell-1}}\mathbb{E}\left[\left\|e^{\frac{R}{2}(2S \wedge 5\varepsilon_{s,\ell-2}'(\lambda))}\sqrt{\dot\mu(\langle x_{s}^*,\theta^*\rangle)\dot\mu(\langle x_{s,a_s},\hat{\theta}_{\ell-2}\rangle)}x_{s,a_s}\right\|_{H_{s-1}^{-1}}\right]\\
    &\leq \sum_{s=\mathcal{T}_{\ell-2}+1}^{\mathcal{T}_{\ell-1}}\mathbb{E}\left[\left\|e^{\frac{R}{2}(2S \wedge 5\varepsilon_{s,\ell-2}'(\lambda))+\frac{R}{2}(2S \wedge \varepsilon_{s,\ell-2}'(\lambda))}\sqrt{\dot\mu(\langle x_{s}^*,\theta^*\rangle)\alpha_{s,\ell-2}(\lambda)\dot\mu(\langle x_{s,a_s},\hat{\theta}_{\ell-2}\rangle)}x_{s,a_s}\right\|_{H_{s-1}^{-1}}\right]\\
    &\leq \sum_{s=\mathcal{T}_{\ell-2}+1}^{\mathcal{T}_{\ell-1}}\mathbb{E}\left[\left\|e^{3R(S \wedge \varepsilon_{s,\ell-2}'(\lambda))}\sqrt{\dot\mu(\langle x_{s}^*,\theta^*\rangle)\alpha_{s,\ell-2}(\lambda)\dot\mu(\langle x_{s,a_s},\hat{\theta}_{\ell-2}\rangle)}x_{s,a_s}\right\|_{H_{s-1}^{-1}}\right]\\
    &\leq \sum_{s=\mathcal{T}_{\ell-2}+1}^{\mathcal{T}_{\ell-1}}\mathbb{E}\left[\underbrace{\left\|\sqrt{\dot\mu(\langle x_{s}^*,\theta^*\rangle)\alpha_{s,\ell-2}(\lambda)\dot\mu(\langle x_{s,a_s},\hat{\theta}_{\ell-2}\rangle)}x_{s,a_s}\right\|_{H_{s-1}^{-1}}}_{\triangleq E_s}\right] \\  &\hspace{4.8em}+\mathbb{E}\left[\smash[b]{\underbrace{\left\|\frac{e^{3RS}(S \wedge \varepsilon_{s,\ell-2}'(\lambda))}{S}\sqrt{\dot\mu(\langle x_{s}^*,\theta^*\rangle)\alpha_{s,\ell-2}(\lambda)\dot\mu(\langle x_{s,a_s},\hat{\theta}_{\ell-2}\rangle)}x_{s,a_s}\right\|_{H_{s-1}^{-1}}}_{\triangleq F_s}\vphantom{\Bigg|}}\right] \; .
    \\
\end{align*}
The first inequality is obtained by applying \cref{lem:GLM_concordance} to compare $\dot\mu(\langle x_{s}^*,\theta^*\rangle)$ and $\dot\mu(\langle x_{s,a_s},\hat{\theta}_{\ell-2}\rangle)$. The second inequality follows from Assumption~\ref{assum:GLM_norm} together with the decomposition $|\langle x_s^* - x_{s,a_s},\theta^*\rangle| + |\langle x_{s,a_s},\theta^* - \hat{\theta}_{\ell-2}\rangle| \leq 4\varepsilon'_{s,\ell-2}(\lambda) + \varepsilon'_{s,\ell-2}(\lambda)$, as implied by~(\ref{eq:GLM_instant_bound}) and the definition of the good event $E'$. The third inequality directly follows from the definition of $\alpha_{s,\ell-2}(\lambda)$ when $\ell \geq 4$. The final bound is obtained by invoking \cref{lem:GLM_exponential_convex}.

\paragraph{Bounding $\sum_{s=\mathcal{T}_{\ell-2}+1}^{\mathcal{T}_{\ell-1}}\mathbb{E}[E_s]$.}
\begin{align*}
    \sum_{s=\mathcal{T}_{\ell-2}+1}^{\mathcal{T}_{\ell-1}}\mathbb{E}[E_s]
    &\leq \sum_{s=\mathcal{T}_{\ell-2}+1}^{\mathcal{T}_{\ell-1}}\sqrt{\mathbb{E}\left[\dot\mu(\langle x_{s}^*,\theta^*\rangle)\right]}\sqrt{\mathbb{E}\left[\left\|\sqrt{\alpha_{s,\ell-2}(\lambda)\dot\mu(\langle x_{s,a_s},\hat{\theta}_{\ell-2}\rangle)}x_{s,a_s}\right\|_{H_{s-1}^{-1}}^2\right]}\\
    &\leq \frac{\sqrt{\mathcal{T}_{\ell-1}- \mathcal{T}_{\ell-2}}}{\sqrt{\hat\kappa}}\sqrt{\sum_{s=\mathcal{T}_{\ell-2}+1}^{\mathcal{T}_{\ell-1}}\mathbb{E}\left[\left\|\sqrt{\alpha_{s,\ell-2}(\lambda)\dot\mu(\langle x_{s,a_s},\hat{\theta}_{\ell-2}\rangle)}x_{s,a_s}\right\|_{H_{s-1}^{-1}}^2\right]}\\
    &\leq \frac{\sqrt{\mathcal{T}_{\ell-1}- \mathcal{T}_{\ell-2}}}{\sqrt{\hat\kappa}}\sqrt{2\log\left(\frac{\det(H_{\mathcal{T}_{\ell-1}})}{\det(\lambda I)}\right)} \\ 
    &\leq \frac{\sqrt{2d(\mathcal{T}_{\ell-1}- \mathcal{T}_{\ell-2})\log (2T)}}{\sqrt{\hat\kappa}} \; .
\end{align*}
The first two inequalities are obtained by successive applications of the Cauchy-Schwarz inequality. The third inequality follows directly from Corollary~\ref{cor:elliptical_potential}. Finally, the last inequality is derived using the same reasoning as in (\ref{eq:elliptical_potential_id}), where the determinant is upper bounded by a trace argument, yielding a logarithmic dependence on $T$.

\paragraph{Bounding $\sum_{s=\mathcal{T}_{\ell-2}+1}^{\mathcal{T}_{\ell-1}}\mathbb{E}[F_s]$.}
\begin{align*}
    \sum_{s=\mathcal{T}_{\ell-2}+1}^{\mathcal{T}_{\ell-1}}\mathbb{E}[F_s]
    &\leq \frac{e^{3RS}}{S}\hspace{-0.4em}\sum_{s=\mathcal{T}_{\ell-2}+1}^{\mathcal{T}_{\ell-1}}\mathbb{E}\left[\varepsilon_{s,\ell-2}(\lambda)\left\|\sqrt{\dot\mu(\langle x_{s}^*,\theta^*\rangle)\alpha_{s,\ell-2}(\lambda)\dot\mu(\langle x_{s,a_s},\hat{\theta}_{\ell-2}\rangle)}x_{s,a_s}\right\|_{H_{s-1}^{-1}}\right]\\
    &= \frac{e^{3RS}\beta(\lambda)}{S}\hspace{-0.7em}\sum_{s=\mathcal{T}_{\ell-2}+1}^{\mathcal{T}_{\ell-1}}\mathbb{E}\left[\max_{y \in \mathcal{A}_s^{(\ell-3)}}\|y\|_{V_{\ell-2}^{-1}}\left\|\sqrt{\dot\mu(\langle x_{s}^*,\theta^*\rangle)\alpha_{s,\ell-2}(\lambda)\dot\mu(\langle x_{s,a_s},\hat{\theta}_{\ell-2}\rangle)}x_{s,a_s}\right\|_{H_{s-1}^{-1}}\right]\\
    &\leq \!\frac{e^{3RS}\beta(\lambda)}{S}\!\!\hspace{-0.7em}\sum_{s=\mathcal{T}_{\ell-2}+1}^{\mathcal{T}_{\ell-1}}\!\!\sqrt{\mathbb{E}\!\left[\max_{y \in \mathcal{A}_s^{(\ell-3)}}\|\sqrt{\dot\mu(\langle x_{s}^*,\theta^*\rangle)}y\|_{V_{\ell-2}^{-1}}^2\right]\mathbb{E}\!\left[\left\|\sqrt{\alpha_{s,\ell-2}(\lambda)\dot\mu(\langle x_{s,a_s},\hat{\theta}_{\ell-2}\rangle)}x_{s,a_s}\right\|_{H_{s-1}^{-1}}^2\right]}\\
    &\leq \frac{e^{3RS}\beta(\lambda)\sqrt{2d(\mathcal{T}_{\ell-1}- \mathcal{T}_{\ell-2})\log (2T)}}{S}\sqrt{\mathbb{E}\left[\smash[b]{\underbrace{\max_{y \in \mathcal{A}_s^{(\ell-3)}}\|\sqrt{\dot\mu(\langle x_{s}^*,\theta^*\rangle)}y\|_{V_{\ell-2}^{-1}}^2}_{\triangleq G_s}\vphantom{\Bigg|}}\right]} \; .
    \\
\end{align*}
The second inequality follows from an application of the Cauchy-Schwarz inequality. The final inequality uses the fact that the expectation $\mathbb{E}\Big[\max_{y \in \mathcal{A}_s^{(\ell-3)}}\|\sqrt{\dot\mu(\langle x_s^*,\theta^*\rangle)}y\|^2_{V_{\ell-2}^{-1}}\Big]$ takes the same value for all $s \in [\mathcal{T}_{\ell-2}+1,\mathcal{T}_{\ell-1}]$, as established in (\ref{eq:GLM_const_max_weight_norm}), together with Corollary~\ref{cor:elliptical_potential}, which bounds the quadratic form by a log-determinant expression.

\paragraph{Bounding $\mathbb{E}[G_s]$ for $s \in [\mathcal{T}_{\ell-2}+1,\mathcal{T}_{\ell-1}]$.}
\begin{align*}
    \mathbb{E}[G_s]
    &= \frac{1}{\mathcal{T}_{\ell-2}- \mathcal{T}_{\ell-3}}\sum_{u=\mathcal{T}_{\ell-3}+1}^{\mathcal{T}_{\ell-2}}\mathbb{E}\left[\max_{y \in \mathcal{A}_s^{(\ell-3)}}\|\sqrt{\dot\mu(\langle x_{s}^*,\theta^*\rangle)}y\|_{V_{\ell-2}^{-1}}^2\right] \\
    &\leq \frac{1}{\mathcal{T}_{\ell-2}- \mathcal{T}_{\ell-3}}\sum_{u=\mathcal{T}_{\ell-3}+1}^{\mathcal{T}_{\ell-2}}\mathbb{E}\left[\max_{y \in \mathcal{A}_s^{(\ell-3)}}\|\sqrt{\dot\mu(\langle x_{s}^*,\theta^*\rangle)}y\|_{H_{u-1}^{-1}}^2\right] \\
    &= \frac{1}{\mathcal{T}_{\ell-2}- \mathcal{T}_{\ell-3}}\sum_{u=\mathcal{T}_{\ell-3}+1}^{\mathcal{T}_{\ell-2}}\mathbb{E}\left[\mathbb{E}\left[\max_{y \in \mathcal{A}_s^{(\ell-3)}}\|\sqrt{\dot\mu(\langle x_{s}^*,\theta^*\rangle)}y\|_{H_{u-1}^{-1}}^2 \Bigg| \mathcal{F}_{u-1}\right]\right] \\
    &= \frac{1}{\mathcal{T}_{\ell-2}- \mathcal{T}_{\ell-3}}\sum_{u=\mathcal{T}_{\ell-3}+1}^{\mathcal{T}_{\ell-2}}\mathbb{E}\left[\mathbb{E}\left[\max_{y \in \mathcal{A}_u^{(\ell-3)}}\|\sqrt{\dot\mu(\langle x_{u}^*,\theta^*\rangle)}y\|_{H_{u-1}^{-1}}^2 \Bigg| \mathcal{F}_{u-1}\right]\right] \\
    &= \frac{1}{\mathcal{T}_{\ell-2}- \mathcal{T}_{\ell-3}}\sum_{u=\mathcal{T}_{\ell-3}+1}^{\mathcal{T}_{\ell-2}}\mathbb{E}\left[\max_{y \in \mathcal{A}_u^{(\ell-3)}}\|\sqrt{\dot\mu(\langle x_{u}^*,\theta^*\rangle)}y\|_{H_{u-1}^{-1}}^2 \right] \\
    &= \frac{1}{\mathcal{T}_{\ell-2}- \mathcal{T}_{\ell-3}}\sum_{u=\mathcal{T}_{\ell-3}+1}^{\mathcal{T}_{\ell-2}}\mathbb{E}\left[\smash[b]{\underbrace{\|\sqrt{\dot\mu(\langle x_{u}^*,\theta^*\rangle)}x_{u,a_u}\|_{H_{u-1}^{-1}}^2}_{\triangleq I_u} \vphantom{\Bigg|}}\right] \; . \\
\end{align*}
 The first inequality follows from the monotonicity of the matrices, since $H_{u-1} \preceq V_{\ell-2}$ for all $u$ in the summation range. The second and fourth equalities are applications of the tower property. The third equality holds because, conditional on $\mathcal{F}_{u-1}$, both $H_{u-1}$ and the arm elimination rule (determined by $\hat{\theta}_1, \dots, \hat{\theta}_{\ell-3}$) are fixed, while the distribution of the contexts remains unchanged. Finally, the last equality follows from the arm-selection strategy: the multiplicative factor $\dot\mu(\langle x_u^*,\theta^* \rangle)$ does not affect the maximization and can therefore be factored out without altering the $\arg\max$.

\paragraph{Bounding $\sum_{u=\mathcal{T}_{\ell-3}+1}^{\mathcal{T}_{\ell-2}}\mathbb{E}[I_u]$.}
\begin{align*}
    \sum_{u=\mathcal{T}_{\ell-3}+1}^{\mathcal{T}_{\ell-2}}\mathbb{E}[I_u]
    &\leq \sum_{u=\mathcal{T}_{\ell-3}+1}^{\mathcal{T}_{\ell-2}}\mathbb{E}\left[e^{R|\langle x_u^*,\theta^*\rangle - \langle x_{u,a_u},\hat{\theta}_{\ell-3}\rangle|}\left\|\sqrt{\dot\mu(\langle x_{u,a_u},\hat{\theta}_{\ell-3}\rangle)}x_{u,a_u}\right\|^2_{H_{u-1}^{-1}}\right]\\
    &\leq \sum_{u=\mathcal{T}_{\ell-3}+1}^{\mathcal{T}_{\ell-2}}\mathbb{E}\left[e^{4RS}\left\|\sqrt{\alpha_{u,\ell-3}(\lambda)\dot\mu(\langle x_{u,a_u},\hat{\theta}_{\ell-3}\rangle)}x_{u,a_u}\right\|_{H_{u-1}^{-1}}^2\right] \\
    &\leq 2e^{4RS}d\log (2T) \; .
\end{align*}
The first inequality is obtained by applying \cref{lem:GLM_concordance}, which allows us to compare $\dot\mu(\langle x_{u}^*,\theta^*\rangle)$ and $\dot\mu(\langle x_{u,a_u},\hat{\theta}_{\ell-3}\rangle)$. The second inequality uses Assumption \ref{assum:norm}, together with the fact that $\alpha_{u,\ell-3}(\lambda) \geq e^{-2RS}$ for all $\ell \geq 4$. Finally, the last inequality follows from Corollary~\ref{cor:elliptical_potential}, which bounds the quadratic form in terms of a log-determinant expression and yields the stated order.

We now combine the previous bounds step by step. First, from the result on $\mathbb{E}[G_s]$, we obtain
\begin{align*}
    \mathbb{E}[G_s]
    \leq \frac{1}{\mathcal{T}_{\ell-2} - \mathcal{T}_{\ell-3}} \sum_{u=\mathcal{T}_{\ell-3}+1}^{\mathcal{T}_{\ell-2}} \mathbb{E}[I_u] \leq \frac{2e^{4RS}d\log (2T)}{\mathcal{T}_{\ell-2} - \mathcal{T}_{\ell-3}}
     \; .
\end{align*}

Next, substituting this into the bound for $\sum_{s=\mathcal{T}_{\ell-2}+1}^{\mathcal{T}_{\ell-1}}\mathbb{E}[F_s]$, we obtain
\begin{align*}
    \sum_{s=\mathcal{T}_{\ell-2}+1}^{\mathcal{T}_{\ell-1}} \mathbb{E}[F_s]
    &\leq \frac{e^{3RS}\beta(\lambda)\sqrt{2d(\mathcal{T}_{\ell-1}- \mathcal{T}_{\ell-2})\log (2T)}}{S}\sqrt{\mathbb{E}\left[G_s\right]} \\
    &\leq \frac{2e^{5RS}\beta(\lambda)d\log (2T)\sqrt{\mathcal{T}_{\ell-1}- \mathcal{T}_{\ell-2}}}{S\sqrt{\mathcal{T}_{\ell-2} - \mathcal{T}_{\ell-3}}}
    \; .
\end{align*}

Furthermore, combining this with the result on $\sum_{s=\mathcal{T}_{\ell-2}+1}^{\mathcal{T}_{\ell-1}}\mathbb{E}[E_s]$, we can bound $\sum_{s=\mathcal{T}_{\ell-2}+1}^{\mathcal{T}_{\ell-1}}\mathbb{E}[D_s]$ as
\begin{align*}
    \sum_{s=\mathcal{T}_{\ell-2}+1}^{\mathcal{T}_{\ell-1}} \mathbb{E}[D_s]
    &\leq \sum_{s=\mathcal{T}_{\ell-2}+1}^{\mathcal{T}_{\ell-1}} \mathbb{E}[E_s] + \sum_{s=\mathcal{T}_{\ell-2}+1}^{\mathcal{T}_{\ell-1}} \mathbb{E}[F_s] \\
    &\leq \frac{\sqrt{2d(\mathcal{T}_{\ell-1}- \mathcal{T}_{\ell-2})\log (2T)}}{\sqrt{\hat\kappa}} + \frac{2e^{5RS}\beta(\lambda)d\log (2T)\sqrt{\mathcal{T}_{\ell-1}- \mathcal{T}_{\ell-2}}}{S\sqrt{\mathcal{T}_{\ell-2} - \mathcal{T}_{\ell-3}}}
    \; .
\end{align*}

Finally, substituting this bound into the expression for $\sum_{t=\mathcal{T}_{\ell-1}+1}^{\mathcal{T}_{\ell}}\mathbb{E}[B_t]$ yields
\begin{align*}
    \sum_{t=\mathcal{T}_{\ell-1}+1}^{\mathcal{T}_{\ell}} \mathbb{E}[B_t]
    &\leq \frac{\beta(\lambda)(\mathcal{T}_{\ell} - \mathcal{T}_{\ell-1})}{\mathcal{T}_{\ell-1} - \mathcal{T}_{\ell-2}}\sum_{s=\mathcal{T}_{\ell-2}+1}^{\mathcal{T}_{\ell-1}} \mathbb{E}[D_s] \\
    &\leq \frac{\beta(\lambda)(\mathcal{T}_{\ell} - \mathcal{T}_{\ell-1})}{\sqrt{\mathcal{T}_{\ell-1} - \mathcal{T}_{\ell-2}}}\left(\frac{\sqrt{2d\log (2T)}}{\sqrt{\hat\kappa}} + \frac{2e^{5RS}\beta(\lambda)d\log (2T)}{S\sqrt{\mathcal{T}_{\ell-2} - \mathcal{T}_{\ell-3}}}\right)
    \; .
\end{align*}

\paragraph{Bounding $\sum_{t=\mathcal{T}_{\ell-1}+1}^{\mathcal{T}_{\ell}}\mathbb{E}[C_t]$.}
\begin{align*}
    \sum_{t=\mathcal{T}_{\ell-1}+1}^{\mathcal{T}_{\ell}}\mathbb{E}[C_t]
    &\leq \frac{e^{4RS}}{S}\sum_{t=\mathcal{T}_{\ell-1}+1}^{\mathcal{T}_{\ell}}\mathbb{E}\left[\dot\mu(\langle x_t^*,\theta^*\rangle)\, \varepsilon_{t,\ell-1}'(\lambda)^2\right] \\
    &= \frac{e^{4RS}\beta(\lambda)^2}{S}\sum_{t=\mathcal{T}_{\ell-1}+1}^{\mathcal{T}_{\ell}}\mathbb{E}\left[\max_{z\in\mathcal{A}_t^{(\ell-2)}}\|\sqrt{\dot\mu(\langle x_t^*,\theta^*\rangle)}z\|_{V_{\ell-1}^{-1}}^2\right] \\
    &= \frac{e^{4RS}\beta(\lambda)^2(\mathcal{T}_{\ell} - \mathcal{T}_{\ell-1})}{S} \, \mathbb{E}\left[\max_{z\in\mathcal{A}_t^{(\ell-2)}}\|\sqrt{\dot\mu(\langle x_t^*,\theta^*\rangle)}z\|_{V_{\ell-1}^{-1}}^2\right] \\
    &= \frac{e^{4RS}\beta(\lambda)^2(\mathcal{T}_{\ell} - \mathcal{T}_{\ell-1})}{S(\mathcal{T}_{\ell-1} - \mathcal{T}_{\ell-2})}\sum_{s=\mathcal{T}_{\ell-2}+1}^{\mathcal{T}_{\ell-1}}\mathbb{E}\left[\max_{z\in\mathcal{A}_t^{(\ell-2)}}\|\sqrt{\dot\mu(\langle x_t^*,\theta^*\rangle)}z\|_{V_{\ell-1}^{-1}}^2\right] \\
    &\leq \frac{e^{4RS}\beta(\lambda)^2(\mathcal{T}_{\ell} - \mathcal{T}_{\ell-1})}{S(\mathcal{T}_{\ell-1} - \mathcal{T}_{\ell-2})}\sum_{s=\mathcal{T}_{\ell-2}+1}^{\mathcal{T}_{\ell-1}}\mathbb{E}\left[\max_{z\in\mathcal{A}_t^{(\ell-2)}}\|\sqrt{\dot\mu(\langle x_t^*,\theta^*\rangle)}z\|_{H_{s-1}^{-1}}^2\right] \\
    &= \frac{e^{4RS}\beta(\lambda)^2(\mathcal{T}_{\ell} - \mathcal{T}_{\ell-1})}{S(\mathcal{T}_{\ell-1} - \mathcal{T}_{\ell-2})}\sum_{s=\mathcal{T}_{\ell-2}+1}^{\mathcal{T}_{\ell-1}}\mathbb{E}\left[\mathbb{E}\left[\max_{z\in\mathcal{A}_t^{(\ell-2)}}\|\sqrt{\dot\mu(\langle x_t^*,\theta^*\rangle)}z\|_{H_{s-1}^{-1}}^2 \Bigg| \mathcal{F}_{s-1}\right]\right] \\
    &= \frac{e^{4RS}\beta(\lambda)^2(\mathcal{T}_{\ell} - \mathcal{T}_{\ell-1})}{S(\mathcal{T}_{\ell-1} - \mathcal{T}_{\ell-2})}\sum_{s=\mathcal{T}_{\ell-2}+1}^{\mathcal{T}_{\ell-1}}\mathbb{E}\left[\mathbb{E}\left[\max_{z\in\mathcal{A}_s^{(\ell-2)}}\|\sqrt{\dot\mu(\langle x_s^*,\theta^*\rangle)}z\|_{H_{s-1}^{-1}}^2 \Bigg| \mathcal{F}_{s-1}\right]\right] \\
    &= \frac{e^{4RS}\beta(\lambda)^2(\mathcal{T}_{\ell} - \mathcal{T}_{\ell-1})}{S(\mathcal{T}_{\ell-1} - \mathcal{T}_{\ell-2})}\sum_{s=\mathcal{T}_{\ell-2}+1}^{\mathcal{T}_{\ell-1}}\mathbb{E}\left[\max_{z\in\mathcal{A}_s^{(\ell-2)}}\|\sqrt{\dot\mu(\langle x_s^*,\theta^*\rangle)}z\|_{H_{s-1}^{-1}}^2 \right] \\
    &= \frac{e^{4RS}\beta(\lambda)^2(\mathcal{T}_{\ell} - \mathcal{T}_{\ell-1})}{S(\mathcal{T}_{\ell-1} - \mathcal{T}_{\ell-2})}\sum_{s=\mathcal{T}_{\ell-2}+1}^{\mathcal{T}_{\ell-1}}\mathbb{E}\left[\|\sqrt{\dot\mu(\langle x_s^*,\theta^*\rangle)}x_{s,a_s}\|_{H_{s-1}^{-1}}^2 \right] \\
    &= \frac{e^{4RS}\beta(\lambda)^2(\mathcal{T}_{\ell} - \mathcal{T}_{\ell-1})}{S(\mathcal{T}_{\ell-1} - \mathcal{T}_{\ell-2})}\sum_{s=\mathcal{T}_{\ell-2}+1}^{\mathcal{T}_{\ell-1}}\mathbb{E}\left[I_s \right] \\
    &\leq \frac{2e^{8RS}\beta(\lambda)^2(\mathcal{T}_{\ell} - \mathcal{T}_{\ell-1})d\log (2T)}{S(\mathcal{T}_{\ell-1} - \mathcal{T}_{\ell-2})} \; .
\end{align*}
The second equality follows from the same reasoning as in (\ref{eq:maximal_quadratic}), since both $V_{\ell-1}$ and the arm elimination rule—determined by $\hat{\theta}_1, \dots, \hat{\theta}_{\ell-2}$—are measurable with respect to $\mathcal{F}_{\mathcal{T}_{\ell-1}}$ and can therefore be treated as fixed conditional on this filtration; given that the contexts are drawn independently and identically, their values coincide. The second inequality uses the monotonicity of the matrices, as $H_{s-1} \preceq V_{\ell-1}$ for all $s$ in the interval. The fourth and sixth equalities apply the tower property. The fifth equality holds because, conditional on $\mathcal{F}_{s-1}$, both $H_{s-1}$ and the arm elimination rule (determined by $\hat{\theta}_1, \dots, \hat{\theta}_{\ell-2}$) are fixed, while the distribution of the contexts remains unchanged. The seventh equality is justified by the arm-selection strategy, as the multiplicative factor $\dot\mu(\langle x_s^*,\theta^* \rangle)$ does not affect the maximization and hence does not alter the $\arg\max$. The final inequality follows from the previously established bound on $\sum_{s = \mathcal{T}_{\ell-2}+1}^{\mathcal{T}_{\ell-1}} \mathbb{E}[I_s]$.

Combining the previously derived bounds for $\sum_{t=\mathcal{T}_{\ell-1}+1}^{\mathcal{T}_\ell}\mathbb{E}[B_t]$ and $\sum_{t=\mathcal{T}_{\ell-1}+1}^{\mathcal{T}_\ell}\mathbb{E}[C_t]$, we obtain
\begin{align*}
    \sum_{t=\mathcal{T}_{\ell-1}+1}^{\mathcal{T}_\ell}\mathbb{E}[A_t] 
    &\leq 4\sum_{t=\mathcal{T}_{\ell-1}+1}^{\mathcal{T}_\ell}\mathbb{E}[B_t] + 4\sum_{t=\mathcal{T}_{\ell-1}+1}^{\mathcal{T}_\ell}\mathbb{E}[C_t] \\
    &\leq \frac{4\beta(\lambda)(\mathcal{T}_{\ell} - \mathcal{T}_{\ell-1})}{\sqrt{\mathcal{T}_{\ell-1} - \mathcal{T}_{\ell-2}}}\left(\frac{\sqrt{2d\log (2T)}}{\sqrt{\hat\kappa}} + \frac{2e^{5RS}\beta(\lambda)d\log (2T)}{S\sqrt{\mathcal{T}_{\ell-2} - \mathcal{T}_{\ell-3}}}\right) + \frac{8e^{8RS}\beta(\lambda)^2(\mathcal{T}_{\ell} - \mathcal{T}_{\ell-1})d\log (2T)}{S(\mathcal{T}_{\ell-1} - \mathcal{T}_{\ell-2})} \; .
\end{align*}
For the case $\ell=4$, this simplifies to
\begin{align*}
    \sum_{t=\mathcal{T}_{\ell-1}+1}^{\mathcal{T}_\ell}\mathbb{E}[A_t] 
    &= \Ocal\left(\frac{RS\sqrt{d+\log T} \cdot \frac{T^{\frac{2}{3}}}{\log\log T}}{\sqrt{\frac{\sqrt[3]{T}}{\log\log T}}} \left(\frac{\sqrt{d\log T}}{\sqrt{\hat\kappa}} + \frac{Re^{5RS}d\log T\sqrt{d+\log T}}{\sqrt\frac{\sqrt[3]{T}}{\log\log T}}\right)\right)\\ 
    &\hspace{1em} +\Ocal\left( \frac{R^2Se^{8RS} \cdot \frac{dT^{\frac{2}{3}}(d+\log T)\log T}{\log\log T}}{\frac{\sqrt[3]{T}}{\log\log T}}\right) \\
    &= \Ocal\left(\frac{RS\sqrt{d(d+\log T)T\log T}}{\sqrt{\hat\kappa \log\log T}} + R^2Se^{8RS}d(d+\log T)T^{\frac{1}{3}}\log T\right) \; .
\end{align*}
For all subsequent intervals with $\ell \geq 5$, we similarly obtain
\begin{align*}
    \sum_{t=\mathcal{T}_{\ell-1}+1}^{\mathcal{T}_\ell}\mathbb{E}[A_t] 
    &= \Ocal\left(\frac{RS\sqrt{d+\log T} \cdot \frac{T^{1-\frac{1}{3\cdot2^{\ell-4}}}}{\log\log T}}{\sqrt{\frac{T^{1-\frac{1}{3\cdot2^{\ell-5}}}}{\log\log T}}} \left(\frac{\sqrt{d\log T}}{\sqrt{\hat\kappa}} + \frac{Re^{5RS}d\log T\sqrt{d+\log T}}{\sqrt\frac{T^{1-\frac{1}{3\cdot2^{\ell-6}}}}{\log\log T}}\right)\right) \\
    &\hspace{1em} +\Ocal\left(\frac{R^2Se^{8RS} \cdot \frac{dT^{1-\frac{1}{3\cdot2^{\ell-4}}}(d+\log T)\log T}{\log\log T}}{\frac{T^{1-\frac{1}{3\cdot2^{\ell-5}}}}{\log\log T}}\right) \\
    &= \Ocal\left(\frac{RS\sqrt{d(d+\log T)T\log T}}{\sqrt{\hat\kappa \log\log T}} + R^2Se^{8RS}d(d+\log T)T^{\frac{1}{3\cdot2^{\ell-5}}}\log T\right) \\
    &= \Ocal\left(\frac{RS\sqrt{d(d+\log T)T\log T}}{\sqrt{\hat\kappa \log\log T}} + R^2Se^{8RS}d(d+\log T)T^{\frac{1}{3}}\log T\right) \; .
\end{align*} 
Therefore, the total worst-case regret is bounded as
\begin{align*}
    \mathcal{R}(T)
    &= \sum_{\ell=1}^{3}\text{Regret}_{\ell} + 
    \sum_{\ell=4}^{B}\text{Regret}_{\ell} \\
    &= \Ocal\!\left(\frac{R T^{\frac{1}{3}}}{\log\log T}\right) + \sum_{\ell=4}^{B}\sum_{t=\mathcal{T}_{\ell-1}+1}^{\mathcal{T}_\ell}\mathbb{E}[\mu(\langle x_t^*,\theta^*\rangle) - \mu(\langle x_{t,a_t},\theta^*\rangle)] \\
    &\leq \Ocal\!\left(\frac{RT^{\frac{1}{3}}}{\log\log T}\right) +\sum_{\ell=4}^{B}\sum_{t=\mathcal{T}_{\ell-1}+1}^{\mathcal{T}_\ell}\mathbb{E}[A_t] \\
    &= \Ocal \left(RS\sqrt{\frac{d(d+\log T)T\log T\log\log T}{\hat\kappa}}+ \left(R^2Se^{8RS}d(d+\log T)\log T\log\log T + \frac{R}{\log\log T} \right) T^{\frac{1}{3}}\right) \\
    &=\tilde\Ocal\left(\frac{dRS\sqrt{T}}{\sqrt{\hat\kappa}} + (R^2Se^{8RS}d^2+R)T^{\frac{1}{3}}\right) \; .
\end{align*}
\end{proof}
\section{Additional Experiments} \label{add_experiments}

\subsection{Linear Contextual Bandits: Experimental Results with Normal Contexts} \label{exp:norm}
We evaluate the performance of \texttt{BLCE-G} and \texttt{BLCE} by measuring the cumulative regret over $T=10{,}000$ rounds. At each iteration, $K$ arms are independently sampled from a \(d\)-dimensional normal distribution, and the parameter vector \(\theta^*\) is drawn from a $d$-dimensional normal distribution. Each experiment is repeated 10 times. We consider  $(K,d) \in \{ (1000,5), (5000,10), (50,20),(100,30)\}$, where the first two pairs represent the large-$K$ regime and the latter two correspond to the small-$K$ regime.

For comparison, we benchmark against state-of-the-art algorithms: \textit{Rarely Switching OFUL} (\texttt{RS-OFUL}; \citealt{abbasi2011improved}), \texttt{BatchLinUCB-DG} (\citealt{ruan2021linear}), \textit{Efficient Batched Algorithm for linear contextual Bandits} (\texttt{SoftBatch}; \citealt{hanna2023contexts}), and \texttt{BatchLearning} (\citealt{zhang2025almost}). The within-interval allocation rate for \texttt{BLCE-G} and \texttt{BLCE} is set to $c=0.5$. For \texttt{RS-OFUL}, the switching parameter is $C=3$, and for \texttt{SoftBatch}, the discretization parameter is $q=1/(8\sqrt{d})$. Algorithms of \citet{hanna2023contexts} incur substantial computational overhead, as reflected in their time complexity reported in \cref{table:comparison}; we therefore omit their regret plots. Importantly, \texttt{BLCE-G} and \texttt{BLCE} are implemented with the exact theoretical hyperparameters specified in our main results, without additional tuning.

We report three types of figures. First, the average cumulative regret (solid line) together with its standard deviation (shaded region) over 10 runs.
Second, zoomed-in views of the regret curves to highlight the differences between \texttt{BLCE-G} and \texttt{BLCE}.
Third, the average update complexity across $10$ runs, illustrating the frequency of parameter updates (for consistency with prior work, we label the horizontal axis using the term ``batch'').

As illustrated in Figure~\ref{normal figure}, both \texttt{BLCE-G} and \texttt{BLCE} consistently outperform all baselines in both the large-$K$ and small-$K$ regimes, achieving lowest regret and exhibiting greater stability.
These results confirm that \texttt{BLCE-G} and \texttt{BLCE} not only attain the tightest theoretical guarantees but also deliver strong empirical performance, thereby fulfilling their design objectives of near-optimal regret and minimal interval complexity. Furthermore, runtime comparisons in \cref{table:runtime_results_2} show that \texttt{BLCE-G} and \texttt{BLCE} incur substantially lower computation cost than other theoretically optimal algorithms. In particular, \texttt{BLCE}, which eliminates reliance on G-optimal design, attains the fastest runtime—comparable even to suboptimal algorithms.

Overall, these experiments demonstrate a distinctive advantage of our approach: \texttt{BLCE-G} and \texttt{BLCE} combine minimax-optimal regret guarantees with practical efficiency. This dual benefit of theoretical optimality and empirical superiority sets them apart from all prior methods for linear contextual bandits.
\begin{figure*}[ht]
\centering
\begin{subfigure}{\textwidth}
\includegraphics[width=\textwidth]{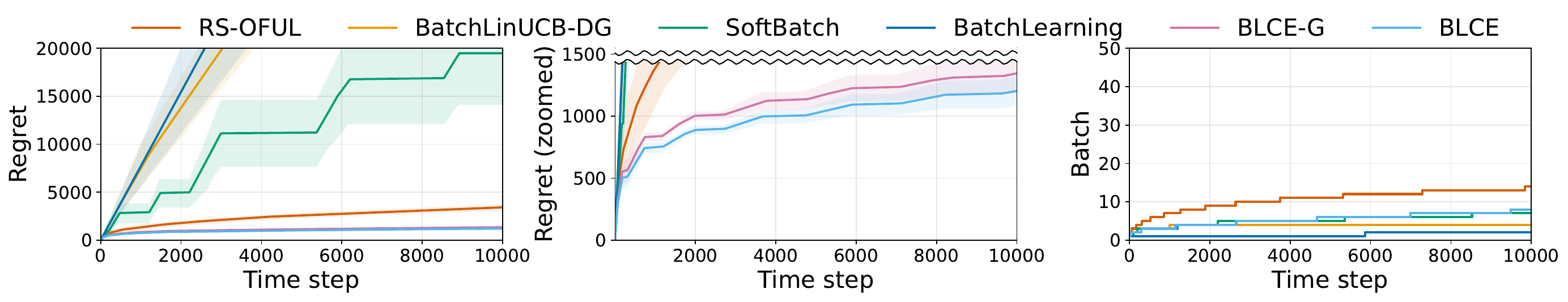}
\centering
{\small \raisebox{0.5cm}{(a) $K=1000, d=5$} \par}
\end{subfigure}\vspace{-1em}
\begin{subfigure}{\textwidth}
\includegraphics[width=\textwidth]{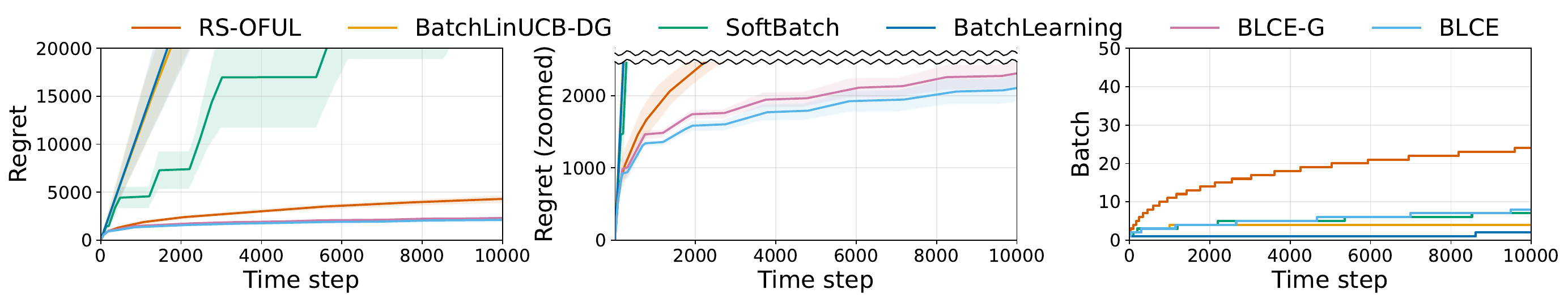}
\centering
{\small \raisebox{0.5cm}{(b) $K=5000, d=10$} \par}
\end{subfigure}\vspace{-1em}
\begin{subfigure}{\textwidth}
\includegraphics[width=\textwidth]{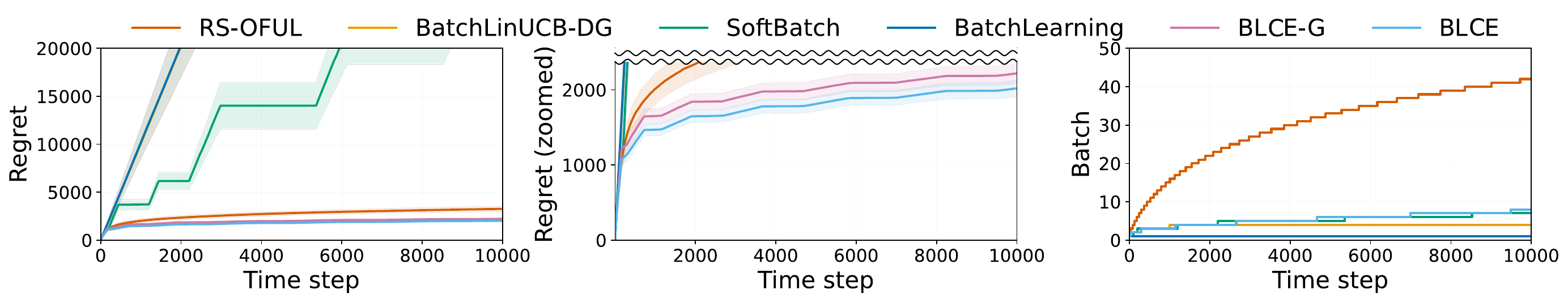}
\centering
{\small \raisebox{0.5cm}{(c) $K=50, d=20$} \par}
\end{subfigure}\vspace{-1em}
\begin{subfigure}{\textwidth}
\includegraphics[width=\textwidth]{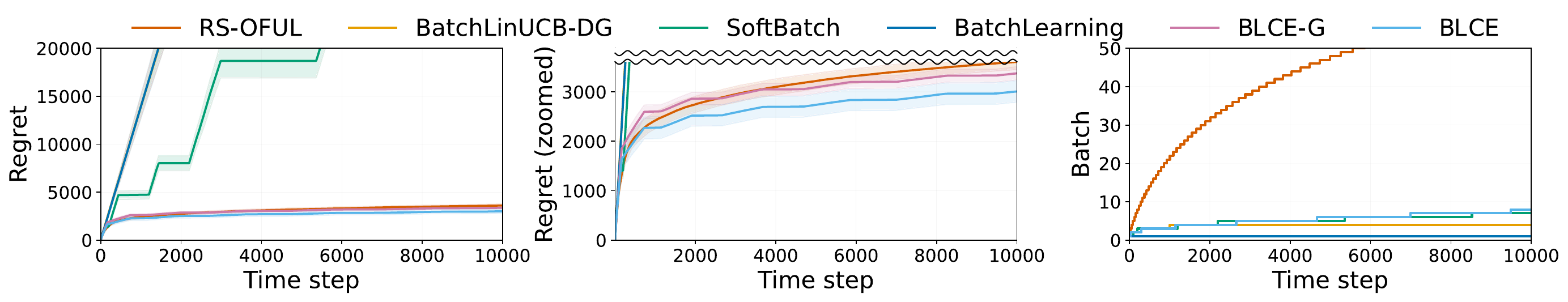}
\centering
{\small \raisebox{0.5cm}{(d) $K=100, d=30$} \par}
\end{subfigure}\vspace{-2em}
\setcounter{figure}{1}
\captionof{figure}{Regret, zoomed-in regret, and interval (batch) complexity over time for different values of $K$ and $d$.}
\label{normal figure}
\end{figure*}

\begin{table}[!ht]  
\caption{Average runtime (seconds) over 10 runs.}
\label{table:runtime_results_2}
\centering
\begin{footnotesize}
\begin{tabular}{lrrrrrrr}
\toprule
& \multicolumn{2}{c}{\textbf{Suboptimal algorithms}} 
& \multicolumn{5}{c}{\textbf{Optimal algorithms}} \\
\cmidrule(lr){2-3}\cmidrule(lr){4-8}
\textbf{$(K,d)$} 
& \texttt{RS-OFUL} & \texttt{SoftBatch} 
& \texttt{BatchLinUCB-DG} & \citet{hanna2023contexts} & \texttt{BatchLearning} & \texttt{BLCE-G} & \texttt{BLCE} \\
\midrule
$(1000,5)$   & $1.37$ & $1.15$ & $148.19$ & Exponential & $143.07$ & $3.58$ & $2.17$ \\
$(5000,10)$  & $10.46$ & $12.62$ & $555.11$ & Exponential & $590.69$ & $9.16$ & $6.19$ \\
$(50,20)$    & $0.54$ & $1.83$ & $981.19$ & Exponential & $46.24$ & $1.39$ & $1.05$ \\
$(100,30)$   & $0.95$ & $3.51$ & $2773.66$ & Exponential & $75.07$ & $1.97$ & $1.41$ \\
\bottomrule
\end{tabular}
\end{footnotesize}
\end{table}
\subsection{Experimental Results for Generalized Linear Contextual Bandits} \label{exp:GLM}
We evaluate the performance of \texttt{BGLE} by measuring the cumulative regret over a horizon of $T=10{,}000$ rounds. At each iteration, $K$ arms are independently sampled from either a \(d\)-dimensional uniform or normal distribution, and the parameter vector \(\theta^*\) is drawn from a $d$-dimensional normal distribution. Each experiment is repeated $20$ times for the parameter pairs  $(K,d) \in \{ (20,2), (50,3)\}$, considering both uniform and normal contexts.

For comparison, we benchmark against a state-of-the-art algorithm: \texttt{B-GLinCB} (\citealt{sawarni2024generalized}). The within-interval allocation rate for \texttt{BGLE} is set to $c=0.5$, and we conduct experiments on logistic bandits with $R=S=1$. Importantly, \texttt{BGLE} is implemented with the exact theoretical hyperparameters from our main results, without any tuning.

We report two types of figures. First, the average cumulative regret (solid line) together with its standard deviation (shaded region) over $20$ runs. Second, the average interval complexity across $20$ runs, showing how frequently each algorithm updates its policy.

As shown in Figure~\ref{GLM figure}, \texttt{BGLE} consistently outperforms the baseline, achieving lowest regret and demonstrating stable performance.
Runtime comparisons in \cref{table:runtime_results_3} further show that \texttt{BGLE} incurs lower computation cost than \texttt{B-GLinCB}.

An additional limitation of \texttt{B-GLinCB} is that it often uses only one interval (see \cref{GLM figure}), even for small values of $K$ and $d$. This behavior arises because its first interval length is determined by
\[
\Big(900R^2S\sqrt{\kappa}e^{3RS}d^3\log T\sqrt{T}\Big)^{\tfrac{2}{3}} \; ,
\]
which easily exceeds $T=10{,}000$ when $d$ is small, thereby preventing meaningful batching.

Overall, these experiments highlight a clear advantage of our approach:
\texttt{BGLE} combines near-optimal regret guarantees with practical
efficiency. This combination of theoretical guarantees and strong empirical
performance distinguishes our method from prior approaches to generalized
linear contextual bandits.

\begin{figure*}[ht]
\centering
\begin{subfigure}{\textwidth}
\includegraphics[width=\textwidth]{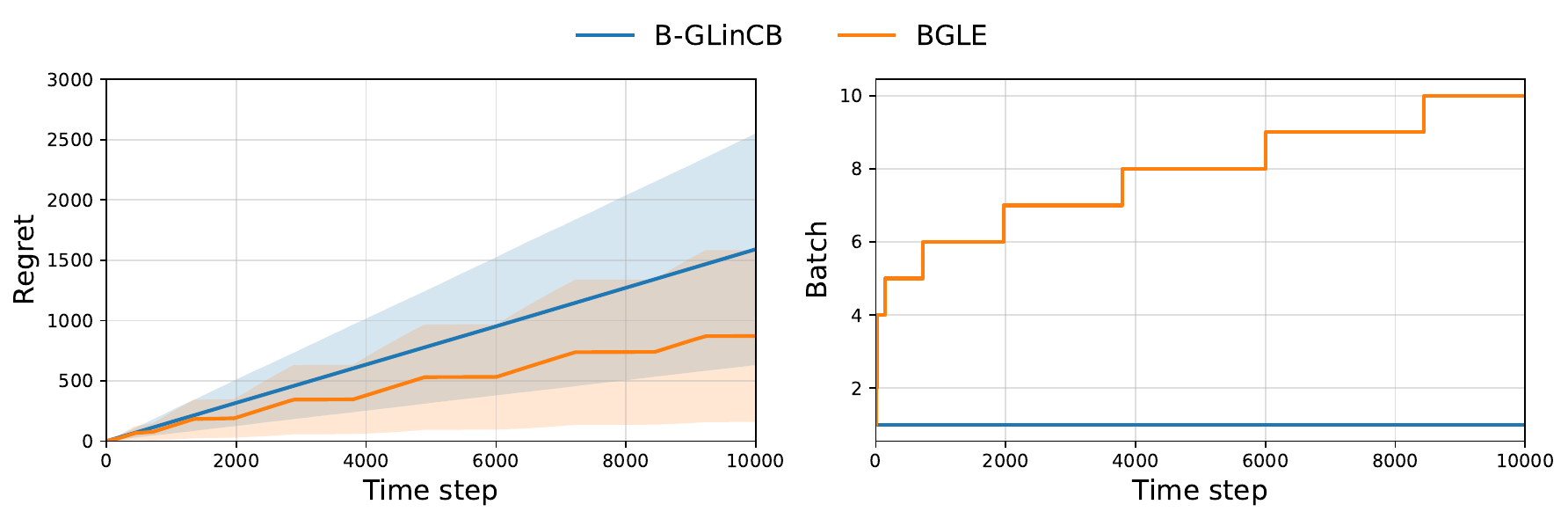}
\centering
{\raisebox{0.5cm}{(a) $K=20, d=2$, uniform} \par}
\end{subfigure}\vspace{-1em}
\begin{subfigure}{\textwidth}
\includegraphics[width=\textwidth]{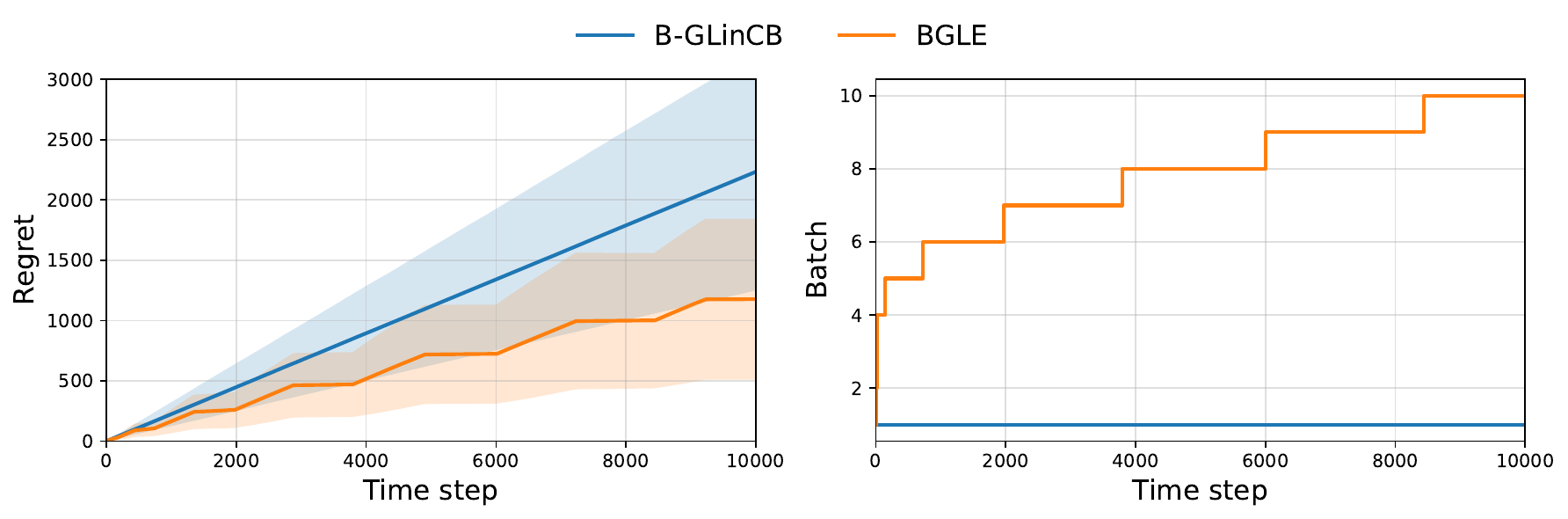}
\centering
{\raisebox{0.5cm}{(b) $K=50, d=3$, uniform} \par}
\end{subfigure}\vspace{-1em}
\begin{subfigure}{\textwidth}
\includegraphics[width=\textwidth]{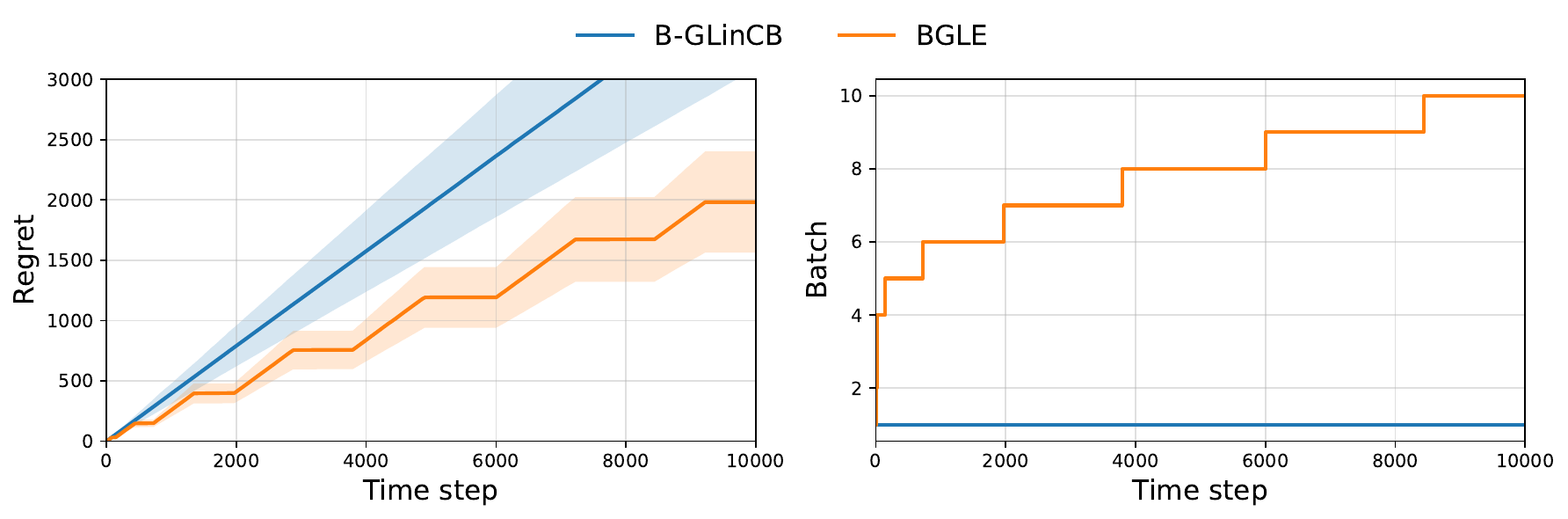}
\centering
{\raisebox{0.5cm}{(c) $K=20, d=2$, normal} \par}
\end{subfigure}\vspace{-1em}
\begin{subfigure}{\textwidth}
\includegraphics[width=\textwidth]{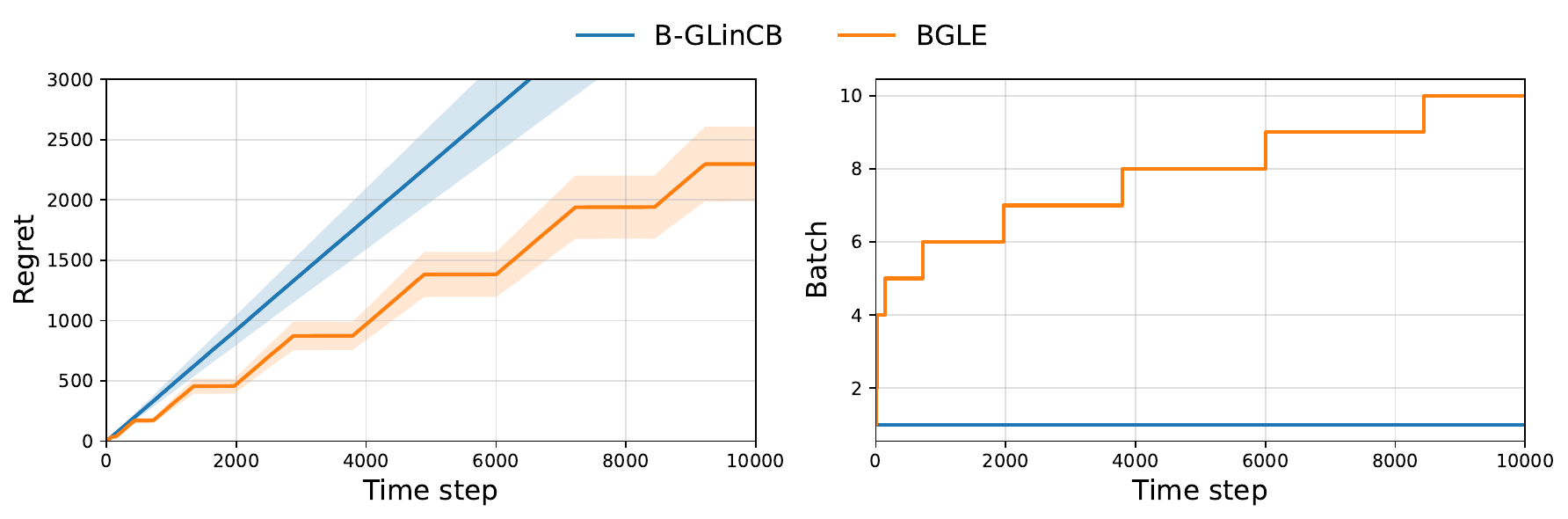}
\centering
{\raisebox{0.5cm}{(d) $K=50, d=3$, normal} \par}
\end{subfigure}\vspace{-2em}
\setcounter{figure}{2}
\captionof{figure}{Regret and interval (batch) complexity over time for different values of $K$ and $d$.}
\label{GLM figure}
\end{figure*}

\begin{table}[!ht]  
\caption{Average runtime (seconds) over 20 runs.}
\label{table:runtime_results_3}
\centering
\begin{normalsize}
\begin{tabular}{lrrrr}
\toprule
& \multicolumn{2}{c}{\textbf{Uniform distribution}} 
& \multicolumn{2}{c}{\textbf{Normal distribution}} \\
\cmidrule(lr){2-3}\cmidrule(lr){4-5}
\textbf{$(K,d)$} 
& \texttt{B-GLinCB} & \texttt{BGLE}  & \texttt{B-GLinCB} & \texttt{BGLE} \\
\midrule
$(20,2)$   & $24.94$ & $3.62$ & $25.82$ & $3.88$  \\
$(50,3)$  & $27.45$ & $3.91$ & $28.78$ & $4.06$ \\
\bottomrule
\end{tabular}
\end{normalsize}
\end{table}
\subsection{Performance Comparison between \texttt{BLCE-G} and \texttt{BLCE}} \label{exp:BLCE-G vs BLCE}
We evaluate the performance of \texttt{BLCE-G} and \texttt{BLCE} over $T = 10{,}000$ rounds with $K = 1000$ arms and feature dimension $d = 5$. At the beginning of each run, the parameter vector $\theta^*$ is sampled from a $d$-dimensional normal distribution. We then conduct \emph{four independent experiments}, each corresponding to a different fixed context distribution. In each experiment, and for every round, the $d$-dimensional feature vectors are drawn i.i.d.\ from one of the following distributions: (i) a Student-$t$ distribution with $1.5$ degrees of freedom, (ii) a Beta$(0.5,0.5)$ distribution supported on $(0,1)$, (iii) a Laplace distribution with location $0$ and scale $1/\sqrt{2}$, or (iv) an Exponential distribution with unit scale supported on $[0,\infty)$. Thus, the four experiments differ only in the underlying feature distribution, while $T$, $K$, and $d$ remain fixed across all settings.

The within-interval allocation rate for both \texttt{BLCE-G} and \texttt{BLCE} is fixed at $c = 0.5$. As in \cref{exp:norm}, we report cumulative regret, zoomed-in regret, and interval (batch) complexity, averaged over $10$ independent runs for each distributional setting.

As shown in Figure~\ref{various distribution figure}, \texttt{BLCE-G} and \texttt{BLCE} exhibit remarkably similar regret performance, with neither algorithm consistently dominating the other. Across most distributional settings, \texttt{BLCE} achieves slightly lower regret, whereas in the heavy-tailed Student-\(t\) distribution case this difference nearly vanishes and \texttt{BLCE-G} matches or occasionally exceeds the performance of \texttt{BLCE}. While theory guarantees that \texttt{BLCE-G} has a strictly smaller minimax regret bound, these guarantees characterize worst-case regimes and do not necessarily arise in benign stochastic environments. Hence, the empirical results do not contradict the theoretical results.

Finally, the runtime comparison in \cref{table:runtime_results_4} shows that \texttt{BLCE} incurs lower computational cost than \texttt{BLCE-G}. By removing the G-optimal design component, \texttt{BLCE} achieves a simpler computational structure and correspondingly faster execution in practice.

\begin{figure*}[ht]
\centering
\begin{subfigure}{\textwidth}
\includegraphics[width=\textwidth]{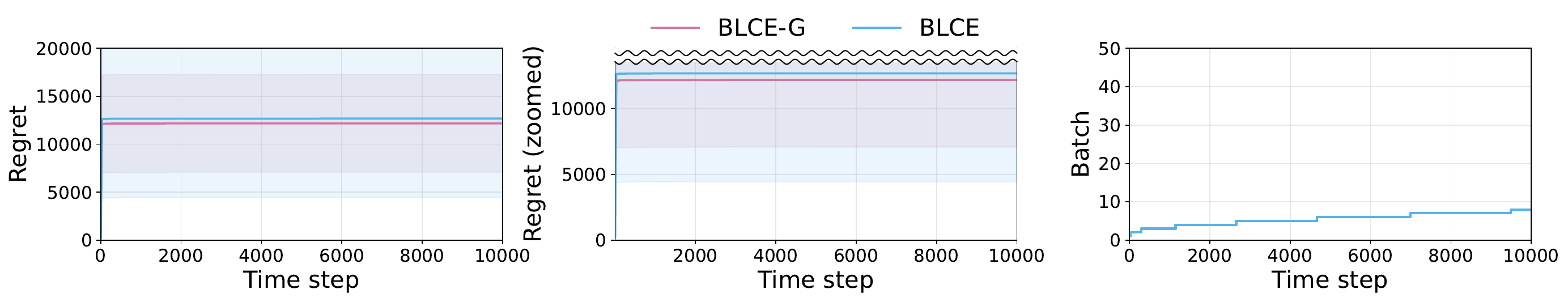}
\centering
{\small \raisebox{0.5cm}{(a) Student-$t$ distribution} \par}
\end{subfigure}\vspace{-1em}
\begin{subfigure}{\textwidth}
\includegraphics[width=\textwidth]{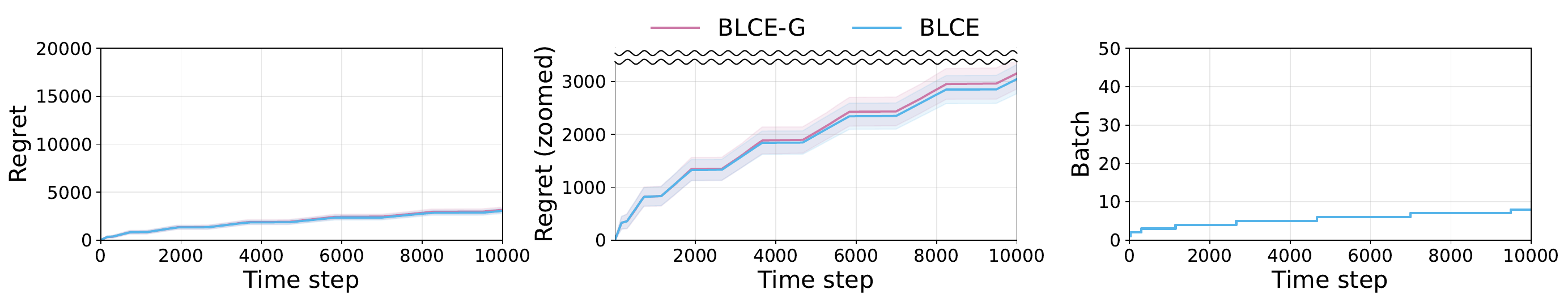}
\centering
{\small \raisebox{0.5cm}{(b) Beta distribution} \par}
\end{subfigure}\vspace{-1em}
\begin{subfigure}{\textwidth}
\includegraphics[width=\textwidth]{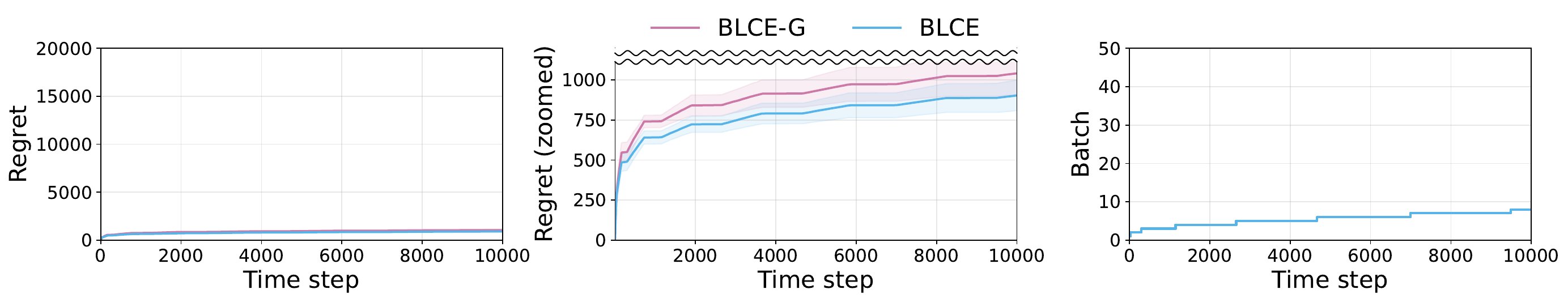}
\centering
{\small \raisebox{0.5cm}{(c) Laplace distribution} \par}
\end{subfigure}\vspace{-1em}
\begin{subfigure}{\textwidth}
\includegraphics[width=\textwidth]{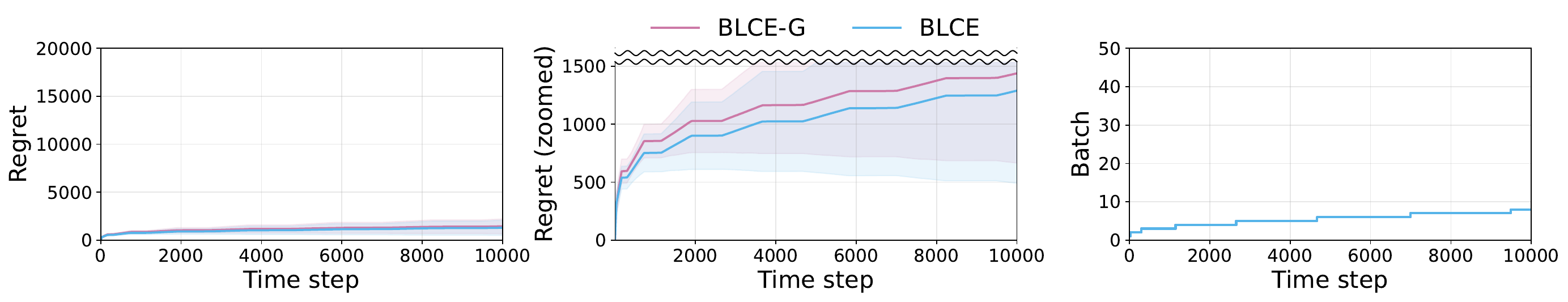}
\centering
{\small \raisebox{0.5cm}{(d) Exponential distribution} \par}
\end{subfigure}\vspace{-2em}
\setcounter{figure}{3}
\captionof{figure}{Regret, zoomed-in regret, and interval (batch) complexity over time for various context distributions.}
\label{various distribution figure}
\end{figure*}

\begin{table}[!ht]  
\caption{Average runtime (seconds) over 10 runs.}
\label{table:runtime_results_4}
\centering
\begin{normalsize}
\begin{tabular}{lrrrr}
\toprule
\textbf{Method} 
& \textbf{Student-$t$} & \textbf{Beta}  & \textbf{Laplace} & \textbf{Exponential} \\
\midrule
\texttt{BLCE-G}   & $9.33$ & $47.57$ & $9.91$ & $48.96$  \\
\texttt{BLCE}  & $3.56$ & $8.16$ & $4.11$ & $7.96$ \\
\bottomrule
\end{tabular}
\end{normalsize}
\end{table}
\subsection{How the Performance of \texttt{BLCE-G}, \texttt{BLCE}, and \texttt{BGLE} Changes with the Allocation Rate \(c\)} \label{exp: varing c}
We evaluate the performance of \texttt{BLCE-G}, \texttt{BLCE}, and \texttt{BGLE} over $T = 10{,}000$ rounds. For \texttt{BLCE-G} and \texttt{BLCE}, we set $K = 1000$ arms with feature dimension $d = 5$, where the arm contexts are independently sampled from either a \(d\)-dimensional uniform or normal distribution. For \texttt{BGLE}, we set $K = 50$ arms with feature dimension $d = 3$, again sampled independently from either a \(d\)-dimensional uniform or normal distribution. At the beginning of each run, the parameter vector $\theta^*$ is sampled from a $d$-dimensional normal distribution. For each algorithm, we vary the within-interval allocation rate $c \in \{0.1, 0.3, 0.5,0.7,0.9\}$ and compare the resulting performance.

As in \cref{exp:norm}, we report cumulative regret, zoomed-in regret, and interval complexity averaged over $10$ independent runs for \texttt{BLCE-G} and \texttt{BLCE}. For \texttt{BGLE}, following \cref{exp:GLM}, we report cumulative regret and interval complexity averaged over $20$ independent runs.

As shown in Figure~\ref{various c figure} and Figure~\ref{various c GLM figure}, decreasing the value of $c$ consistently improves the empirical performance of \texttt{BLCE-G}, \texttt{BLCE}, and \texttt{BGLE}. This indicates that more greedy within-interval arm selection tends to yield better regret. Our theoretical analysis only requires that $c$ be chosen in the interval $(0,1)$ and does not predict how performance varies with $c$. Understanding why smaller $c$---that is, more greedy selection---leads to improved performance in practice remains an interesting direction for future work. The empirical findings suggest that, under stochastic contextual variation, prioritizing greedy exploitation may already provide strong performance even in the rare parameter update setting. Developing a rigorous mathematical explanation for this behavior would be an important step toward a deeper theoretical understanding of linear contextual bandits.

\begin{figure*}[ht]
\centering
\begin{subfigure}{\textwidth}
\includegraphics[width=\textwidth]{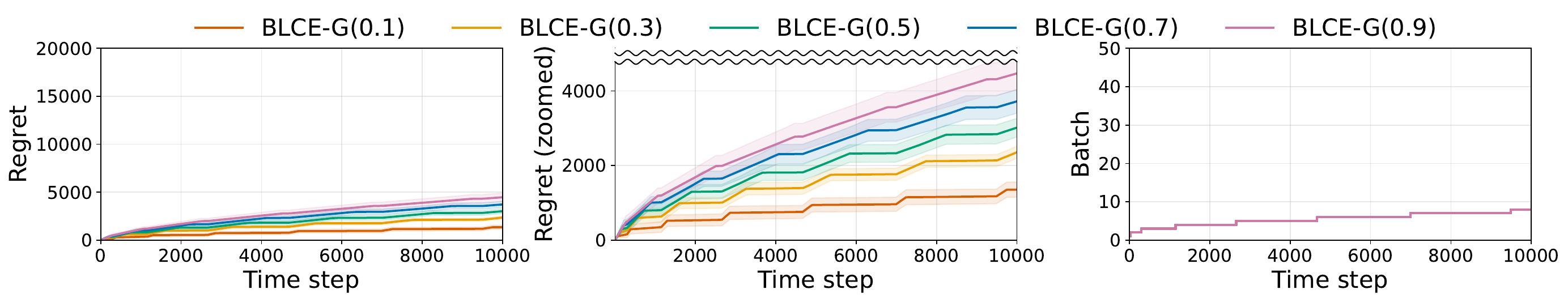}
\centering
{\small \raisebox{0.5cm}{(a) \texttt{BLCE-G}, uniform} \par}
\end{subfigure}\vspace{-1em}
\begin{subfigure}{\textwidth}
\includegraphics[width=\textwidth]{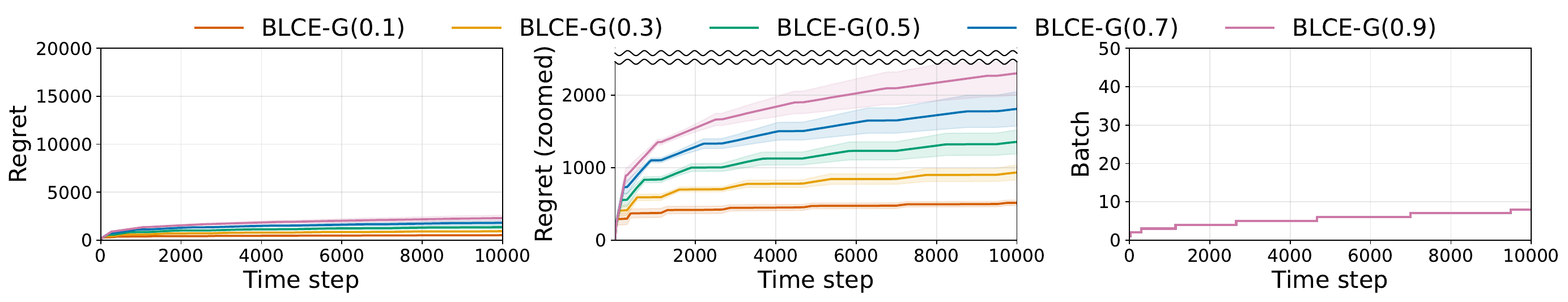}
\centering
{\small \raisebox{0.5cm}{(b) \texttt{BLCE-G}, normal} \par}
\end{subfigure}\vspace{-1em}
\begin{subfigure}{\textwidth}
\includegraphics[width=\textwidth]{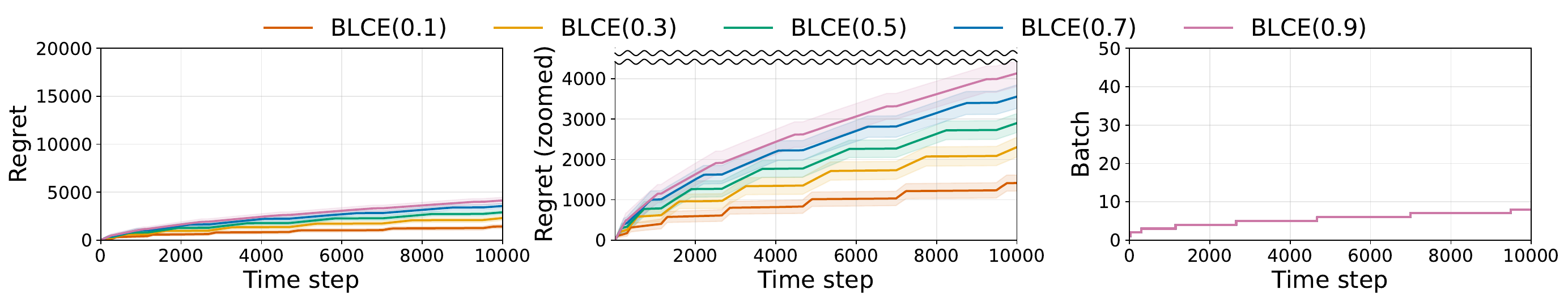}
\centering
{\small \raisebox{0.5cm}{(c) \texttt{BLCE}, uniform} \par}
\end{subfigure}\vspace{-1em}
\begin{subfigure}{\textwidth}
\includegraphics[width=\textwidth]{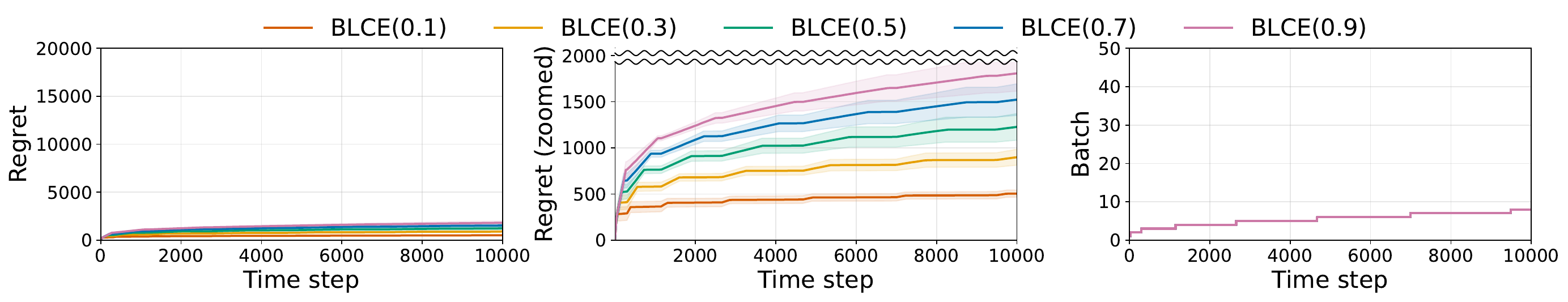}
\centering
{\small \raisebox{0.5cm}{(d) \texttt{BLCE}, normal} \par}
\end{subfigure}\vspace{-2em}
\setcounter{figure}{4}
\captionof{figure}{Regret and interval (batch) complexity over time for different values of $c$.}
\label{various c figure}
\end{figure*}

\begin{figure*}[ht]
\centering
\begin{subfigure}{\textwidth}
\includegraphics[width=\textwidth]{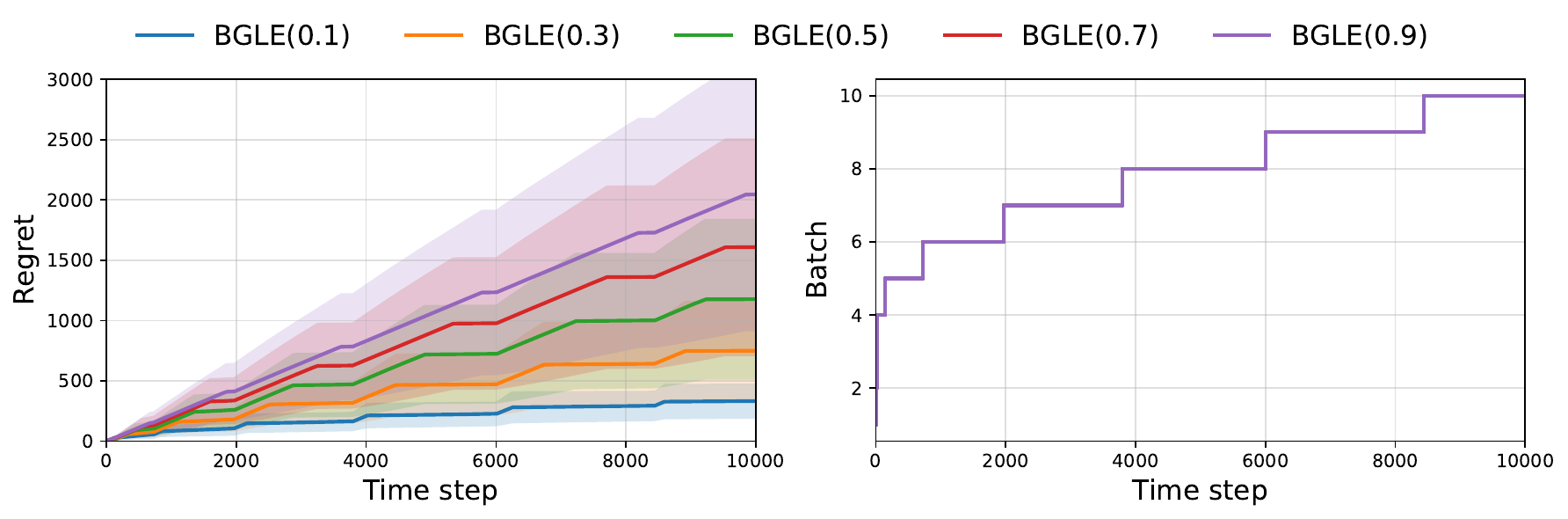}
\centering
{\raisebox{0.5cm}{(a) \texttt{BGLE}, uniform} \par}
\end{subfigure}\vspace{-1em}
\begin{subfigure}{\textwidth}
\includegraphics[width=\textwidth]{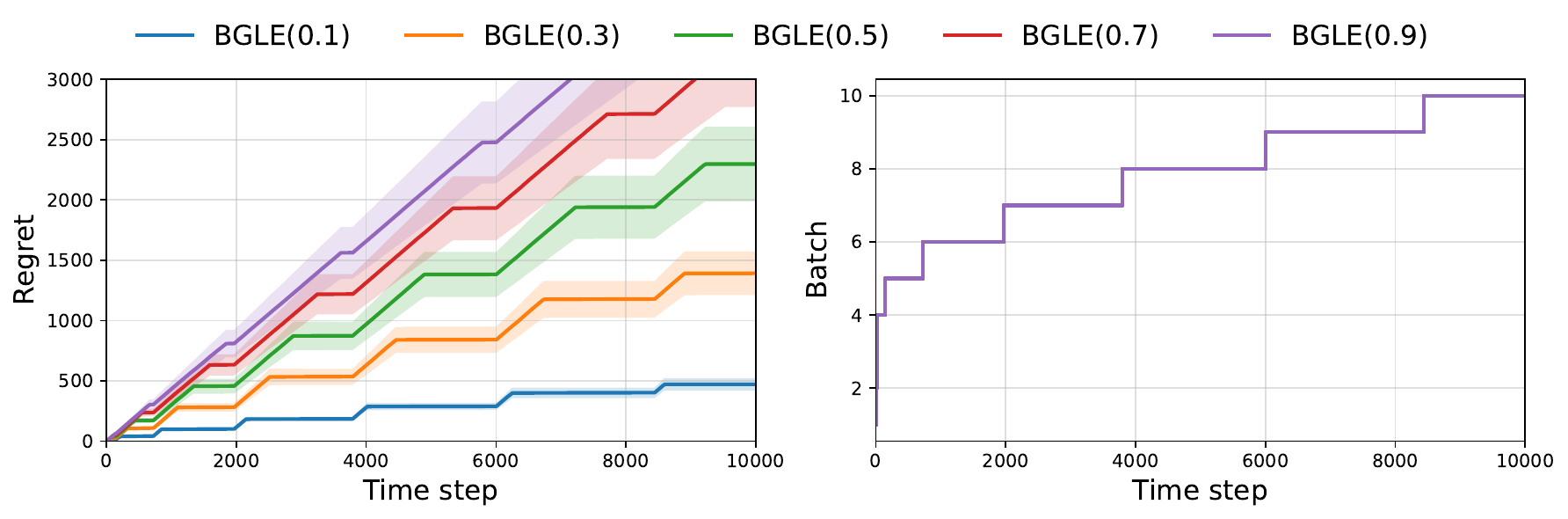}
\centering
{\raisebox{0.5cm}{(b) \texttt{BGLE}, normal} \par}
\end{subfigure}\vspace{-1em}
\setcounter{figure}{5}
\captionof{figure}{Regret and interval (batch) complexity over time for different values of $c$.}
\label{various c GLM figure}
\end{figure*}

\end{document}